%% file: arxiv.tex
\documentclass[12pts]{article}
\usepackage[utf8]{inputenc}
\usepackage[round]{natbib}
\usepackage[margin=1.1in]{geometry}
\usepackage{amsmath}
\usepackage{amssymb}
\usepackage{amsthm}
\usepackage{hyperref}
\usepackage{xcolor}
 \definecolor{RoyalBlue}{rgb}{0, 0, 179}
\hypersetup{
    citecolor= RoyalBlue,
    colorlinks=true,
    linkcolor=RoyalBlue,
    filecolor=magenta,      
    urlcolor=cyan,
}
\usepackage{thmtools}
\usepackage{thm-restate}
\usepackage{cleveref}
\usepackage{fancyhdr}
\usepackage[mathscr]{euscript}
\usepackage{bbm}
\usepackage{mathtools}
\usepackage{graphicx}
\usepackage{subcaption}
\usepackage{tikz}
\usepackage{wrapfig}
\usetikzlibrary{positioning}
\usetikzlibrary{decorations.markings}
\input{custom_definitions.tex}

\usepackage{titletoc}
\usepackage{wrapfig}

\newtheorem{defn}{Definition}
\newtheorem{thm}[defn]{Theorem}
\newtheorem{lem}[defn]{Lemma}
\newtheorem{cor}[defn]{Corollary}
\newtheorem{ass}{A\hspace{-4pt}}
\newtheorem{bss}{B\hspace{-4pt}}
\newtheorem{eg}[defn]{Example}

\newtheorem{propos}[defn]{Proposition}
\usepackage[max2]{authblk}
\title{On the Statistical Properties of Generative Adversarial Models for Low Intrinsic Data Dimension}
\author[1]{Saptarshi Chakraborty\thanks{email: \texttt{saptarshic@berkeley.edu}}}
 \author[1,2,3]{Peter L.~Bartlett\thanks{email: \texttt{peter@berkeley.edu}}}
 \affil[1]{Department of Statistics, UC Berkeley}
  \affil[2]{Department of Electrical Engineering and Computer Sciences, UC Berkeley}
 \affil[3]{Google DeepMind}
\date{\vspace{-5ex}}

\begin{document}
\maketitle
\begin{abstract}
Despite the remarkable empirical successes of Generative Adversarial Networks (GANs), the theoretical guarantees for their statistical accuracy remain rather pessimistic. In particular, the data distributions on which GANs are applied, such as natural images, are often hypothesized to have an intrinsic low-dimensional structure in a typically high-dimensional feature space, but this is often not reflected in the derived rates in the state-of-the-art analyses. In this paper, we attempt to bridge the gap between the theory and practice of GANs and their bidirectional variant, Bi-directional GANs (BiGANs), by deriving statistical guarantees on the estimated densities in terms of the intrinsic dimension of the data and the latent space. We analytically show that if one has access to $n$ samples from the unknown target distribution and the network architectures are properly chosen, the expected Wasserstein-1 distance of the estimates from the target scales as $\mathcal{O}\left( n^{-1/d_\mu } \right)$  for GANs and $\tilde{\mathcal{O}}\left( n^{-1/(d_\mu+\ell)} \right)$  for BiGANs,  where $d_\mu$ and $\ell$ are the upper Wasserstein-1 dimension of the data-distribution and latent-space dimension, respectively. The theoretical analyses not only suggest that these methods successfully avoid the curse of dimensionality, in the sense that the exponent of $n$ in the error rates does not depend on the data dimension but also serve to bridge the gap between the theoretical analyses of GANs and the known sharp rates from optimal transport literature.  Additionally, we demonstrate that GANs can effectively achieve the minimax optimal rate even for non-smooth underlying distributions, with the use of interpolating generator networks.
\end{abstract}



\section{Introduction}
The problem of generating new data from an unknown distribution by observing independent and identically distributed samples from the same has been of great importance to researchers and finds fruitful applications in the fields of statistics, computer vision,  bio-medical imaging, astronomy, and so on \citep{alqahtani2021applications}.  Recent developments in deep learning have led to the discovery of Generative Adversarial Networks (GANs) \citep{NIPS2014_5ca3e9b1}, which formulates the problem as a (often zero-sum) game between two adversaries, called the generator and the discriminator. GANs have been incredibly successful, especially in the field of computer vision in generating realistic and high-resolution samples resembling natural images. The generator takes input from a relatively low-dimensional white noise, e.g., normal or uniform distribution, and tries to output realistic examples from the target distribution, while the discriminator tries to differentiate between real and fake samples. Both the generator and discriminator are typically realized by some class of deep networks. 

Different variants of GANs have been proposed in the past few years, often by varying the underlying divergence measure used and involving additional architectures, to learn different aspects of the data distribution. For example, $f$-GANs \citep{nowozin2016f} generalize vanilla GANs by incorporating a $f$-divergence as a dissimilarity measure. Wasserstein GANs (WGANs)\citep{arjovsky2017wasserstein} incorporate the first-order Wasserstein distance (also known as Kantorovich distance) to deduce a better-behaved GAN-objective. Some popular variants of GAN include MMD-GAN \citep{dziugaite2015training}, LSGAN \citep{mao2017least}, Cycle-GAN \citep{zhu2017unpaired}, DualGAN \citep{yi2017dualgan}, DiscoGAN \citep{kim2017learning} etc. Recent progress in GANs has allowed us to learn not only the map from the latent space to the data space (the generator) but the reverse map as well. This map called the encoder, is useful in finding a proper low-dimensional representation of the samples. Bi-directional GANs (BiGANs)  \citep{donahue2017adversarial} implement an adversarial learning scheme with the discriminator trying to distinguish between the fake data-latent code pair and its real counterpart (see Section~\ref{background} for a detailed exposition). Ideally, during the training process, the encoder learns a useful feature representation of the data in the latent space for auxiliary machine-learning tasks and data visualization.

The empirical successes of GANs have motivated researchers to study their theoretical guarantees. \citet{biau2020some}  analyzed the asymptotic properties of vanilla GANs along with parametric rates. \cite{biau2021some} also analyzed the asymptotic properties of WGANs. \citet{liang2021well} explored the min-max rates for WGANs for different non-parametric density classes and under a sampling scheme from a kernel density estimate of the data distribution; while \citet{schreuder2021statistical} studied the finite-sample rates under adversarial noise.  \citet{uppal2019nonparametric} derived the convergence rates for Besov discriminator classes for WGANs. \citet{luise2020generalization} conducted a theoretical analysis of WGANs under an optimal transport-based paradigm. Recently, \citet{asatryan2020convenient} and \citet{belomestny2021rates} improved upon the works of \citet{biau2020some} to understand the behavior of GANs for H\"{o}lder class of density functions. \citet{arora2017generalization} showed that generalization might not hold in standard metrics. However, they show that under a restricted ``neural-net distance", the GAN is indeed guaranteed to generalize well. Recently, \citet{arora2018do} showed that GANs and their variants might not be well-equipped against mode collapse. In contrast to GANs, the theoretical understanding of BiGANs remains rather limited. While \cite{liu2021non} attempted to establish generalization bounds for the BiGAN problem, it also suffers from the curse of dimensionality as the rates depend on the nominal high-dimensionality of the entire data space.  

Although significant progress has been made in our theoretical understanding of GAN, some limitations of the existing results are yet to be addressed. For instance, the generalization bounds frequently suffer from the curse of dimensionality. In practical applications, data distributions tend to have high dimensionality, making the convergence rates that have been proven exceedingly slow. However, high-dimensional data, such as images, texts, and natural languages, often possess latent low-dimensional structures that reduce the complexity of the problem. For example, it is hypothesized that natural images lie on a low-dimensional structure, despite its high-dimensional pixel-wise representation \citep{pope2020intrinsic}. Though in classical statistics there have been various approaches, especially using kernel tricks and Gaussian process regression that achieve a fast rate of convergence that depends only on their low intrinsic dimensionality \citep{bickel2007local,kim2019uniform}, such results are largely unexplored in the context of GANs. Recently, \citet{JMLR:v23:21-0732} expressed the generalization rates for GAN when the data has low-dimensional support in the Minkowski sense and the latent space is one-dimensional; while \citet{dahal2022deep} derived the convergence rates under the Wasserstein-1 distance in terms of the manifold dimension. It is important to note that the Minkowski dimension, although useful, does not fully capture the low-dimensional nature of the underlying distribution (see Section~\ref{intrinsic}), while the compact Riemannian manifold assumption of the support of the target distribution and the assumption of a bounded density of the target distribution on this manifold in \citep{dahal2022deep} is a very strong assumption that might not hold in practice. Furthermore, none of the aforementioned approaches tackle the problem in its full generality and match the sharp convergence rates of the empirical distributions in the optimal transport literature \citep{weed2019sharp}. Additionally, it remains unknown whether the GAN estimates of the target distribution are optimal in the minimax sense. 

\paragraph{Contributions }
In an attempt to overcome the aforementioned drawbacks in the current literature, the major findings of the paper are highlighted below.
\begin{itemize}
    \item In order to bridge the gap between the theory and practice of such generative models, in this paper, we develop a framework to establish the statistical convergence rates of GANs and BiGANs in terms of the upper Wasserstein dimension \citep{weed2019sharp} of the underlying target measure. 
    \item Informally, our results guarantee that if one has access to $n$ independent and identically distributed samples (i.i.d.) from $\mu$ and if the network architectures are properly chosen, the expected $\beta$-H\"{o}lder Integral Probability Metric (IPM) between the estimated and target distribution for GANs scales at most at a rate of $ \E \|\mu - \hG_\sharp \nu \|_{\sH^\beta} \lesssim n^{-\beta/d^\ast_\beta(\mu)}$. Here, $d^\ast_\beta(\mu)$ is the $\beta$-upper Wasserstein dimension of the target distribution $\mu$ (see Definition \ref{ed}). The recent statistical guarantees of GANs \citep{JMLR:v23:21-0732,dahal2022deep} follow as a direct corollary of our main result. Similarly, for BiGANs, we can guarantee that the expected Wasserstein-1 distance between the estimated joint distributions scales roughly at a rate (ignoring the $\log$-factors) of  $ \E \sW_1\left((\mu, \hE_\sharp \mu), (\hD_\sharp \nu, \nu)\right)  \precsim n^{-1/(d^\ast_1(\mu) +\ell)}$. Here $\hE$ and $\hD$ are the optimal sample encoders and decoders and $\ell$ is the dimension of the latent space. Notably, when the support encompasses the entire dataspace, this outcome aligns with the findings of \citet{liu2021non}.
    \item  We introduce the entropic dimension, characterizing a distribution's intrinsic dimensionality, and demonstrate its relevance in deep learning theory, notably in approximation and generalization bounds.  Our investigation into the approximation capabilities of ReLU networks reveals that achieving $\epsilon$-approximation in the $\fL_p(\mu)$-norm for $\alpha$-H\"{o}lder functions demands roughly $\cO(\epsilon^{-\bar{d}_{\alpha p}(\mu)/\alpha})$ weight terms (refer to Theorem~\ref{approx}).
     This contrasts with prior estimates of $\cO(\epsilon^{-d/\alpha} )$ \citep{yarotsky2017error} or $\cO(\epsilon^{-\overline{\text{dim}}_M(\mu)/\alpha})$ \citep{JMLR:v21:20-002}. Here \(\bar{d}_{p \alpha}(\mu) \le \overline{\text{dim}}_M(\mu) \le d\) denote the $p\alpha$-entropic dimension, upper Minkowski dimension of $\mu$ and the dimension of the data space, respectively (see Section \ref{intrinsic}). The result implies sustained approximation capabilities even with smaller ReLU networks. Furthermore, we demonstrate that the $\fL_p(\mu)$-metric entropy of $\beta$-H\"{o}lder functions scales at most at a rate of $\cO(\epsilon^{-\bar{d}_{p\beta}(\mu)})$, enhancing the foundational results of \cite{kolmogorov1961}.
    \item Finally, we derive minimax optimal rates for estimating under the $\alpha$-H\"{o}lder IPM, establishing that the GAN estimates can approximately achieve this minimax optimal rate.
\end{itemize}
\paragraph{Organization}

The remaining sections are structured as follows: Section~\ref{background} revisits necessary notations, and definitions, and outlines the problem statement. In Section~\ref{intrinsic}, we revisit the concept of intrinsic dimension and introduce a novel dimension termed the entropic dimension of a measure, comparing it with commonly used metrics. This entropic dimension proves pivotal in characterizing both the $\fL_p$-covering number of H\"{o}lder functions (refer to Theorem~\ref{kt}). We show that the Wasserstein dimension determines the convergence rate of the empirical measure to the population in the H\"older Integral IPM in   Theorem~\ref{corr1}. The subsequent focus shifts to theoretical analyses of GANs and BiGANs in Section~\ref{theo ana}. We begin by presenting the assumptions in Section~\ref{assumptions}, followed by stating the main result in Section~\ref{main results} and providing a proof sketch in Section~\ref{pf_main_result}, with detailed proofs available in the Appendix. Section~\ref{minimax bounds} demonstrates that GANs can achieve the minimax optimal rates for estimating distributions, followed by concluding remarks in Section~\ref{conclusion}.


\section{Background}\label{background}
Before we go into the details of the theoretical results, we introduce some notation and recall some preliminary concepts. 

\paragraph{Notations}
We use the notation $x \vee y : = \max\{x,y\}$ and $x \wedge y : = \min\{x,y\}$.  $T_\sharp \mu$ denotes the push-forward of the measure $\mu$ by the map $T$. $B_\varrho(x,r)$ denotes the open ball of radius $r$ around $x$, with respect to (w.r.t.) the metric $\varrho$. For any measure $\gamma$, the support of $\gamma$ is defined as, $\text{supp}(\gamma) = \{x: \gamma(B_\varrho(x,r))>0, \text{ for all }r >0\}$.  For any function $f: \cS \to \Real$, and any measure $\gamma$ on $\cS$, let $\|f\|_{\fL_p(\gamma)} : = \left(\int_\cS |f(x)|^p d \gamma(x) \right)^{1/p}$, if $0<p< \infty$. Also let, $\|f\|_{\fL_\infty(\gamma)} : = \operatorname{ess\, sup}_{x \in \text{supp}(\gamma)}|f(x)|$. For any function class $\cF$, and distributions $P$ and $Q$, $\|P - Q\|_\cF = \sup_{f \in \cF} |\int f dP - \int f dQ|$. For function classes $\cF_1$ and $\cF_2$, $\cF_1 \circ \cF_2 = \{f_1 \circ f_2: f_1 \in \cF_1, \, f_2 \in \cF_2\}$. We say $A_n \lesssim B_n$ (also written as $A_n = \cO(B_n)$) if there exists $C>0$, independent of $n$, such that $A_n \le C B_n$. Similarly, the notation, ``$\precsim$" (also written as $A_n = \tilde{\cO}(B_n)$) ignores poly-log factors in  $n$. We say $A_n \asymp B_n$, if $A_n \lesssim B_n$ and $B_n \lesssim A_n$. For any $k \in \mathbb{N}$, we let $[k] = \{1, \dots, k\}$. For two random variables $X$ and $Y$, we say that $X  \overset{d}{=} Y$, if the random variables have the same distribution. We use bold lowercase letters to denote members of $\mathbb{N}^k$ for $k \in \mathbb{N}$. $\Pi_{\cA}$ denotes the set of all probability measures on the set $\cA$. $\sW_p(\cdot, \cdot)$ denotes the Wasserstein $p$-distance between probability distributions. $\mu \otimes \nu$ denotes the product distribution of $\mu$ and $\nu$.

\begin{defn}[Covering and Packing Numbers] 
    \normalfont 
    For a metric space $(\cS,\varrho)$, the $\epsilon$-covering number w.r.t. $\varrho$ is defined as:
    \(\cN(\epsilon; \cS, \varrho) = \inf\{n \in \mathbb{N}: \exists \, x_1, \dots x_n \text{ such that } \cup_{i=1}^nB_\varrho(x_i, \epsilon) \supseteq \cS\}.\) A minimal $\epsilon$ cover of $\cS$ is denoted as $\cC(\epsilon; \cS, \varrho)$.
    Similarly, the $\epsilon$-packing number is defined as:
    \(\cM(\epsilon; \cS, \varrho) = \sup\{m \in \mathbb{N}: \exists \, x_1, \dots x_m \in \cS \text{ such that } \varrho(x_i, x_j) \ge \epsilon, \text{ for all } i \neq j\}.\)
\end{defn}
\begin{defn}[Neural networks]\normalfont
 Let $L \in \mathbb{N}$ and $ \{N_i\}_{i \in [L]} \in \mathbb{N}$. Then a $L$-layer neural network $f: \Real^d \to \Real^{N_L}$ is defined as,
\begin{equation}
\label{ee1}
f(x) = A_L \circ \sigma_{L-1} \circ A_{L-1} \circ \dots \circ \sigma_1 \circ A_1 (x)    
\end{equation}
Here, $A_i(y) = W_i y + b_i$, with $W_i \in \Real^{N_{i} \times N_{i-1}}$ and $b_i \in \Real^{N_{i-1}}$, with $N_0 = d$. Note that $\sigma_j$ is applied component-wise.  Here, $\{W_i\}_{1 \le i \le L}$ are known as weights, and $\{b_i\}_{1 \le i \le L}$ are known as biases. $\{\sigma_i\}_{1 \le i \le L-1}$ are known as the activation functions. Without loss of generality, one can take $\sigma_\ell(0) = 0, \, \forall \, \ell \in [L-1]$. We define the following quantities:  
(Depth) $\cL(f) : = L$ is known as the depth of the network; (Number of weights) the number of weights of the network $f$ is denoted as  $\cW(f) = \sum_{i=1}^L N_i N_{i-1}$; 
(maximum weight) $\cB(f) = \max_{1 \le j \le L} (\|b_j\|_\infty) \vee \|W_j\|_{\infty}$ to denote the maximum absolute value of the weights and biases.
\begin{align*}
    \cN \cN_{\{\sigma_i\}_{1 \le i \le L-1}} (L, W, B, R) = \{ & f \text{ of the form \eqref{ee1}}: \cL(f) \le L , \, \cW(f) \le W, \cB(f) \le B,\\
    & \sup_{x \in \Real^d}\|f(x)\|_\infty  \le R  \}.
\end{align*}
 If $\sigma_j(x) = x \vee 0$, for all $j=1,\dots, L-1$, we use the notation $\cR \cN (L, W, B, R)$ to denote $\cN \cN_{\{\sigma_i\}_{1 \le i \le L-1}} (L, W, B, R)$.
 \end{defn}
 \begin{defn}[H\"{o}lder functions]\normalfont
Let $f: \mathcal{S} \to \Real$ be a function, where $\mathcal{S} \subseteq \Real^d$. For a multi-index $\bs = (s_1,\dots,s_d)$, let, $\partial^{\bs} f = \frac{\partial^{|\bs|} f}{\partial x_1^{s_1} \dots \partial x_d^{s_d}}$, where, $|\bs| = \sum_{\ell = 1}^d s_\ell $. We say that a function $f: \cS \to \Real$ is $\beta$-H\"{o}lder (for $\beta >0$) if
\[ \|f\|_{\sH^\beta}: =\sum_{\bs: 0 \le |\bs| \le \lfloor \beta \rfloor} \|\partial^{\bs} f\|_\infty  + \sum_{\bs: |\bs| = \lfloor \beta \rfloor} \sup_{x \neq y}\frac{\|\partial^{\bs} f(x)  - \partial^{\bs} f(y)\|}{\|x - y\|^{\beta - \lfloor \beta \rfloor}} < \infty. \]
If $f: \Real^{d_1} \to \Real^{d_2}$, then we define $\|f\|_{\sH^{\beta}} = \sum_{j = 1}^{d_2}\|f_j\|_{\sH^{\beta}}$. For notational simplicity, let, $\sH^\beta(\cS_1, \cS_2,C) = \{f: \cS_1 \to \cS_2: \|f\|_{\sH^\beta} \le C\}$. Here, both $\cS_1$ and $\cS_2$ are subsets of real vector spaces.
\end{defn}
\subsection{Generative Adversarial Networks (GANs)}
Generative adversarial networks or GANs provide a simple yet effective way to generate samples from an unknown data distribution, given i.i.d. samples from the same. Suppose that $\sX$ is the data space and $\mu$ is a probability distribution on $\sX$. For simplicity, we will assume that $\sX = [0,1]^{d}$. Let $\sZ$ be the latent space, from which it is easy to generate samples. We will take $\sZ = [0,1]^{\ell}$. GANs view the problem by modeling a two-player (often zero-sum) game between the discriminator and the generator. The generator $G$ maps elements of $\sZ$ to $\sX$. The generator generates fake data points by first simulating a point in $\sZ$ from some distribution $\nu$ and passing it through the generator $G$. The discriminator's job is to distinguish between the original and fake samples.  Wasserstein GAN's objective is to solve the min-max problem, 
\[\inf_G  \|\mu - G_\sharp \nu\|_{\Phi},\] where $\Phi$ is a class of discriminators.

One often does not have access to $\mu$ but, only samples $X_1, \dots,X_n \overset{\text{i.i.d.}}{\sim} \mu$. Moreover one also assumes that one can generate samples, $Z_1,\dots,Z_m$ independently from $\nu$. The empirical distributions of $\mu$ and $\nu$ are given by, $\hat{\mu}_n = \frac{1}{n} \sum_{i=1}^n \delta_{X_i}$ and $\hat{\nu}_m = \frac{1}{n} \sum_{j=1}^m \delta_{Z_j}$, respectively. In practice, one typically solves the following empirical objectives.
\begin{equation}\label{emp}
     \inf_G \sup_\phi \frac{1}{n} \sum_{i=1}^n  \phi(X_i) - \E_{Z \sim \nu} \phi(G(Z))  \quad \text{or} \quad \inf_G \sup_\phi \frac{1}{n} \sum_{i=1}^n  \phi(X_i) - \frac{1}{m} \sum_{j=1}^m \phi(G(Z_j)).
\end{equation}
One typically realizes $\sG$ through a class of ReLU networks.  The empirical estimates of the generator are defined as:
\begin{equation}\label{gan_est}
    \hG_n = \argmin_{G \in \sG} \|\hmu_n - G_\sharp \nu\|_{\Phi} \quad \text{or} \quad \hG_{n,m} = \argmin_{G \in \sG} \|\hmu_n - G_\sharp \hnu_m\|_{\Phi}.
\end{equation}
Of course, in practice, one finds an approximation of $\hG$ only up to an optimization error, which we ignore for the simplicity of exposition; however one can incorporate this error into the theoretical results that we derive.

To measure the efficacy of the above generators, one computes the closeness of the target distribution $\mu$ and the estimated distribution $G_\sharp \nu$. As a measure of comparison, one takes an IPM for some function class $\sF$ and defines the excess risk of the estimator $\hG$ as:
\[\mathfrak{R}^{\text{GAN}}_{\sF}(\hG) = \|\mu - \hG_\sharp \nu\|_{\sF}.\]
In what follows, we take $\sF = \Phi = \sH^\beta(\Real^d, \Real, 1)$. It should be noted that while the discriminator is trained to differentiate real and fake samples, its role is limited to binary classification. As such, the discriminator network can be replaced by any efficient classifier, in principle. Neural networks are employed to approximate these Lipschitz or H\"{o}lder IPMs and enable efficient optimization for practical purposes. 

\subsection{Bidirectional GANs}
For notational simplicity, in the BiGAN context, we refer to the generator as the decoder and denote it by $D:\sZ \to \sX$. In BiGANs, one also constructs an encoder $E: \sX \to \sZ$. The discriminator is modified to take values from both the latent space and the data space, i.e. $\psi: \sX \times \sZ \to [0,1]$. For notational simplicity, $(f_\sharp P,P)$  and $(P, f_\sharp P)$ denote the distributions of $(f(X),X)$ and  $(X,f(X))$, respectively,  where $X \sim P$. The discriminator in the BiGAN objective also aims to distinguish between the real and fake samples as much as possible, whereas the encoder-decoder pair tries to generate realistic fake data-latent samples to confuse the discriminator. The BiGAN problem is formulated as follows:
\begin{align}
& \min_{D , \, E}  \|(\mu,E_\sharp\mu) - (D_\sharp \nu, \nu)\|_{\Psi},
\end{align}
where $\Psi$ denotes the set of all discriminators. In practice, one typically has access only to samples $X_1, \dots,X_n \overset{\text{i.i.d.}}{\sim} \mu \text{ and } Z_1,\dots,Z_m\overset{\text{i.i.d.}}{\sim}\nu$, and these are used to solve one of the following empirical objectives.
\begin{align*}
   & \inf_{D, \, E} \sup_\psi \frac{1}{n} \sum_{i=1}^n  \psi(X_i, E(X_i)) - \E_{Z \sim \nu} \psi(D(Z), Z),  \\
   & \inf_{D, \, E} \sup_\psi \frac{1}{n} \sum_{i=1}^n  \psi(X_i, E(X_i)) - \frac{1}{m} \sum_{j=1}^m\psi(D(Z_j), Z_j).
\end{align*}
The functions $D$, $E$, and $\psi$ are typically realized by deep networks though it is not uncommon to model $\psi$ through functions belonging to certain Hilbert spaces \citep{li2017mmd}. 

We let $\sD$ and $\sE$ be the set of all decoders and encoders of interest, which we will take to be a class of ReLU networks. The empirical decoder and encoder pairs are defined as:
\begin{align}
  &  (\hD_n, \hE_n) = \argmin_{D \in \sD, E \in \sE} \|(\hmu_n, E_\sharp \hmu_n) - (D_\sharp \nu, \nu)\|_{\Psi} \label{bigan_est1}\\
  & (\hD_{n,m}, \hE_{n,m}) = \argmin_{D \in \sD, E \in \sE} \|(\hmu_n, E_\sharp \hmu_n) - (G_\sharp \hnu_m, \hnu_m)\|_{\Psi}. \label{bigan_est2}
\end{align}

 To measure the efficacy of the above estimates, one computes the closeness of the joint distribution $(\mu, \hE_\sharp \mu)$ to $ (\hD_\sharp \nu, \nu)$ as,
\[\mathfrak{R}^{\text{BiGAN}}_{\sF}(\hD, \hE) = \|(\mu, \hE_\sharp \mu) - (\hD_\sharp \nu, \nu)\|_{\sF}.\]
As before, we take $\sF = \Psi = \sH^\beta(\Real^{d+\ell}, \Real, 1)$.
\section{Intrinsic Dimension of Data Distribution}
\label{intrinsic}
Often, real data is hypothesized to lie on a lower-dimensional structure within the high-dimensional representative feature space. To characterize this low-dimensionality of the data, researchers have defined various notions of the effective dimension of the underlying measure from which the data is assumed to be generated. Among these approaches, the most popular ones use some sort of rate of increase of the covering number, in the $\log$-scale, of most of the support of this data distribution. Let $(\cS, \varrho)$ be a Polish space and let $\mu$ be a probability measure defined on it. Throughout the remainder of the paper, we take $\varrho$ to be the $\ell_\infty$-norm. 
Before, we proceed we recall the $(\epsilon, \tau)$-cover of a measure 
\citep{posner1967epsilon} as: \(\sN_\epsilon(\mu, \tau) = \inf\{\cN(\epsilon; S, \varrho): \mu(S) \ge 1-\tau\},\)  i.e. $\sN_\epsilon(\mu, \tau)$ counts the minimum number of $\epsilon$-balls required to cover a set $S$ of probability at least $1-\tau$.

The most simplistic notion of the dimension of a measure is the upper Minkowski dimension of its support, which is defined as: \[\overline{\text{dim}}_M(\mu) = \limsup_{\epsilon \downarrow 0} \frac{\log\cN(\epsilon;\text{supp}(\mu), \ell_\infty)}{\log (1/ \epsilon)}.\] 
Since this notion of dimensionality depends only on the covering number of the support and does not assume the existence of a smooth correspondence to a smaller dimensional Euclidean space, this notion not only incorporates smooth manifolds but also covers highly non-smooth sets such as fractals.  The statistical convergence guarantees of different estimators in terms of the upper Minkowski dimension are well-studied in the literature. \citet{kolmogorov1961} provided a comprehensive study on the dependence of the covering number of different function classes on the underlying upper Minkowski dimension of the support. \citet{JMLR:v21:20-002} showed how deep learners can incorporate this low-dimensionality of the data that is also reflected in their convergence rates. Recently, \citet{JMLR:v23:21-0732} showed that WGANs can also adapt to this low-dimensionality of the data. However, one key drawback of using the upper Minkowski dimension is that if the measure spans over the entire sample space, though being concentrated on only some regions, can result in a high value of the dimension. We refer the reader to Examples \ref{eg1} and \ref{eg2} for such instances. 

To overcome the aforementioned difficulty, as a notion of the intrinsic dimension of a measure $\mu$,  \citet{dudley1969speed} defined the entropic dimension of a measure as:
\[\kappa_\mu = \limsup_{\epsilon \downarrow 0} \frac{\log \sN_\epsilon(\mu,\epsilon)}{\log (1/\epsilon)}.\]
Inspired by this definition, we define the 
$\alpha$-entropic dimension of $\mu$ as follows.
\begin{defn}[Entropic Dimension]\label{ed}\normalfont
    For any $\alpha>0$, we define the $\alpha$-entropic dimension of $\mu$ as:
     \[\bar{d}_\alpha(\mu) = \limsup_{\epsilon \downarrow 0} \frac{\log \sN_\epsilon(\mu,\epsilon^\alpha)}{\log (1/\epsilon)}.\]
\end{defn}

Clearly, $\bar{d}_\alpha(\mu) \ge 0$.  Note that Dudley's entropic dimension is $\bar{d}_1(\mu)$. \cite{dudley1969speed} showed that when the data is i.i.d. $\E \|\hmu_n - \mu\|_{\sH^1}$ roughly scales as $\cO \left( n^{-1/\bar{d}_1(\mu)}\right)$, while $\E \operatorname{Dist}_{LP}(\hmu_n, \mu)$ roughly scales as $\cO \left( n^{-1/(2+\bar{d}_1(\mu))}\right)$, where $\operatorname{Dist}_{LP}(\cdot, \cdot)$ denotes the L\'evy-Prokhorov metric \citep{doi:10.1137/1101016}. As shown in Theorems \ref{kt} and \ref{approx}, the entropic dimension characterizes the metric entropy of $\beta$-H\"older functions and the approximation capabilities of neural networks with ReLU activation.  \citet{weed2019sharp} developed upon Dudley's entropic dimension to characterize the expected convergence rate of the Wasserstein-$p$ distance between a measure and its empirical counterpart. The upper and lower Wasserstein dimension of $\mu$ is defined as follows:

\begin{defn}[Upper and Lower Wasserstein Dimensions \citep{weed2019sharp}]\label{def1} \normalfont
For any $ \alpha>0$, the $\alpha$-upper dimension of $\mu$ is defined by 
\[d^\ast_\alpha(\mu) = \inf \left\{s \in (2\alpha, \infty) : \limsup_{\epsilon \downarrow 0} \frac{\log \sN_{\epsilon}\left(\mu, \epsilon^{\frac{s\alpha}{s-2\alpha}}\right)}{\log(1/\epsilon)} \leq s \right\}.\]
The lower Wasserstein dimension is defined as:
$d_\ast(\mu) = \lim_{\tau \downarrow 0}\liminf_{\epsilon \downarrow 0} \frac{\log \sN_\epsilon(\mu,\tau)}{\log(1/\epsilon)}$.
\end{defn}

\cite{weed2019sharp} showed that, roughly, $n^{-1/d_\ast(\mu)} \lesssim \E \sW_p(\hmu_n,\mu) \lesssim n^{-1/d^\ast_p(\mu)}$. In Proposition~\ref{lem1}, we give some elementary properties of these dimensions with proof in Appendix~\ref{p_lem1}.  We also recall the definition of regularity dimensions and packing dimension of a measure \citep{fraser2017upper}.
\begin{defn}[Regularity dimensions]
\normalfont
The upper and lower regularity dimensions of a measure are defined as:
\begin{align*}
    \overline{\text{dim}}_{\text{reg}}(\mu) = & \inf\bigg\{ s: \exists \, C>0 \text{ such that, for all } 0 < r < R \text{ and } x \in \text{supp}(\mu), \\
    & \hspace{2cm} \frac{\mu(B(x,R))}{\mu(B(x,r))} \le C \left(\frac{R}{r}\right)^s\bigg\},\\
        \underline{\text{dim}}_{\text{reg}}(\mu) = & \sup\bigg\{ s: \exists \, C >0\text{ such that, for all } \, 0 < r < R  \text{ and } x \in \text{supp}(\mu), \\
        & \hspace{2cm} \frac{\mu(B(x,R))}{\mu(B(x,r))} \ge C \left(\frac{R}{r}\right)^s\bigg\}.
    \end{align*} 
\end{defn}

\begin{defn}[Upper packing dimension] \normalfont The upper packing dimension of a measure $\mu$ is defined as:
    $\overline{\text{dim}}_P(\mu) = \operatorname{ess\, sup} \left\{\limsup\limits_{r \rightarrow 0} \frac{\log \mu(B(x,r))}{\log r}: x \in \text{supp}(\mu)\right\}$.
\end{defn}

 \begin{restatable}{propos}{pflemone}
\label{lem1}
For any $\mu$ and for any $\alpha \ge 0$, the following hold:
\begin{enumerate}
    \item[(a)] $d_\ast(\mu) \le \bar{d}_\alpha(\mu) \le d^\ast_\alpha(\mu) $, 
    \item[(b)] if $\alpha_1 \le \alpha_2$, then, $\bar{d}_{\alpha_1}(\mu) \le \bar{d}_{\alpha_2}(\mu)$,
    \item[(c)] $\bar{d}_\alpha(\mu) \le \overline{\text{dim}}_M(\mu)$,
    \item[(d)] if $\alpha \in \left(0, \overline{\text{dim}}_P(\mu)/2\right)$, $d^\ast_\alpha(\mu) \le \overline{\text{dim}}_P(\mu) \le \overline{\text{dim}}_{\text{reg}}(\mu) $,
    \item[(e)]  $\bar{d}_\alpha(\mu) \le \overline{\text{dim}}_P(\mu) \le \overline{\text{dim}}_{\text{reg}}(\mu) $,
    \item[(f)] $\underline{\text{dim}}_{\text{reg}}(\mu) \le d_\ast(\mu)  $.
\end{enumerate}
\end{restatable}

It was noted by \cite{weed2019sharp} that the upper Wasserstein dimension is usually smaller than the upper Minkowski dimension as noted in the following lemma.
\begin{propos}\label{lem_weed} (Proposition 2 of \cite{weed2019sharp})
    If $\overline{\text{dim}}_M \ge 2 \alpha$, then, $d^\ast_\alpha(\mu) \le \overline{\text{dim}}_M(\mu)$.
\end{propos}

 Propositions~\ref{lem1} (c) and \ref{lem_weed} imply that for high-dimensional data, both the entropic dimension and the upper Wasserstein dimension are no larger than the upper Minkowski dimension of the support. In many cases, strict inequality holds as seen in the following examples.
\begin{eg}\label{eg1}\normalfont
Let the measure $\mu$ on $\mathbb{N}$, be such that $\mu(n) = 2^{- n}$ for all $n \in \mathbb{N}$. Clearly, the support of $\mu$ is $\mathbb{N}$, which has an upper Minkowski dimension of $\infty$. To find $\bar{d}_\alpha(\mu)$, we first fix $\epsilon \in (0,1)$. For any $n \in \mathbb{N}$, let $A_n = [n]$. We observe that, $\mu\left(A_n^\complement\right) =  \frac{1}{2^n}$. We take 
$K  = \lceil \alpha \log_2(1/\epsilon)\rceil $.  Clearly, $\mu(A_{K}) \ge 1 - \epsilon^\alpha$. We can cover $A_{K}$ by at most $K$ many intervals of length $\epsilon$. Thus, $\sN_\epsilon(\mu,\epsilon^\alpha) \le K = \lceil  \log_2(1/\epsilon)\rceil$. Hence,
\[0 \le \bar{d}_\alpha(\mu) = \limsup_{\epsilon \downarrow 0} \frac{ \log \sN_\epsilon(\mu, \epsilon^\alpha)}{-\log \epsilon} \le  \limsup_{\epsilon \downarrow 0} \frac{ \log \lceil  \alpha \log_2(1/\epsilon)\rceil}{-\log \epsilon} =0.\] Thus, $\bar{d}_\alpha(\mu) = 0$, for all $\alpha >0$. Similarly, let $s > 2 \alpha$ and take, $K =  \lceil \frac{s \alpha}{s - 2 \alpha}\log_2\left(1/\epsilon\right)\rceil$. According to the argument above, 
\[\limsup_{\epsilon \downarrow 0} \frac{\log \sN_\epsilon(\mu, \epsilon^{\frac{s \alpha}{s - 2 \alpha}})}{-\log \epsilon} \le  \limsup_{\epsilon \downarrow 0} \frac{  \log \lceil \frac{s \alpha}{s - 2 \alpha}\log_2\left(1/\epsilon\right)\rceil}{-\log \epsilon} = 0 < s.\]
Thus, $d^\ast_\alpha(\mu) = 2 \alpha$. Thus, for this example, we get, $\bar{d}_\alpha(\mu) < d^\ast_\alpha(\mu) < \overline{\text{dim}}_M(\mu)$.
\end{eg}
Even when $\overline{\text{dim}}_M(\mu) < \infty$, it can be the case that $\bar{d}_\alpha(\mu) < d^\ast_\alpha(\mu) < \overline{\text{dim}}_M(\mu)$ as seen in the following example.
\begin{eg}\label{eg2}\normalfont
Let  $\mu$ be such that $\mu\left(\left(\frac{1}{n_1}, \dots, \frac{1}{n_d}\right)\right) = \frac{1}{2^n}$ for all $n_1, \dots, n_d \in \mathbb{N}$. We know that $\overline{\text{dim}}_M(\mu) = \frac{d}{2}$, by Example 1.14 of \citet{bishop2017fractals}. A calculation, similar to Example \ref{eg1} shows that $\bar{d}_\alpha(\mu) = 0$ and $d^\ast_\alpha(\mu) = 2\alpha$, for all $\alpha >0$. Thus, if $\alpha < \frac{d}{4}$, $\bar{d}_\alpha(\mu) < d^\ast_\alpha(\mu) < \overline{\text{dim}}_M(\mu)$.
\end{eg}

There are cases where both $\bar{d}_\alpha(\mu)$ and $d^\ast_\alpha(\mu)$ are both positive and a strict inequality might hold. We give an example as follows.
\begin{eg}\label{eg3}
    \normalfont
    Suppose that $\mu$ is a probability measure on $\mathbb{N}$, such that, $\mu(n) = n^{-2/3} - (n+1)^{-2/3}$. Then for any $\epsilon \in (0,1)$, let $K = \lceil \epsilon^{-3\alpha/2} \rceil$. Then, $\mu(A_K) \ge 1 - \epsilon^\alpha$. Again, $A_K$ requires at most $K$ intervals of length $\epsilon$ to be covered, which implies that $\sN_\epsilon(\mu, \epsilon^\alpha) \le K = \lceil \epsilon^{-3\alpha/2} \rceil$. Hence,
    \[\bar{d}_\alpha(\mu) = \limsup_{\epsilon \downarrow 0} \frac{ \log \sN_\epsilon(\mu, \epsilon^\alpha)}{-\log \epsilon} \le \limsup_{\epsilon \downarrow 0} \frac{ \log \lceil \epsilon^{-3\alpha/2} \rceil }{-\log \epsilon} = 3\alpha/2.\]
    Further, by definition, $d^\ast_\alpha(\mu) \ge 2 \alpha$. This implies that $\bar{d}_\alpha < d^\ast_\alpha(\mu) < \overline{\text{dim}}_M(\mu)$.
\end{eg}

The entropic dimension can be used to characterize the metric entropy of the set of all H\"{o}lder functions in the $\fL_p$-norm w.r.t the corresponding measure.  We focus on the following theorem that strengthens the seminal result by \citet{kolmogorov1961} (see Lemma~\ref{kt_og}). It is important to note that when the measure has an intrinsically lower dimension compared to the high dimensionality of the entire feature space, the metric entropy only depends on the $p\beta$-entropic dimension of the underlying measure $\mu$, not the dimension of the feature space, i.e. $d$.  
\begin{restatable}{thm}{pfkt}\label{kt}
    Let $\sF = \sH^\beta(\Real^d, \Real, C)$, for some $C>0$. Then, if $s>\bar{d}_{p\beta}(\mu)$, then, there exists an $\epsilon^\prime >0$, such that, if $0 < \epsilon \le \epsilon^\prime$, then,  $\log\cN(\epsilon; \sF, \fL_p(\mu)) \lesssim \epsilon^{-s/\beta}$.
\end{restatable}

We note that if $\overline{\text{dim}}_M(\mu) < \infty$, then, by Proposition~\ref{lem1} (c), $\bar{d}_{p\beta}(\mu) \le \overline{\text{dim}}_M(\mu)$. Thus as an immediate consequence of Theorem \ref{kt}, we observe that, if $s>\overline{\text{dim}}_M(\mu)$, then, for $\epsilon$ small enough,  $\log\cN(\epsilon; \sF, \fL_p(\mu)) \lesssim \epsilon^{-s/\beta}$, recovering a similar result as observed by \cite{kolmogorov1961} (Lemma \ref{kt_og}).

Next, we characterize the rate of convergence of $\hmu_n$ to $\mu$ under the $\beta$-H\"{o}lder IPM. The proof technique stems from Dudley's seminal work \citep{dudley1969speed}. The idea is to construct a dyadic-like partition of the data space and approximate the functions through their Taylor approximations on each of these small pieces. The reader is referred to Appendix \ref{pf_cor1} for a proof of this result.
\begin{restatable}{thm}{pfcorrone}\label{corr1}
   Let $\sF = \sH^\beta(\Real^d, \Real, C)$, for some $C>0$. Then, for any $d^\star > d^\ast_\beta(\mu)$, we can find $n_0 \in \mathbb{N}$, such that if $n \ge n_0$, \[\E \|\hmu_n -\mu\|_{\sF} \lesssim n^{-\beta / d^\star}.\]
 Here $n_0$ might depend on $\mu$ and $d^\star$.
\end{restatable}

Note that if $\beta, C = 1$, then, $\sF = \sH^\beta(\Real^d, \Real, 1)$ is the set of all 1-bounded and 1-Lipschitz functions (denoted as $\text{BL}$) on $\Real^d$. Thus, for any $d^\star > d_1^\ast(\mu)$, we observe that \(\E \|\hmu_n - \mu \|_{\text{BL}} \lesssim n^{-1/d^\star},\) recovering the bound derived by \citet{dudley1969speed}.
Furthermore, if $\mu$ has a bounded support, i.e. there exist $M>0$, such that, $\mu([-M,M]^d) =1$, then it is easy to see that $\E \sW_1(\hmu_n, \mu) \le  \E \|\hmu_n - \mu\|_{\sH^\beta([0,1]^d,\Real, M\sqrt{d})} \lesssim n^{-1/d^\star}$, which recovers the rates derived by \citet{weed2019sharp}, for Wasserstein-1 distance.

\section{Theoretical Analyses}\label{theo ana}
\subsection{Assumptions}\label{assumptions}
To lay the foundation for our analysis of GANs and BiGANs, we introduce a set of assumptions that form the basis of our theoretical investigations. These assumptions encompass the underlying data distribution, the existence of ``true" generator and encoder-decoder pairs, and certain smoothness properties.  For the purpose of the theoretical analysis, we assume that the data are independent and identically distributed from some unknown target distribution $\mu$ on $[0,1]^d$. This is a standard assumption in the literature \citep{liang2021well, JMLR:v23:21-0732} and is stated formally as follows:

\begin{ass}
    \label{dat_assumption}
    We assume that $X_1, \dots, X_n$ are independent and identically distributed according to the probability distribution $\mu$, such that $\mu\left([0,1]^d\right) = 1$.
\end{ass}

Furthermore, to facilitate the analysis of GANs, we introduce the concept of a ``true" generator, denoted as $\tilde{G}$. We assume that $\tilde{G}$ belongs to a function space $\sH^{\alpha_g}\left(\Real^\ell,\Real^d, C_g \right)$ and that the target distribution $\mu$ can be represented as $\mu = \tilde{G}_\sharp\nu$.
\begin{ass}\label{model1}
    There exists a $\tilde{G} \in \sH^{\alpha_g}\left(\Real^\ell,\Real^d, C_g \right)$, such that $\mu = \tilde{G}_\sharp\nu$.
\end{ass}
Similarly, for the BiGAN problem, we consider the existence of a smooth encoder-decoder pair, denoted as $\tilde{D}$ and $\tilde{E}$ respectively. These functions belong to the spaces $\sH^{\alpha_d}\left(\Real^\ell,\Real^d, C_d \right)$ and $\sH^{\alpha_e}\left(\Real^d,\Real^\ell, C_e \right)$, respectively, and satisfy the property that the joint distribution of $(X, {\tilde{E}}(X))$ is equal in distribution to $(\tilde{D}(Z), Z)$, where $X \sim \mu$ and $Z \sim \nu$. 
\begin{bss}\label{model2}
 There exists $\tilde{D} \in \sH^{\alpha_d}\left(\Real^\ell,\Real^d, C_d \right)$ and ${\tilde{E}} \in \sH^{\alpha_e}\left(\Real^d,\Real^\ell, C_e \right)$, such that, $(X, {\tilde{E}} (X)) \overset{d}{=}(\tilde{D} ( Z),Z)$. Here, $X \sim \mu$ and $Z \sim \nu$.
 \end{bss}
A direct consequence of Assumption B\ref{model2} is that the composition of the true encoder and decoder functions yields the identity map almost surely, as described in Lemma \ref{skorohod}. The notation, a.e. [$\mu$] denotes almost everywhere under the probability measure $\mu$.
\begin{restatable}{lem}{skorohod}\label{skorohod}
Under assumption~B\ref{model2}, $\tilde{D}\circ \tilde{E}(\cdot) = \text{id}(\cdot), \, \text{a.e.} \, [\mu]$ and $\tilde{E} \circ \tilde{D}(\cdot) = \text{id}(\cdot), \, \text{a.e.} \, [\nu]$. 
\end{restatable}

 \subsection{Main Results}\label{main results}
 Under assumptions A\ref{dat_assumption}, A\ref{model1} and B\ref{model2}, one can bound the expected excess risk of the GAN and BiGAN estimates of the target distribution in terms of the sample size with the exponent only depending on the upper Wasserstein dimension of the target distribution and the exponent of the H\"{o}lder IPM considered.  The main results of this paper are summarized in the following two theorems.  We recall that $\Phi = \sH^\beta(\Real^d, \Real, 1)$ and $\Psi= \sH^\beta(\Real^{d+\ell},\Real,1)$ denote the discriminator classes and $\hat{G}_n$, $\hat{G}_{n,m}$ denote the sample minimizers for the GAN-problem, defined in \eqref{gan_est}. Similarly, we recall the notations $(\hat{D}_n, \hat{E}_n) $  and $(\hat{D}_{n,m}, \hat{E}_{n,m})$ as the sample minimizers for the BiGAN-problem, defined in \eqref{bigan_est1} and \eqref{bigan_est2}, respectively. 
\begin{restatable}[Error rate for GANs]{thm}{pfgan}\label{main_gan}
     Suppose assumptions~A\ref{dat_assumption} and A\ref{model1} hold and let $s > d^\ast_\beta(\mu)$. There exist constants $N$, $c$ that might depend on $d, \ell, \alpha_g, \beta$ and $ \tilde{G}$, such that, if $n \ge N$, we can choose  $\sG = \cR \cN (L_g, W_g, B_g, 2C_g)$ with the network parameters as
    \( L_g \le c \log n,\, W_g \le c n^{\frac{\beta \ell}{\alpha_g s   (\beta \wedge 1)}} \log n.\)
    Then,
    \begin{equation}\label{m1}
        \E \|\mu - (\hG_n)_\sharp \nu\|_{\Phi} \lesssim  n^{-\beta/s} .
    \end{equation}
   Furthermore, if $m \ge \upsilon_n$, where, $\upsilon_n = \inf \left\{m \in \mathbb{N}: \frac{(\log m)^2}{ m^{\left(\max\left\{2 + \frac{\ell}{\alpha_g(\beta \wedge 1)}, \frac{d}{\beta}\right\}\right)^{-1}}} \le n^{-\beta/s}\right\}$, and the network parameters are chosen as,
   \( L_g \le c \log m,\, W_g \le c m^{\frac{\ell}{2 \alpha_g(\beta \wedge 1) + \ell}} \log m,\)
   then
   \begin{equation}
       \label{m2}
       \E \|\mu - (\hG_{n,m})_\sharp \nu\|_{\Phi} \lesssim  n^{-\beta/s} .
   \end{equation}
\end{restatable}

\begin{restatable}[Error rate for BiGANs]{thm}{pfbigan}\label{main_bigan}
    Suppose assumptions A\ref{dat_assumption} and B\ref{model2} hold and let $s_1 > \bar{d}_{\alpha_e}(\mu)$ and $s_2 > d^\ast_\beta(\mu)$. There exists constants $N$, $c$ that might depend on $d, \ell, \alpha_d, \alpha_e, \beta, \tilde{D}$ and $\tilde{E}$, such that, if $n \ge N$, we can choose the networks $\sE = \cR \cN (L_e, W_e, B_e, 2C_e)$ and $\sD = \cR \cN (L_d, W_d, B_d, 2C_d)$ with the network parameters chosen as
    \begin{align}
        L_e \le c  \log n,\, W_e \le c  n^{\frac{s_1}{2 \alpha_e (\beta \wedge 1) + s_1}} \log n, \, L_d \le c \log n,\, W_d \le c  n^{\frac{ \beta \ell }{\alpha_d(\beta \wedge 1)(s_2+\ell)}} \log n. \label{m45}
    \end{align}
    Then,
    \begin{equation}\label{m3}
        \E \|(\mu, (\hE_n)_\sharp \mu) - ((\hD_n)_\sharp \nu, \nu)\|_{\Psi} \lesssim  n^{-\frac{1}{\max\left\{2  + \frac{s_1}{\alpha_e(\beta \wedge 1)}, \frac{s_2 + \ell}{\beta} \right\}}} (\log n)^2. 
    \end{equation}
Furthermore, if the network parameters are chosen as
\begin{align}
    L_e \le c \log n,\, W_e \le c n^{\frac{s_1}{2 \alpha_e (\beta \wedge 1) + s_1}} \log n, \, L_d \le c \log m,\, W_d \le c m^{\frac{\ell}{2 \alpha_d(\beta \wedge 1) + \ell}} \log m, \label{e30}
\end{align}
   then, \begin{align}
       & \E \|(\mu, (\hE_{n,m})_\sharp \mu) - ((\hD_{n,m})_\sharp \nu, \nu)\|_{\Psi} \nonumber\\
       \lesssim & n^{-\frac{1}{\max\left\{2  + \frac{s_1}{\alpha_e(\beta \wedge 1)}, \frac{s_2 + \ell}{\beta} \right\}}} (\log n)^2 + m^{-\frac{1}{\max\left\{2 + \frac{d}{\alpha_g (\beta \wedge 1)}, \frac{\ell + d}{\beta}\right\}}} (\log m)^2. \label{m44}
   \end{align}
\end{restatable}
It should be noted that the dependence on $d$ in $N$ and $c$ may be influenced by a multitude of factors including the constants in Theorems 18 and 21 and thus, can potentially be exponential. Additionally, although the BiGAN problem might seem symmetric on the surface, there are subtle asymmetries that make the bounds on $W_e$ and $W_d$ asymmetric in \eqref{m45}. For example, for the estimation problem \eqref{bigan_est1}, one has full access to $\nu$ but only $n$ samples from $\mu$. The bounds for $W_e$ and $W_d$ in \eqref{e30} are, however, more symmetric since one has access to $m$ samples from $\nu$ instead of the entire distribution in the estimation problem \eqref{bigan_est2}, resulting in a bound that is more symmetric in $m$ and $n$ as shown in \eqref{e30}. Since one has the luxury to generate as many samples as one wants from $\nu$, increasing $m$ matches the rate of \eqref{m44} to \eqref{m3} as shown in the following corollary.
\begin{restatable}{cor}{cornew}\label{cor_m}
    Let $\vartheta = \frac{\max\left\{2  + \frac{\ell}{\alpha_d(\beta \wedge 1)}, \frac{d + \ell}{\beta} \right\}}{\max\left\{2  + \frac{s_1}{\alpha_e(\beta \wedge 1)}, \frac{s_2 + \ell}{\beta} \right\}} \vee 1$. Under the assumptions of Theorem~\ref{main_bigan}, if $m \ge n^\vartheta$, and the network parameters are chosen as
    \begin{align}
        L_e \le c \log n,\, W_e \le c n^{\frac{s_1}{2 \alpha_e (\beta \wedge 1) + s_1}} \log n, \, L_d \le c \log n,\, W_d \le c m^{\frac{\ell}{2 \alpha_d(\beta \wedge 1) + \ell}} \log n, \label{m46}
    \end{align}
   then, 
   \begin{equation}\label{m4}
       \E \|(\mu, (\hE_{n,m})_\sharp \mu) - ((\hD_{n,m})_\sharp \nu, \nu)\|_{\Psi} \lesssim  n^{-\frac{1}{\max\left\{2  + \frac{s_1}{\alpha_e(\beta \wedge 1)}, \frac{s_2 + \ell}{\beta} \right\}}} (\log n)^2 .
   \end{equation}
\end{restatable}
\paragraph{Implications for BiGANs}
    We note that when the true model for the BiGAN problem is assumed to be Lipschitz, i.e. if $\alpha_d = \alpha_e =1$ in B\ref{model2} and we take the $\|\cdot\|_{\text{BL}}$-metric i.e. the IPM w.r.t. $\Psi = \sH^1(\Real^{d+\ell}, \Real,1)$, then, the excess risk for the BiGAN problem roughly scales as $\tilde{\cO}\left(n^{-1/\left(\max\left\{2  + \bar{d}_1(\mu),  d^\ast_1(\mu) + \ell \right\}\right)}\right)$, barring the log-factors. Thus, the expected excess risk scales at a rate of $\tilde{\cO}\left(n^{-1/\left(\bar{d}_1(\mu) \vee d^\ast_1(\mu) + \ell \vee 2\right)}\right)$. In practice, one usually takes $\ell \ge 2$, and by the fact that $\bar{d}_1(\mu) \le d^\ast_1(\mu)$ (Proposition \ref{lem1} (a)), we observe that this rate is at most  $\tilde{\cO}\left(n^{-1/( d^\ast_1(\mu) + \ell)}\right)$, which is akin to the minimax estimation rate for estimating the joint distribution of $\mu \otimes \nu$ on $\Real^{d+\ell}$ \citep{niles2022minimax,singh2018minimax}. 
    
    Further, we note that Assumption B\ref{model2}  implies that the supports of $\mu$ and $\nu$ are homeomorphic, which is rather restrictive. The proof of this result can be found in Appendix~\ref{prop_19_pf}.
    \begin{restatable}{propos}{propnineteen}\label{prop19}
    Under Assumption B\ref{model2}, the supports of $\mu$ and $\nu$ are homeomorphic.
    \end{restatable}
    It would be interesting to see if this assumption could be lifted to derive an analogous result akin to Theorem \ref{main 2} for the BiGAN problem as well. The main hindrance is that although for $\hat{\mu}_n$, one can construct $D$ such that $D_\sharp \nu \approx \hat{\mu}_n$, the same cannot be said the other way around, i.e. there might not exist a $E$ such that $E_\sharp \hat{\mu}_n \approx \nu$. Further, both these approximations have to hold jointly, which is even more difficult to show. We believe that this would be an interesting direction for future research.
\paragraph{Comparison of Rates with Recent Literature}
    We also note that for the GAN problem, the expected error for estimating the target density through $\hG_\sharp \mu$ roughly scales as $\E \|\mu- \hG_\sharp \nu\|_{\sH^\beta} = \cO\left(n^{-\beta/d^\ast_\beta(\mu)} \right)$. For high-dimensional data, it can be expected that $\overline{\text{dim}}_M(\mu) \gg 2 \beta$, which (by Lemma~\ref{lem_weed}) would imply that, $d^\ast_\beta(\mu) \le \overline{\text{dim}}_M(\mu)$, with strict inequality holding in many cases (see Examples \ref{eg1} and \ref{eg2}). We thus observe that the derived rate,  derived rate is faster than 
    the ones derived by \citet{JMLR:v23:21-0732} and \citet{chen2020distribution}, who showed that the error rates scale as $\tilde{\cO}\left(n^{-\beta/\overline{\text{dim}}_M(\mu)} \vee n^{-1/2}\right)$ and $\tilde{\cO}\left(n^{-\beta/(2 \beta + d)} \right)$, respectively.  
 We note that the removal of the excess log-factor is an artifact of the requirement that $s>d^\ast_\beta(\mu)$ and one can similarly do away with the log-factors in the rates derived by \citet{JMLR:v23:21-0732}. 
\paragraph{Inference for Data supported on a Manifold}
Recall that we call a set $\sA$ is $\tilde{d}$-regular w.r.t. the $\tilde{d}$-dimensional Hausdorff measure (see Definition~\ref{defn_hausdorff}) $\fH^{\tilde{d}}$, if \(\fH^{\tilde{d}}(B_\varrho(x, r)) \asymp r^{\tilde{d}},\)
for all $x \in \sA$ (see Definition 6 of \cite{weed2019sharp}). It is known \citep[Proposition 8]{weed2019sharp} that if $\operatorname{supp}(\mu)$ is $d$ regular and $\mu \ll \fH^{\tilde{d}}$, for $\beta \in [1, \tilde{d}/2]$, $d_\ast(\mu) = d^\ast_\beta(\mu) = \tilde{d}$. Thus, by Theorem~\ref{main_gan}, the error rates for GANs scale at a rate of $\tilde{\cO}(n^{-\beta/\tilde{d}})$. Since, compact $\tilde{d}$-dimensional differentiable manifolds are $\tilde{d}$-regular \citep[Proposition 9]{weed2019sharp}, this implies that when the support of $\mu$ is a compact $\tilde{d}$-dimensional differentiable manifold, the error rates scale as $\tilde{\cO}(n^{-\beta/\tilde{d}})$, which recovers a similar result as that by \cite{dahal2022deep} as a special case. A similar result holds when $\text{supp}(\mu)$ is a nonempty, compact convex set spanned by an affine space of dimension $\tilde{d}$; the relative boundary of a nonempty, compact convex set of dimension $\tilde{d}+1$; or a self-similar set with similarity dimension $\tilde{d}$.

\paragraph{Network Sizes}
    We note that Theorems \ref{main_gan} and \ref{main_bigan} suggest that the depths of the encoders and decoders can be chosen as the order of $\log n$ and the number of weights can be chosen as some exponent of $n$ and $m$. The degrees of these exponents only depend on $\bar{d}_\beta(\mu)$, $d^\ast_\beta(\mu)$ and $\ell$ i.e. the dimension of the problem and the smoothness of the true model. The bounds on the number of weights in Theorems~\ref{main_gan} and~\ref{main_bigan} also show that one requires a smaller network if one lets $\alpha_e, \, \alpha_g$ and $\alpha_d$ grow, i.e. when the true encoder, generator/decoder are very smooth and well-behaved. This is because the optimal choices for $W_e$ and $W_g$ (or $W_d$) in terms of $n$ in Theorems \ref{main_gan} and \ref{main_bigan}, decreases as $\alpha_e, \, \alpha_g$ and $\alpha_d$ increase. This is quite expected that one requires less complicated architectures when the target densities are smooth enough.
    
\paragraph{Scaling of $m$} We note that in Theorem~\ref{main_gan}, for the estimator $\hat{G}_{n,m}$,  one roughly requires $m \gtrsim n^{\frac{\beta}{d^\ast_\beta(\mu)} \max\{2+\ell/(\alpha_g (\beta\wedge 1)), d/\beta\}}$. If the true generator is Lipschitz, $\alpha_g = 1$ and one considers the difference in the BL-metric, i.e. $\beta =1$, then, $m \gtrsim n^{\max\{d, (2+\ell)\}/d^\ast_\beta(\mu)} = n^{d/d^\ast_\beta(\mu)}$, for all practical purposes as $d \gg \ell$.  Thus, the dimension affects the sample scaling; higher $d$ values necessitate more generated samples to achieve equivalent accuracy. This dimension-dependent scaling also appears in recent works by \citet{JMLR:v23:21-0732}. It would be interesting to explore whether this dependency can be reduced in future research. Nonetheless, since practitioners can generate as many samples as needed from $\nu$, this dependency may be less limiting in practical settings.

The remainder of Section \ref{theo ana} discusses the proofs of Theorems \ref{main_gan} and \ref{main_bigan} by first decomposing the error into misspecification and generalization terms and individually bounding them.  
\subsection{Proof of the Main Results}\label{pf_main_result}
\subsubsection{Error Decomposition}\label{errord}
As a first step towards deriving bounds on the expected risks for both the GAN and BiGAN problems, we now derive the following oracle inequality that bounds this excess risk in terms of the approximation error and a generalization gap with proof in Appendix \ref{p_oracle}. A similar result was derived by \citet{JMLR:v23:21-0732} for analyzing the GAN problem.
\begin{restatable}[Oracle inequalities for GANs and BiGANs]{lem}{pforacle} \label{oracle}
    Suppose that $\Phi = \sH^\beta(\Real^d, \Real, 1)$ and $\Psi = \sH^\beta(\Real^{d+\ell}, \Real, 1)$, then 
    
    (a) if $\hG_n$ and $\hG_{n,m}$ are the  generator estimates defined in \eqref{gan_est}, the following hold:
    \begingroup
    \allowdisplaybreaks
    \begin{align}
       &  \|\mu - (\hG_n)_\sharp \nu\|_{\Phi} \le \inf_{G \in \sG} \|\mu - G_\sharp \nu\|_{\Phi} + 2 \|\mu -\hmu_n\|_{\Phi}, \label{o1}\\
        &  \|\mu - (\hG_{n,m})_\sharp \nu\|_{\Psi} \le \inf_{G \in \sG} \|\mu - G_\sharp \nu\|_{\Phi} + 2 \|\mu -\hmu_n\|_{\Phi} + 2 \|\nu - \hnu_m\|_{\Phi \circ \sG}. \label{o2}
    \end{align} 
    \endgroup
    
    (b) if $(\hD_n, \hE_n)$ and $(\hD_{n,m}, \hE_{n,m})$ are the  decoder-encoder estimates defined in \eqref{bigan_est1} and \eqref{bigan_est2}, respectively, the following hold:
    \begingroup
    \allowdisplaybreaks
    \begin{align}
          \|(\mu, (\hE_n)_\sharp \mu) - (\hD_n)_\sharp \nu, \nu)\|_{\Psi}\le & \inf_{D \in \sD, \, E \in \sE}\|(\mu, E_\sharp \mu) - (D_\sharp \nu, \nu)\|_{\Psi} + 2 \|\hmu_n - \mu\|_{\sF_1}, \label{o3}\\
         \|(\mu, (\hE_{n,m})_\sharp \mu) - (\hD_{n,m})_\sharp \nu, \nu)\|_{\Psi} \le & \inf_{D \in \sD, \, E \in \sE}\|(\mu, E_\sharp \mu) - (D_\sharp \nu, \nu)\|_{\Psi} + 2 \|\hmu_n - \mu\|_{\sF_1} \nonumber \\
       & + 2 \|\hnu_m - \nu\|_{\sF_2}, \label{o4}
    \end{align}
    \endgroup
    where, $\sF_1 = \{\psi(\cdot, E(\cdot)): \psi \in \Psi, \, E \in \sE\}$ and $\sF_2 = \{\psi(D(\cdot),\cdot): \psi \in \Psi, \, D \in \sD\}$.
\end{restatable}
Using Lemma~\ref{oracle}, we aim to bound each of the terms individually in the following sections. The misspecification errors are controlled by deriving a new approximation result involving ReLU networks, whereas, the generalization gap is tackled through empirical process theory. It is important to note that in Lemma \ref{oracle}, the term $\|\mu - \hat{\mu}_n\|_\Phi$, which scales at a rate of roughly $\cO\left(n^{-\beta/d^\ast_\beta(\mu)}\right)$ (by Theorem \ref{corr1}), dominates the other error terms as can be observed in the proof, leading to the error rate in Theorem \ref{main_gan}. 
\subsubsection{Bounding the Misspecification Error}\label{sec_mis}
The approximation capabilities of neural networks have received a lot of attention in the past decade. The seminal works of \cite{cybenko1989approximation} and  \cite{hornik1991approximation} delve into the universal approximation of networks with sigmoid-like activations to show that wide one-hidden-layer neural networks can approximate any continuous function on a compact set. With the recent advancements in deep learning, there has been a surge of research investigating the approximation capabilities of deep neural networks \citep{yarotsky2017error,petersen2018optimal,shen2022optimal,JMLR:v21:20-002,schmidt2020nonparametric}.

In this section, we show how a ReLU network with a large enough depth and width can approximate any function lying on a low-dimensional structure.  We suppose that $f \in \sH^\alpha \left( \Real^{d} , \Real, C\right)$ and $\gamma$ is a measure on $\Real^d$. We show that for any $\epsilon>0$ and $s>\bar{d}_{\alpha p }(\gamma)$, we can find a ReLU network $\hat{f}$ of depth at most $\cO(\log(1/\epsilon))$ and number of weights at most $\cO(\epsilon^{-s/\alpha}\log(1/\epsilon))$ such that $\|f - \hat{f}\|_{\fL_p(\gamma)} \le \epsilon$. Note that when $\text{supp}(\mu)$ has a finite Minkowski dimension, by Proposition~\ref{lem1} (c), we observe that $\bar{d}_{\alpha p} \le \overline{\text{dim}}_M(\mu)$. Thus,  the number of weights required for an $\epsilon$-approximation, in the $\fL_p$ sense, requires at most $\cO(\epsilon^{-\overline{\text{dim}}_M(\mu)/\alpha} \log(1/\epsilon))$, recovering a similar result as derived by \citet{JMLR:v21:20-002} as a special case. Note that the required number of weights for low-dimensional data, i.e. when $\bar{d}_{\alpha p}(\gamma) \ll d$, is much smaller than $\cO(\epsilon^{-d/\alpha} \log(1/\epsilon))$, that holds when approximating on the entire space w.r.t. $\ell_\infty$-norm \citep{yarotsky2017error,chen2019efficient}. The idea is to approximate the Taylor series expansion of the corresponding H\"{o}lder functions. The general proof technique was developed by \cite{yarotsky2017error}. Our result is formally stated in the following theorem.
\begin{restatable}{thm}{pfapprox}\label{approx}
    Suppose that $f \in \sH^\alpha(\Real^{d}, \Real,C)$, for some $C >0$ and let $s > \bar{d}_{\alpha p}(\mu)$. Then, we can find constants  $\epsilon_0$ and $a$, that might depend on $\alpha$, $d$ and $C$, such that, for any $\epsilon \in (0, \epsilon_0]$, there exists a ReLU network, $\hat{f}$ with $\cL(\hat{f}) \le a \log(1/\epsilon)$, $\cW(\hat{f}) \le a \log(1/\epsilon)  \epsilon^{-s/\alpha}$, $\cB(\hat{f}) \le a \epsilon^{-1/\alpha}$ and $\cR(\hat{f}) \le 2C$, that satisfies, 
    \(\|f-\hat{f}\|_{\fL_p(\gamma)} \le \epsilon .\) 
\end{restatable}

Applying the above theorem, one can control the model-misspecification error for GANs and BiGANs as follows. It is important to note that none of the approximations require the number of weights of the approximating network to increase exponentially with $d$, i.e. the dimension of the entire data space. 
 \begin{restatable}{lem}{lemapproxone}\label{approx_1}
     Suppose assumption~A\ref{model1} holds. There exists an $\epsilon_0$, such that, for any $0 < \epsilon \le \epsilon_0$, we can take $\sG = \cR \cN(L_g, W_g, B_g, 2C_g)$ such that  $L_g \lesssim \log(1/\epsilon)$, $W_g \lesssim \epsilon^{-\ell/\alpha_g} \log(1/\epsilon)$, $B_g \lesssim \epsilon^{-1/\alpha_g}$  and  $\inf_{G \in \sG} \|\mu - G_\sharp \nu\|_{\Phi} \lesssim \epsilon^{\beta \wedge 1} $.
 \end{restatable}
 
 \begin{restatable}{lem}{lemapproxtwo}\label{approx_2}
     Suppose assumption~B\ref{model2} holds and let $s>\bar{d}_{\alpha_e}(\mu)$. There exists $\epsilon_0>0$, such that, for any $0 < \epsilon_1, \epsilon_2 \le \epsilon_0$, we can take $\sE = \cR \cN(L_e, W_e, B_e, 2C_e)$ and  $\sD = \cR \cN(L_d, W_d, B_d, 2 C_d)$ such that 
     \begingroup
     \allowdisplaybreaks
     \begin{align*}
        & L_e \lesssim \log(1/\epsilon_e),\, W_e \lesssim \epsilon_e^{-s/\alpha_e} \log(1/\epsilon_e), \text{ and } B_e \lesssim \epsilon^{-1/\alpha_e} ;\\
        & L_d \lesssim \log(1/\epsilon_d), \, W_d \lesssim \epsilon_d^{-\ell/\alpha_d} \log(1/\epsilon_d),\text{ and } B_d \lesssim \epsilon_d^{-1/\alpha_d} .
     \end{align*}    
\endgroup
Then, $\inf_{G \in \sG, \, E \in \sE} \|(\mu, E_\sharp \mu) - (G_\sharp \nu, \nu)\|_{\Psi} \lesssim \epsilon_e^{\beta \wedge 1} + \epsilon_d^{\beta \wedge 1}$.
 \end{restatable}
\subsubsection{Generalization Gap}\label{sec_gen}
The third step to bounding the excess risk is to bound the generalization gaps w.r.t. the function classes discussed in Lemma~\ref{oracle}. 
To do so, we first derive a bound on the metric entropy of networks from $\Real^d \to \Real^{d^\prime}$ that have piece-wise polynomial activations, extending the results of \cite{bartlett2019nearly} to vector-valued networks. 
Recall the notation $\cF_{|_{X_{1:n}}} = \{(f(X_1), \dots, f(X_n))^\top \in \Real^{n \times d^\prime}: f \in \cF\}$. 
\begin{restatable}{lem}{pflemapr}\label{lem_apr_18}
    Suppose that $n \ge 6$ and $\cF: \Real^d \to \Real^{d^\prime}$ be a class of bounded neural networks with depth at most $L$ and the number of weights at most $W$. Furthermore, the activation functions are piece-wise polynomial activation with the number of pieces and degree at most $k \in \mathbb{N}$. Then, there exists positive constants $\theta$  and $\epsilon_0$, 
    such that, if $n \ge \theta (W + 6 d^\prime + 2 d^\prime L ) (L+3) \left(\log (W + 6 d^\prime + 2 d^\prime L) + L + 3\right)$ and $\epsilon\in (0,\epsilon_0]$,  \[\log \cN(\epsilon; \cF_{|_{X_{1:n}}}, \ell_\infty) \lesssim (W + 6 d^\prime + 2 d^\prime L ) (L+3) \left(\log (W + 6 d^\prime + 2 d^\prime L) + L + 3\right) \log\left(\frac{n d^\prime}{\epsilon }\right), \]
    where $d^\prime$ is the output dimension of the networks in $\cF$.
\end{restatable}
As a corollary of the above result, we can bound the metric entropies of the function classes in Lemma~\ref{oracle} as a function of the number of samples used and the size of the networks classes $\sG$ and $\sE$. Apart from appealing to Lemma~\ref{lem_apr_18}, one also uses the bound on the metric entropy of $\beta$-H\"{o}lder functions as in Theorem \ref{kt}.  We recall that,
   \( \sF_1 =  \left\{\psi(\cdot, E(\cdot)): \psi \in \Psi, \, E \in \sE\right\} \quad \text{and} \quad \sF_2 = \left\{\psi(G(\cdot),\cdot): \psi \in \Psi, \, G \in \sG\right\}.\)
\begin{restatable}{cor}{pfcorone}\label{cor1}
    Suppose that $\sD= \cR\cN(L_d, W_d, B_d, 2C_d)$ and $\sG = \cR\cN(L_g, W_g, B_g, 2C_g)$ with $L_d, L_g \ge 3$, $W_g \ge 6d + 2 d L_g$ and $W_d \ge 6d + 2 d L_d$.    Then, there is a constant $c$, such that, $m \ge c \left( W_g L_g ( \log W_g + L_g)\right) \vee \left( W_d L_d ( \log W_d + L_d)\right)$, then, 
    \begingroup
    \allowdisplaybreaks
    \begin{align*}
    \log \cN \left(\epsilon; (\Psi \circ G)_{|_{Z_{1:m}}}, \ell_\infty\right) \lesssim & \, \epsilon^{-\frac{d}{\beta}} + W_g L_g \left( \log W_g +   L_g\right) \log \left(\frac{m d}{\epsilon }\right),\\
         \log \cN \left(\epsilon; (\sF_2)_{|_{Z_{1:m}}}, \ell_\infty\right) \lesssim & \, \epsilon^{-\frac{d + \ell}{\beta}} + W_d L_d \left( \log W_d +   L_d\right) \log \left(\frac{m d}{\epsilon }\right).
    \end{align*}
    \endgroup
\end{restatable}

We note that though the above metric entropies depend exponentially on the data dimension $d$, this is not a problem in finding the generalization error as this exponential dependence is only in terms of $m$, the number of generated fake samples, which the practitioner can increase to tackle this curse of dimensionality. Using Corollary~\ref{cor1}, we can bound stochastic errors of these function classes in Lemma~\ref{gen}.
\begin{restatable}{lem}{pfgen}\label{gen}
 Suppose that $\sD= \cR\cN(L_d, W_d, B_d, 2C_d)$ and $\sG = \cR\cN(L_g, W_g, B_g, 2C_g)$ with $L_d, L_g \ge 3$, $W_g \ge 6d + 2 d L_g$ and $W_d \ge 6d + 2 d L_d$.    Then, there is a constant $c$, such that, $m \ge c \left( W_g L_g ( \log W_g + L_g)\right) \vee \left( W_d L_d ( \log W_d + L_d)\right)$, then,
 \begingroup
 \allowdisplaybreaks
    \begin{align*}
    \E \|\hnu_m - \nu\|_{\Phi \circ \sG} \lesssim & \, m^{-\frac{\beta}{d }} \vee m^{-1/2} \log m + \sqrt{\frac{W_g L_g (\log W_g + L_g) \log(md)}{m}},\\
       \E \|\hnu_m - \nu\|_{\sF_2} \lesssim &\,  m^{-\frac{\beta}{d + \ell}} \vee (m^{-1/2} \log m) + \sqrt{\frac{W_d L_d (\log W_d + L_d) \log(md)}{m}}. 
    \end{align*}
    \endgroup
    Here $c$ is the same as in Corollary \ref{cor1}.
\end{restatable}

Next, we focus on deriving a uniform concentration bound w.r.t. the function class $\sF_1$. This is a little trickier, compared to the ones derived in Lemma~\ref{gen}, in the sense that one does not have direct control over the $\ell_\infty$-metric entropy of the function class $\Psi$ (and in turn, $\sF_1$) in terms of the Wasserstein dimension of the data. We resolve the issue by performing a one-step discretization and appealing to the sub-Gaussian property of the data for a fixed $E \in \sE$. The proof is detailed in Appendix \ref{app_g4}. 
\begin{restatable}{lem}{lemtf}
    \label{Lem_3.7}
    Suppose that $\sE = \cR\cN(L_e, W_e, B_e,2 C_e)$. Then, for any $s>d^\ast_{\beta}(\mu)$, we can find an $n^\prime \in \mathbb{N}$, such that if $n \ge n^\prime$, 
    \[\E \|\hmu_n - \mu\|_{\sF_1} \lesssim n^{-1/2}\left( W_e\log \left(2L_eB_e^{L_e} (W_e+1)^{L_e} n^{\frac{1}{2 (\beta \wedge 1)}}\right) \right)^{1/2} + n^{-\frac{\beta}{s + \ell}}.\]
\end{restatable}

To complete the proof, the idea is to combine the pieces discussed throughout Sections~\ref{sec_mis} and~\ref{sec_gen}. The goal is to choose the sizes of the neural network classes, i.e. the class of generators, encoders and decoders, such that both the misspecification errors and generalization errors are small. A large network size implies the misspecification error will be small enough but will make the bounds in Lemma~\ref{gen} loose and vice versa. Since the network sizes are expressed in terms of $\epsilon_e$, $\epsilon_g$ and $\epsilon_d$ in Lemmas \ref{approx_1} and \ref{approx_2}, expressing them in terms of $n$ in an optimal way that serves the purpose of finding a trade-off between the two errors to minimize the bound on the excess risk in Lemma~\ref{oracle}. The detailed proofs are given in Appendix \ref{pf_main}.

\section{Optimal Bounds for GANs with Interpolating Generators}
\label{minimax bounds}
In this section, we show that GAN estimates can (almost) achieve the minimax optimal rate for estimating distributions with a low intrinsic dimensional structure. To generalize the scenario further, we drop assumption A\ref{model1} and only work with A\ref{dat_assumption}. We replace assumption A\ref{model1} and assume that $\nu$ is an absolutely continuous distribution on $[0,1]^\ell$. This is satisfied by the uniform and normal distributions on the latent space, which are the commonly used choices in practice. It is important to note that lifting assumption A\ref{model1}, requires that the generator network have many more parameters than compared to the case when the target distribution is smooth (i.e. under A\ref{model1}), which is expected. Following a similar analysis as done in Section \ref{theo ana}, we can arrive at the following theorem, which states that the generator network can be selected in such a way as to obtain a rate of convergence of roughly, $\cO\left(n^{-\beta/d^\ast_\beta(\mu)}\right)$, with proof in Appendix~\ref{pf_main_2}. The networks class in Theorem \ref{main 2} is constructed so that some members of the class can (almost) linearly interpolate the data, as opposed to the Taylor series approximation used in the proof of Theorem \ref{approx}. Compared  to Theorem \ref{main_gan}, the networks required in Theorem \ref{main 2} are shallower, with a constant depth as opposed to a $\cO(\log n)$ depth. However, one does not have any control over the maximum value of the weight of these nearly interpolating networks.
\begin{restatable}{thm}{maintwo}\label{main 2}
   Suppose Assumption~A\ref{dat_assumption}  holds and let $\nu$ be absolutely continuous on $[0,1]^\ell$. Then, if $s > d^\ast_\beta(\mu)$, one can choose  $\sG = \cR \cN (L_g, W_g,\infty, 1)$ with the network parameters as $L_g \ge 2$ as a constant and $W_g \asymp n$, such that
    \( \E \|\mu - (\hG_n)_\sharp \nu\|_{\Phi} \lesssim  n^{-\beta/s} .\)
   Furthermore, if $m \ge n^{d/s+1}$,   then
   \( \E \|\mu - (\hG_{n,m})_\sharp \nu\|_{\Phi} \lesssim  n^{-\beta/s} .\)
\end{restatable}

To understand whether the bounds in Theorem \ref{main 2} are optimal, we derive a minimax lower bound for the expected excess risk for the GAN problem for estimating distributions whose upper Wasserstein dimension is upper bounded by some constant. Assuming $d \ge 2 \beta$, we fix a constant $d^\star \in [2 \beta,d]$ and  consider the family of probability measures, 
\[ \fM^{d^\star, \beta} = \{\mu \in \Pi_{[0,1]^d}: d^\ast_\beta(\mu) \le d^\star\},\]
the set of all distributions on $[0,1]^d$, whose upper Wasserstein dimension is at most $d^\star$. We note that this collection of distributions contains distributions on specific manifolds of dimension $d^\star$ or less. Here, we use the notation $\Pi_{\cA}$ to denote the set of all probability measures on $\cA$. The minimax expected risk for this problem is given by,
\[\mathfrak{M}_n = \inf_{\hat{\mu}} \sup_{\mu \in \fM^{d^\star, \beta}} \E_\mu \|\hat{\mu} - \mu\|_{\sH^\beta(\Real^d, \Real, 1)},\]
where the infimum is taken over all measurable estimates of $\mu$, i.e. on $\{\hat{\mu}: X_{1:n} \to \Pi_{[0,1]^d}: \hat{\mu} \text{ is measurable}\}$.  Here, we write $\E_\mu$ to denote that the expectation is taken with respect to the joint distribution of $X_1, \dots, X_n$, which are i.i.d. $\mu$. $X_{1:n}$ denotes the data $X_1, \dots, X_n$. Theorem \ref{thm_30} characterizes this minimax rate, which states that the rate cannot be made any faster than roughly, $\cO\left(n^{-\beta/d^\star}\right)$.

\begin{restatable}{thm}{thmtwentynine}\label{thm_30}
    Suppose that $d\ge 2\beta$ and let $d^\star \in [2 \beta , d]$. Then for any $s < d^\star $,
    \[ \mathfrak{M}_n = \inf_{\hat{\mu}} \sup_{\mu \in \fM^{d^\star, \beta}} \E_\mu \|\hat{\mu} - \mu\|_{\sH^\beta(\Real^d, \Real, 1)}  \gtrsim n^{-\beta/s},\]
    where the infimum is taken over all measurable estimates of $\mu$, based on the data.
\end{restatable}
By Theorem, \ref{main 2}, we observe that for any $\mu \in \fM^{d^\star, \beta}$ and any $\overline{s} > d^\star$, $\E_\mu \|\hat{G}_\sharp \nu - \mu\|_{\sH^\beta(\Real^d, \Real, C)} \lesssim n^{-\beta/\overline{s}}$. Thus, in $\fM^{d^\star}$, the GAN estimator, $\hat{G}_\sharp \nu$, roughly achieves the optimal rate.

Suppose that $\fM_{d_\star} = \{\mu: d_\ast(\mu) \le d_\star\}$, then it is clear that $\fM^{d_\star, \beta} \subseteq \fM_{d_\star}$ by Proposition~\ref{lem1}.  Hence, if $ 2 \beta \le d_\star \le d$ and $\underline{s} < d_\star$ then, 
\[ \inf_{\hat{\mu}} \sup_{\mu \in \fM_{d_\star}} \E_\mu \|\hat{\mu} - \mu\|_{\sH^\beta(\Real^d, \Real, 1)}  \ge \inf_{\hat{\mu}} \sup_{\mu \in \fM^{d_\star, \beta}} \E_\mu \|\hat{\mu} - \mu\|_{\sH^\beta(\Real^d, \Real, 1)} \gtrsim n^{-\beta/\underline{s}}\]
Following a similar proof as that of Theorem~\ref{thm_30}, one can do away with the condition of $d_\star \ge 2\beta$ and arrive at the following theorem. 

\begin{restatable}{thm}{thmthirty}\label{thm_31}
    Suppose $d_\star \in [0, d]$. Then for any $s < d_\star $, \[ \mathfrak{M}_n = \inf_{\hat{\mu}} \sup_{\mu \in \fM_{d_\star}} \E_\mu \|\hat{\mu} - \mu\|_{\sH^\beta(\Real^d, \Real, 1)}  \gtrsim n^{-\beta/s} \vee n^{-1/2} , \]
    where the infimum is taken over all measurable estimates of $\mu$, based on the data.
\end{restatable}

 It is worth observing that for the special case of $\beta =1$, i.e. under the Wasserstein-1 distance, this rate matches with the ones derived in literature \citep{niles2022minimax, singh2018minimax}, i.e. $\cO(n^{-1/d})$. 

\section{Discussions and Conclusion}\label{conclusion}
 This paper delves into the theoretical properties of GANs and Bi-directional GANs when the data exhibits an intrinsically low-dimensional structure within a high-dimensional feature space. We propose to characterize the low-dimensional nature of the data distribution through its upper Wasserstein dimension and the so-called $\alpha$-entropic dimension, which we develop by extending Dudley's notion of entropic dimension \citep{dudley1969speed}. Specifically, we not only show that the classical result by \citet{kolmogorov1961} can be strengthened by incorporating this entropic dimension in the metric entropy but also show that the convergence rates of the empirical distribution to the target population in the $\beta$-H\"{o}lder IPM scales as $\cO\left(n^{-\beta/d^\ast_\beta(\mu)}\right)$, extending the results by \cite{weed2019sharp}. Furthermore, we improve upon the existing results on the approximation capabilities of ReLU networks.  By balancing the generalization gap and approximation errors, we establish that under the assumption that the true generator and encoders are H\"{o}lder continuous, the excess risk in terms of an H\"{o}lder IPM can be bounded in terms of the ambient upper Wasserstein dimension of the target measure and the latent space. This bypasses the curse of dimensionality of the full data space, strengthens the known results in GANs, and leads to novel bounds for BiGANs.  The derived results also match the sharp convergence rates for the empirical distribution available in the optimal transport literature. We also show that the GAN estimate of the target distribution can roughly achieve the minimax optimal rates for estimating intrinsically low-dimensional distributions.

While our results provide insights into the theoretical properties of GANs and Bi-directional GANs, it is important to acknowledge that estimating the full error of these generative models used in practice involves considering an optimization error term. Unfortunately, accurately estimating this term remains a challenging task due to the non-convex and complex nature of the minimax optimization. However, it is worth pointing out that our error analyses are independent of optimization and can be seamlessly combined with optimization analyses. An additional area that we have not addressed in this paper is that the $\beta$-H\"{o}lder discriminators are typically also realized by ReLU networks in practice. While this is an important consideration, it requires imposing restrictions on the networks to closely resemble smooth functions, which is typically achieved through various regularization techniques. However, ensuring that these regularization techniques provably lead to well-behaved discriminators remains a challenging task from a theoretical perspective. We believe that this is a promising avenue for future research. One can also attempt to extend our approximation results to more general functions representable on some other basis than non-smooth functions such as Besov spaces or Fourier series approximations, for instance using ideas of \citet{suzuki2018adaptivity} or \citet{bresler2020sharp}. 

\section*{Acknowledgments}
We gratefully acknowledge the support of the NSF and the Simons Foundation for the Collaboration on the Theoretical Foundations of Deep Learning through awards DMS-2031883 and \#814639 and the NSF's support of FODSI through grant DMS-2023505.

\newpage


\appendix
\section*{Appendices}


\tableofcontents
\section{Proofs from Section \ref{intrinsic}}\label{p_intrinsic}
This section provides the proofs of the results from Section \ref{intrinsic}. For notational simplicity, we write $\text{fat}(S, \epsilon) = \{y: \inf_{x \in S} \varrho(x, y) \le \epsilon\}$ for the $\epsilon$-fattening of the set $S$, w.r.t. the metric $\varrho$.
\subsection{Proof of Proposition~\ref{lem1}}\label{p_lem1}
Before we prove Proposition \ref{lem1}, we first show the following key lemma.
\begin{lem}\label{lem_b1}
    For any $0 < \tau_1 \le \tau_2 < 1$, $\sN_\epsilon(\mu, \tau_1) \ge \sN_\epsilon(\mu, \tau_2)$.
\end{lem}
\begin{proof}
It is easy to see that $\{S: \mu(S) \ge 1 - \tau_1\} \subseteq \{S: \mu(S) \ge 1 - \tau_2\} $. Hence, $\sN_\epsilon(\mu, \tau_1) = \inf\{\cN(\epsilon;S, \varrho): \mu(S) \ge 1 - \tau_1\} \ge \inf\{\cN(\epsilon;S, \varrho): \mu(S) \ge 1 - \tau_2\} = \sN_\epsilon(\mu, \tau_2)$.  
\end{proof}
We are now ready to prove Proposition~\ref{lem1} as shown below.

\pflemone*

\begin{proof}
\textbf{Proof of part (a)}: Suppose that  \[\sA =  \left\{s \in (2\alpha, \infty) : \limsup_{\epsilon \downarrow 0} \frac{\log \sN_{\epsilon}\left(\mu, \epsilon^{\frac{s\alpha}{s-2\alpha}}\right)}{\log(1/\epsilon)} \leq s \right\}.\] Fix $s \in \sA$. By definition, $s>2\alpha$. Since for $\epsilon<1$, $\epsilon^\alpha \ge \epsilon^{\frac{s\alpha}{s-2\alpha}}$. Thus, by Lemma~\ref{lem_b1}, $\sN_\epsilon(\mu, \epsilon^\alpha) \le \sN_\epsilon\left(\mu, \epsilon^{\frac{s\alpha}{s-2\alpha}}\right)$, when $\epsilon<1$. Fix $\tau \in (0,1)$. Hence, \[\bar{d}_\alpha(\mu) = \limsup_{\epsilon \downarrow 0} \frac{\log \sN_\epsilon(\mu, \epsilon^\alpha) }{-\log \epsilon} \le \limsup_{\epsilon \downarrow 0} \frac{\log \sN_\epsilon(\mu, \epsilon^{\frac{s\alpha}{s-2\alpha}}) }{-\log \epsilon} \le s.\]
Taking infimum over $s \in \sA$ gives us, $\bar{d}_\alpha(\mu) \le d^\ast_\alpha(\mu)$.

To observe the other implication, we first note that for any fixed $\tau \in (0,1)$,  $\epsilon^\alpha \le \tau$, if $\epsilon \le \tau^{1/\alpha}$. Thus, by Lemma~\ref{lem_b1}, $\sN_\epsilon(\mu, \epsilon^\alpha) \ge \sN_\epsilon(\mu, \tau)$, if $\epsilon \le \tau^{1/\alpha}$. Hence,
\[\bar{d}_\alpha(\mu) = \limsup_{\epsilon \downarrow 0} \frac{\log \sN_\epsilon(\mu, \epsilon^\alpha) }{-\log \epsilon} \ge \limsup_{\epsilon \downarrow 0} \frac{\log \sN_\epsilon(\mu, \tau) }{-\log \epsilon} \ge \liminf_{\epsilon \downarrow 0} \frac{\log \sN_\epsilon(\mu, \tau) }{-\log \epsilon}. \]
Taking limit on both sides as $\tau \downarrow 0$, gives us, $\bar{d}_\alpha(\mu) \ge d_\ast(\mu)$.

\textbf{Proof of part (b)}: If $\alpha_1 \le \alpha_2$, then, we observe that $\sN_\epsilon(\mu, \epsilon^{\alpha_1}) \le \sN_\epsilon(\mu, \epsilon^{\alpha_2})$, for $\epsilon<1$. Thus, \(\bar{d}_{\alpha_1}(\mu) = \limsup_{\epsilon \downarrow 0} \frac{\log \sN_\epsilon(\mu, \epsilon^{\alpha_1}) }{-\log \epsilon} \le \limsup_{\epsilon \downarrow 0} \frac{\log \sN_\epsilon(\mu, \epsilon^{\alpha_2}) }{-\log \epsilon} =  \bar{d}_{\alpha_2}(\mu).\)

\textbf{Proof of part (c)}: For any $\alpha>0$, $\sN_\epsilon(\mu,\epsilon^\alpha) \le  \sN_\epsilon(\mu,0) =\cN(\epsilon; \text{supp}(\mu), \varrho) $. Thus taking $\limsup$ as $\epsilon \downarrow 0$ gives us the result. 

 \textbf{Proof of part (d)} Suppose that $\sA = \left\{s \in (2\alpha, \infty) : \limsup_{\epsilon \downarrow 0} \frac{\log \sN_{\epsilon}\left(\mu, \epsilon^{\frac{s\alpha}{s-2\alpha}}\right)}{\log(1/\epsilon)} \leq s \right\}$. Let $0 < \epsilon < 1$, $s > \overline{\text{dim}}_P(\mu)$  
 and $\tau = \epsilon^{\frac{s \alpha}{s - 2 \alpha}}$. $S$ be such that $\mu(S) \ge 1- \tau$ and $\cN(\epsilon;S, \varrho) = \sN_\epsilon(\mu, \tau)$. We let $R = \text{diam}(S) \vee 1$. Let $\{x_1, \dots, x_M\}$ be an optimal $2\epsilon$-packing of $S \cap \text{supp}(\mu)$. By the definition of the upper packing dimension, for any $s > \overline{\text{dim}}_{P}(\mu) $ we can find $r_0 < 1$, such that, 
 \begin{align*}
     & \frac{\log \mu(B(x,r))}{\log r} \le s, \, \forall r \le r_0 \text{ and } x \in \text{supp}(\mu) \\
     \implies &  \mu(B(x,r)) \ge r^s, \, \forall r \le r_0 \text{ and } x \in \text{supp}(\mu).
 \end{align*}
Thus, if $\epsilon \le r_0$, \(1 \ge \mu\left(\cup_{i=1}^M B(x_i, \epsilon)\right) = \sum_{i=1}^M \mu\left( B(x_i, \epsilon)\right) \ge M \epsilon^s \implies M \le \epsilon^{-s}.\) By Lemma \ref{cov_pack}, we know that $\sN_\epsilon(\mu, \tau) =\cN(\epsilon; S, \varrho) \le M \le  \epsilon^{-s}$. Thus,
\[ \limsup_{\epsilon \downarrow 0} \frac{ \log \sN_\epsilon\left(\mu, \epsilon^{\frac{s \alpha}{s-2 \alpha}}\right)}{-\log \epsilon}  \le s \implies s \in \sA \implies d^\ast_\alpha(\mu) \le s.\]
Since $d^\ast_\alpha(\mu) \le s$, for all $s > \overline{\text{dim}}_P(\mu)$, we get, $d^\ast_\alpha(\mu) \le \overline{\text{dim}}_P(\mu)$. The inequality $\overline{\text{dim}}_P(\mu) \le \overline{\text{dim}}_{\text{reg}}(\mu)$ follows from \citet[Theorem 2.1]{fraser2017upper}.

\textbf{Proof of part (e)}: 

 Let $0 < \epsilon < 1$, $s > \overline{\text{dim}}_P(\mu)$  
 and $\tau = \epsilon^{\alpha}$. $S$ be such that $\mu(S) \ge 1- \tau$ and $\cN(\epsilon;S, \varrho) = \sN_\epsilon(\mu, \tau)$. We let $R = \text{diam}(S) \vee 1$. Let $\{x_1, \dots, x_M\}$ be an optimal $2\epsilon$-packing of $S \cap \text{supp}(\mu)$. By the definition of the upper packing dimension, for any $s > \overline{\text{dim}}_{P}(\mu) $ we can find $r_0 < 1$, such that, 
\(\frac{\log \mu(B(x,r))}{\log r} \le s, \, \forall r \le r_0 \text{ and } x \in \text{supp}(\mu) \implies  \mu(B(x,r)) \ge r^s, \, \forall r \le r_0 \text{ and } x \in \text{supp}(\mu).\)
Thus, if $\epsilon \le r_0$, \(1 \ge \mu\left(\cup_{i=1}^M B(x_i, \epsilon)\right) = \sum_{i=1}^M \mu\left( B(x_i, \epsilon)\right) \ge M \epsilon^s \implies M \le \epsilon^{-s}.\) By Lemma \ref{cov_pack}, we know that $\sN_\epsilon(\mu, \tau) =\cN(\epsilon; S, \varrho) \le M \le  \epsilon^{-s}$. Thus,
\[ \bar{d}_\alpha(\mu) = \limsup_{\epsilon \downarrow 0} \frac{ \log \sN_\epsilon\left(\mu, \epsilon^{\alpha}\right)}{-\log \epsilon}  \le s .\]
Since $\bar{d}_\alpha(\mu) \le s$, for all $s > \overline{\text{dim}}_P(\mu)$, we get, $\bar{d}_\alpha(\mu) \le \overline{\text{dim}}_P(\mu)$. The inequality $\overline{\text{dim}}_P(\mu) \le \overline{\text{dim}}_{\text{reg}}(\mu)$ follows from Theorem 2.1 of \cite{fraser2017upper}.

\textbf{Proof of part (f)}: Let $0 < \epsilon < 1$ and $S$ be such that $\mu(S) \ge 1-\tau$ and $\cN(\epsilon; S, \varrho) = \sN_\epsilon(\mu, \tau)$. We let $R = \text{diam}(S) \vee 1$. Let $\{x_1, \dots, x_N\}$ be an optimal $\epsilon$-net of $S \cap \text{supp}(\mu)$. By the definition of the upper regularity dimension, for any $s < \overline{\text{dim}}_{\text{reg}}(\mu) $ we observe that,
 $\mu(B(x_i, \epsilon)) \le \frac{\mu(B(x_i, R))}{C R^s} \epsilon^{s} \le \frac{(1-\tau)}{C R^s} \epsilon^{s}$. We observe that, $\mu(\cup_{i=1}^N B(x_i,\epsilon)) \ge \mu(S) \ge 1-\tau \implies \sum_{i=1}^N \mu( B(x_i,\epsilon)) \ge 1 -\tau$. This implies that,
 \[N \epsilon^s \frac{1-\tau}{C R^s}\ge \sum_{i=1}^N \mu( B(x_i,\epsilon)) \ge 1 -\tau \implies N \ge C R^s \epsilon^{-s}.\]
 Thus, \(\cN(\epsilon; S, \varrho) \ge C R^s (2\epsilon)^{-s} \implies   \liminf_{\epsilon \downarrow 0} \frac{ \log \sN_\epsilon(\mu, \tau)}{-\log \epsilon}  \ge s,\) which implies that \[(d_\ast(\mu) = \lim_{\tau \downarrow 0}\limsup_{\epsilon \downarrow 0} \frac{ \log \sN_\epsilon(\mu, \tau)}{-\log \epsilon}  \ge s.\]
 Hence, $d_\ast(\mu) \ge s$, for any $s < \overline{\text{dim}}_{\text{reg}}(\mu)$, which gives us the desired result.
 \end{proof}

 \subsection{Proof of Theorem \ref{kt}}

Next, we give a proof of Theorem \ref{kt}. Before proceeding, we consider the following lemma.
 \begin{lem}\label{lem_e1}
     For any $s > \bar{d}_{\alpha}(\mu)$, we can find an $\epsilon^\prime \in (0,1)$, such that if $\epsilon \in (0,\epsilon^\prime]$ and a set $S$, such that $\mu(S) \ge 1 - \epsilon^\alpha$ and $\cN(\epsilon; S, \varrho) \le \epsilon^{-s}$.
 \end{lem}
 \begin{proof}
     By definition of $\bar{d}_\alpha(\mu) = \limsup_{\epsilon \downarrow 0} \frac{\log \sN_\epsilon(\mu, \epsilon^\alpha)}{\log(1/\epsilon)}$, if $s>\bar{d}_\alpha(\mu)$, we can find an $\epsilon^\prime \in (0,1)$, such that if $\epsilon \in (0,\epsilon^\prime]$, $\frac{\log \sN_\epsilon(\mu, \epsilon^\alpha)}{\log(1/\epsilon)} \le s$, which implies that $\sN_\epsilon(\mu, \epsilon^\alpha) \le \epsilon^{-s}$. By definition of $\sN_\epsilon(\mu, \cdot)$, we can find $S$ that satisfies the conditions of the Lemma.
 \end{proof}
By construction, for $s> \bar{d}_{p\beta}(\mu)$, we can find $\epsilon^\prime>0$, such that if $\epsilon \in (0, \epsilon^\prime]$, we can find a bounded $S \subset \Real^d$, such that $\cN(\epsilon;S, \ell_\infty) \le  \epsilon^{-s}$ and $\mu(S) \ge 1 -  \epsilon^{p\beta}$. We fix such an $\epsilon \in (0, \epsilon^\prime]$.

Let $M = \sup_{x \in S} \|x\|_\infty$ and $K = \lceil \frac{M}{\epsilon}\rceil$. For any $\bi \in [K]^d$, let $\theta^{\bi} = (-M + i_1 \epsilon, \dots, -M + i_d \epsilon )$. We also let, $\sP_\epsilon = \{B_{\ell_\infty}(\theta^{\bi}, \epsilon): \bi \in [K]^d\}$. By construction, the sets in $\sP_\epsilon$ are disjoint. We first claim the following: 
\begin{lem}\label{lem_b2}
$|\{A \in \sP_\epsilon: A \cap S \neq \emptyset\} | \le 2^d \epsilon^{-s} $.
\end{lem}
\begin{proof}
Let, $r = \cN(\epsilon; S, \ell_\infty)$ and suppose that $\{a_1,\dots, a_r\}$ be an $\epsilon$-net of $S$ and  $\sP_\epsilon^\ast = \{ B_{\ell_\infty}(a_i, \epsilon): i \in [r]\}$ be an optimal $\epsilon$-cover of $S$. Note that each box in $\sP_\epsilon^\ast$ can intersect at most $2^d$ boxes in $\sP_\epsilon$. This implies that,
\(|\sP_\epsilon \cap S| \le \left| \sP_\epsilon \cap \left(\cup_{i=1}^r B_{\ell_\infty}(a_i, \epsilon)\right)\right| = \left| \cup_{i=1}^r \left(\sP_\epsilon \cap B_{\ell_\infty}(a_i, \epsilon)\right)\right|  \le 2^d r,\)
which concludes the proof. 
\end{proof}


\pfkt*

\begin{proof}
Let $\sI = \{\bi \in [k]^d: B_{\ell_\infty}(\theta^{\bi}, \epsilon) \in \sP_\epsilon  \text{ and }  B_{\ell_\infty}(\theta^{\bi}, \epsilon) \cap S \neq \emptyset \}$. For any $x \in \Real^d$, let, $h_{x,r}(y) = \one\{y \in B_{\ell_\infty}(x,r)\}$. Fix any $f \in \cH^\beta(\Real^d, \Real, C)$. We define, $P^f_{\theta}(x) = \sum_{|\bs| \le \lfloor \beta \rfloor} \frac{\partial^{\bs} f(\theta)}{\bs!} (x - \theta)^{\bs}$. Since, $|\partial^{\bs} f(\theta)| \le C$, we can define a $\delta$-net, $U_\delta$, of $[-C,C]$ of size at most $C/\delta$. We define the function class
\[\cF = \left\{ \sum_{\bi \in \sI }\sum_{|\bs| \le \lfloor \beta \rfloor} \frac{b_{\bs, \theta^{\bi}}}{\bs!} (x - \theta^{\bi})^{\bs} h_{\theta^{\bi},\epsilon}(x): b_{\bs, \theta^{\bi}} \in U_\delta\right\}.\]
Clearly, for any $f \in \cH^\beta(\Real^d, \Real, C)$, we can find, $\hat{f} \in \cF$, such that, \[\hat{f} = \sum_{\bi \in \sI }\sum_{|\bs| \le \lfloor \beta \rfloor} \frac{b_{\bs, \theta^{\bi}}}{\bs!} (x - \theta^{\bi})^{\bs} h_{\theta^{\bi}, \epsilon}(x),\] with $|b_{\bs, \theta^{\bi}} - \partial^{\bs} f(\theta^{\bi})|\le \delta$. Thus, for any $x \in \text{fat}(S,\epsilon)$, let $\theta $ be such that $\|\theta -x\| \le \epsilon$ and $\theta \in \{\theta^{\bi}: \bi \in \sI\}$.
\begingroup
\allowdisplaybreaks
\begin{align*}
  &  |f(x) - \hat{f}(x)| \\
  \le &  |\sum_{\bi \in \sI }\sum_{|\bs| \le \lfloor \beta \rfloor} \frac{b_{\bs, \theta^{\bi}}}{\bs!} (x - \theta^{\bi})^{\bs} h_{\theta^{\bi},\epsilon}(x)  - f(x)|\\
   = & \left|\sum_{|\bs| \le \lfloor \beta \rfloor} \frac{b_{\bs, \theta}}{\bs!} (x - \theta^{\bi})^{\bs} - f(x)\right|\\
     \le & \left|\sum_{|\bs| \le \lfloor \beta \rfloor} \frac{\partial^{\bs} f(\theta)}{\bs!} (x - \theta)^{\bs} - f(x)\right| + \left|\sum_{|\bs| \le \lfloor \beta \rfloor} \frac{b_{\bs, \theta}}{\bs!} (x - \theta)^{\bs} -\sum_{|\bs| \le \lfloor \beta \rfloor} \frac{\partial^{\bs}f(\theta)}{\bs!} (x - \theta)^{\bs}\right|\\
      \le & \left|\sum_{|\bs| = \lfloor \beta \rfloor} \frac{(x - \theta)^{\bs}}{\bs!}  ( \partial^{\bs} f(x) - \partial^{\bs} f(\theta^\prime))\right| + \left|\sum_{|\bs| \le \lfloor \beta \rfloor} \frac{|b_{\bs, \theta} - \partial^{\bs}f(\theta)|}{\bs!} (x - \theta)^{\bs} \right|\\
      \le & \sum_{|\bs| = \lfloor \beta \rfloor} \frac{\|x - \theta\|_\infty^{\bs}}{\bs!}  \|x - \theta\|_\infty^{\beta - \lfloor \beta \rfloor} + \sum_{|\bs| \le \lfloor \beta \rfloor} \frac{\delta}{\bs!} \\
       \le & \sum_{|\bs| \le \lfloor \beta \rfloor} \frac{\epsilon^\beta}{\bs!}   + \sum_{|\bs| \le \lfloor \beta \rfloor} \frac{\delta}{\bs!} \\
       \lesssim & \epsilon^\beta   + \delta \\
\end{align*}
\endgroup
In the above calculations, $\theta^\prime$ lies on the line-segment joining $x$ and $\theta$. We take $\delta = \epsilon^\beta$. It is easy to see that $|\cF| \le \left(\frac{C}{\delta}\right)^{|\sI| \lfloor \beta \rfloor^d}.$ By definition $\hat{f}(x) = 0$ if $x \not \in \text{fat}(S,\epsilon)$. Thus,
\begingroup
\allowdisplaybreaks
\begin{align*}
    \int |f(x) - \hat{f}(x)|^p d\mu(x) = & \int_{\text{fat}(S,\epsilon)} |f(x) - \hat{f}(x)|^p d\mu(x) + \int_{\text{fat}(S,\epsilon)^\complement} |f(x) - \hat{f}(x)|^p d\mu(x)\\
    \le & \sup_{x \in \text{fat}(S,\epsilon)} |f(x) - \hat{f}(x)|^p  + \int_{\text{fat}(S,\epsilon)^\complement} |f(x) |^p d\mu(x)\\
    \lesssim & \epsilon^{p\beta} + \mu(\text{fat}(S,\epsilon)^\complement)\\
     \le & \epsilon^{p\beta} + \mu(S^\complement)\\
     \le & \epsilon^{p\beta} + \epsilon^{p\beta}\\
     \lesssim & \epsilon^{p\beta}.
\end{align*}
\endgroup

Clearly, $\cF$ forms a $c \epsilon^\beta $-cover of $\cH^\beta(\Real^d, \Real,C)$ in the $\fL_p(\mu)$-norm, for some constant $c>0$. Thus, 
\[\log\cN \left(c\epsilon^\beta; \cH^\beta(\Real^d, \Real,C), \|\cdot\|_{\fL_p(\mu)} \right) \le |\sI| \lfloor \beta \rfloor^d \log(C/\delta) \lesssim \epsilon^{-s} \log(1/\epsilon) .\]
Replacing $\epsilon$ with $(\epsilon/c)^{1/\beta}$ gives us,
\[\log\cN \left(\epsilon; \cH^\beta(\Real^d, \Real,C), \|\cdot\|_{\fL_p(\mu)} \right) \lesssim \epsilon^{-s/\beta} \log(1/\epsilon) \le \epsilon^{-s^\prime/\beta} ,\]
for any $s^\prime > s$. Now rewriting $s$ for $s^\prime$ gives us the desired result.
\end{proof}
 \subsection{Proof of Theorem \ref{corr1}}\label{pf_cor1}
We will prove a more general version of Theorem~\ref{corr1}. We begin by generalizing the notion of (upper) Wasserstein dimension for a class of distributions. Suppose that $\Lambda$ is a family of distributions on $[0,1]^d$. For any $\tau\in [0,1]$, we define, 
\[\sN_\epsilon(\Lambda, \tau) = \inf\{\cN(\epsilon; S, \varrho): \mu(S) \ge 1-\tau, \, \forall \mu \in \Lambda\}.\]

The interpretation of $\sN_\epsilon(\cdot, \tau)$ is the same as that for the case of a single measure, defined in Section~\ref{intrinsic}. We define the upper Wasserstein dimension of this family as:
\[d^\ast_\alpha(\Lambda) = \inf \left\{s \in (2\alpha, \infty) : \limsup_{\epsilon \downarrow 0} \frac{\log \sN_{\epsilon}\left(\Lambda, \epsilon^{\frac{s\alpha}{s-2\alpha}}\right)}{\log(1/\epsilon)} \leq s \right\}.\]

 The proof of Theorem \ref{corr1} requires some supporting lemmas. We sequentially state and prove these lemmas as we proceed. We first prove in Lemma~\ref{lem_c1} that the set $\sA$ (defined below) takes the shape of an interval (left open or closed). 
\begin{lem}\label{lem_c1}
   Suppose that $\sA = \left\{s \in (2\alpha, \infty) : \limsup_{\epsilon \downarrow 0} \frac{\log \sN_{\epsilon}\left(\Lambda, \epsilon^{\frac{s\alpha}{s-2\alpha}}\right)}{\log(1/\epsilon)} \leq s \right\}$. Then, $\sA \supseteq (d^\ast_\alpha(\mu), \infty)$.
\end{lem}
\begin{proof}
     We begin by claiming the following: 
     \begin{equation}\label{claim:*}
        \text{\textbf{Claim}: If } s_1 \in \sA \text{ then } s_2 \in \sA, \text{ for all } s_2 \ge s_1 .
     \end{equation}
     To observe this, we note that,  if $s_2 \ge s_1>2 \alpha$ and $\epsilon \in (0,1)$, 
 \[\frac{s_1\alpha}{s_1-2\alpha} \ge  \frac{s_2\alpha}{s_2-2\alpha} \implies \epsilon^{\frac{s_1\alpha}{s_1-2\alpha}} \le  \epsilon^{\frac{s_2\alpha}{s_2-2\alpha}} \implies \sN_{\epsilon}\left(\Lambda, \epsilon^{\frac{s_1\alpha}{s_1-2\alpha}}\right) \ge \sN_{\epsilon}\left(\Lambda, \epsilon^{\frac{s_2\alpha}{s_2-2\alpha}}\right).\]
Here the last implication follows from Lemma~\ref{lem_b1}. Thus,
\[\limsup_{\epsilon \downarrow 0} \frac{\log \sN_{\epsilon}\left(\Lambda, \epsilon^{\frac{s_2\alpha}{s_2-2\alpha}}\right)}{\log(1/\epsilon)} \le \limsup_{\epsilon \downarrow 0} \frac{\log \sN_{\epsilon}\left(\Lambda, \epsilon^{\frac{s_1\alpha}{s_1-2\alpha}}\right)}{\log(1/\epsilon)} \le s_1 \le s_2.\]
Hence, $s_2 \in \sA$.

Let $s > d^\ast_\alpha(\Lambda)$, then by definition of infimum, we note that
we can find $s^{\prime} \in [d^\ast_\alpha(\Lambda), s)$, such that, $s^{\prime} \in \sA$. Since, $s>s^\prime \in \sA$, by Claim~\eqref{claim:*}, $s \in \sA$. Thus, for any $s > d^\ast_\alpha(\Lambda)$, $s\in \sA$, which proves the lemma.

\end{proof}
An immediate corollary of Lemma \ref{lem_c1} is as follows.
  \begin{cor}\label{cor_c1}
      Let $s > d^\ast_{\alpha}(\Lambda)$. Then, there exists $\epsilon^\prime \in (0,1]$, such that if $ 0< \epsilon \le \epsilon^\prime$, then, there exists a set $S$, such that $\cN(\epsilon; S, \varrho) \le \epsilon^{-s} $ and $\mu(S) \ge 1 - \epsilon^{\frac{s \alpha}{s-2\alpha}}$, for all $\mu \in \Lambda$.
  \end{cor}
 
 \begin{proof}
    Let $\delta = s - d^\ast_\alpha(\mu)$.  By Lemma~\ref{lem_c1}, we observe that  $s^\prime = d^\ast_\alpha(\mu) + \delta/2 \in \sA$. Now by definition of $\limsup$, one can find a $\epsilon^\prime>0$, such that, $\frac{\log \sN_\epsilon\left(\mu, \epsilon^{\frac{s^\prime \alpha}{s^\prime-2\alpha}}\right)}{\log(1/\epsilon)} \le s^\prime + \delta/2 = s $, for all $\epsilon \in (0, \epsilon^\prime]$. The result now follows from observing that $\sN_\epsilon\left(\mu, \epsilon^{\frac{s \alpha}{s^-2\alpha}}\right) \le \sN_\epsilon\left(\mu, \epsilon^{\frac{s^\prime \alpha}{s^\prime-2\alpha}}\right) \le \epsilon^{-s}$.
 \end{proof}

We suppose that $\delta = d^\star - d^\ast_\beta(\mu)$ and let $d^\prime = d^\ast_\beta(\mu) + \delta/2$. By Corollary \ref{cor_c1}, we can find an $\epsilon^\prime \in (0,1]$, such that if $\epsilon \in (0, \epsilon^\prime]$, we can find $S_\epsilon$, such that $\cN(\epsilon; S_\epsilon, \ell_\infty) \le \epsilon^{-d^\prime} $ and $\mu(S_\epsilon) \ge 1 - \epsilon^{\frac{d^\prime \beta}{d^\prime-2\beta}}$, for all $\mu \in \Lambda$. As a first step of constructing a dyadic-like partition of $[0,1]^d$, we state and prove the following lemma that helps us create the base of this sequential partitioning of the data space. For notational simplicity, we use $\operatorname{diam}(A) = \sup_{x,y \in A} \varrho(x,y)$ to denote the diameter of a set w.r.t. the metric $\varrho$. 

\begin{lem}\label{lem_c2}
    For any $r \ge \lceil \log_3(1/\epsilon^\prime) -2\rceil$, we can find disjoint sets  $S_{r,0},\dots,S_{r, m_r}$, such that, $\cup_{j=0}^{m_r} S_{r,j} = \Real^d$. Furthermore, $m_r \le 3^{d^\prime (r+2)}$, $\operatorname{diam}(S_{r,j})  \le 3^{-(r+1)}$, for all $j=1, \dots ,m_r$ and $\mu(S_{r,0}) \le 3^{-\frac{d^\prime(r+2) \beta}{d^\prime-2\beta}}\, \forall \mu \in \Lambda$.
\end{lem}
\begin{proof}
    We take $\epsilon = 3^{-(r+2)}$. Clearly, $0< \epsilon \le \epsilon^\prime$. We take $S_{r,0} = S_\epsilon^\complement$. By definition of covering numbers, we can find a minimal $\epsilon$-net $\{x_1, \dots, x_{m_r}\}$, such that $S \subseteq\cup_{j=1}^{m_r} B_{\ell_\infty}(x_i, \epsilon)$ and $m_r \le \epsilon^{-d^\prime} = 3^{d^\prime(r+2)}$. We construct $S_{r,1}, \dots S_{r,m_r}$ as follows:
    \begin{itemize}
        \item Take $S_{r,1} = B_{\ell_\infty}(x_1, \epsilon) \setminus S_{r,0}$.
        \item For any $j =2,\dots, m_r$, we take $S_{r,j} = B_{\ell_\infty}(x_j, \epsilon) \setminus \left(\cup_{j^\prime=0}^{j-1} S_{r,j^\prime}\right)$.
    \end{itemize}
    By construction $\{S_{r,j}\}_{j=0}^{m_r}$ are disjoint. Moreover, $\mu(S_{r,0}) = 1 - \mu(S_\epsilon) \le \epsilon^{\frac{d^\prime\beta}{d^\prime - 2 \beta}} = 3^{-\frac{d^\prime(r+2) \beta}{d^\prime-2\beta}}$, for all $\mu \in \Lambda$. Furthermore since, $S_{r,j} \subseteq B_{\ell_\infty}(x_j, \epsilon)$, \[\text{diam}(S_{r,j}) \le \text{diam}(B_{\ell_\infty}(x_j, \epsilon)) = 2 \epsilon = 2 \times 3^{-(r+2)} \le 3^{-(r+1)}.\]
\end{proof}

 We now construct a sequence of collection of sets $\{\sQ^{\ell}\}_{\ell=1}^r$ as follows:
 \begin{itemize}
     \item Take $\sQ^r = \{S_{r,j}\}_{j=1}^{m_r}$.
     \item Given $\ell+1$, let, \(Q^\ell_1 = \bigcup_{\substack{Q \in \sQ^{\ell+1}, \\ Q \cap S_{\ell,1} \neq \emptyset  }}(Q \setminus S_{\ell,0}) \) and if $2 \le j \le m_{\ell}$, we let, \[Q^\ell_j = \bigg(\bigcup_{\substack{ Q \in \sQ^{\ell+1}, \\ Q \cap S_{r,j} \neq \emptyset}}
     (Q\setminus S_{\ell,0} )\bigg) \setminus \left(\cup_{j^\prime = 1}^{j-1} Q_{j^\prime}^\ell\right).\] Take $\sQ^\ell = \{Q^{\ell}_j\}_{j=1}^{m_\ell}$.
 \end{itemize}
 Clearly for any $Q \in \sQ^\ell$, $\sup\limits_{Q \in \sQ^\ell}\text{diam}(Q) \le 2 \sup\limits_{Q^\prime \in \sQ^{\ell+1}} \text{diam}(Q^\prime) + \sup\limits_{1 \le j \le m_\ell}\text{diam}(S_{\ell,j}) \le 3 \times 3^{-(\ell+1)} = 3^{-\ell}$, by induction.

 Also, if $Q \in \sQ^{\ell+1}$, we can find a $Q^\prime \in \sQ^\ell$, such that $Q \subseteq Q^\prime \cup S_{\ell,0}$. Also if $Q_1, Q_2 \in \sQ^\ell$, $Q_1 \cap Q_2 = \emptyset$. Furthermore, $|\sQ^\ell| \le m_\ell$, by construction.

 Let $\epsilon = n^{-1/d^\prime}$. Recall that, $\delta = d^\star - d^\ast_\beta(\mu)$ and $d^\prime = d^\star_\beta(\mu) + \delta/2$. We take $n_0$, large enough such that $\epsilon \le \epsilon^\prime$, if $n \ge n_0$. Let $t$ be the smallest integer such that $3^{-t} \le \epsilon$. Also, let $s$ be the smallest integer such that $3^{-s} \le \epsilon^{\frac{d^\prime-2\beta}{d^\prime}}$. Clearly, $s \le t$. Also, $3^t \le \frac{3}{\epsilon}$ and $3^s \le 3 \epsilon^{\frac{2\beta-d^\prime}{d^\prime}}$.

 Let $P^f_{\theta}(x) = \text{Clip}\left(\sum_{|\bs| \le \lfloor \beta \rfloor} \frac{\partial^{\bs} f(\theta)}{\bs!} (x - \theta)^{\bs}, -C,C\right)$, where $\text{Clip}(x,a,b) = (x \wedge b) \vee a$. We choose a point  $x_{r,j} \in Q^r_j$, $s \le r \le t$ and $j \in [m_r]$. We also define the bivariate function $q(\cdot, \cdot)$ as $q(\ell,j) = j^\prime$ if $Q^\ell_j \cap Q^{\ell-1}_{j^\prime} \neq \emptyset$. This $j^\prime$ is unique, since we can find a $Q_{j^\prime}^{\ell-1}$ which contains $Q_j^\ell$ and since sets in $\sQ^{\ell-1}$ are disjoint.

We denote, $S_{s:t,0} = \cup_{r=s}^t S_{r,0}$ and let, $M_r = \sum_{j=1}^{m_r} |(\mu_n - \mu)(Q^r_j)|$. The following result (Lemma~\ref{lem_c3}) helps us control $\E M_r$.

\begin{lem}\label{lem_c3}
    Suppose that $\{A_1, \dots, A_m\}$ forms a partition of a set $T$. Then, \[\E M_r = \E \sum_{j=1}^m|\mu_n(A_j) - \mu(A_j)| \le \sqrt{\frac{m \mu(T)}{n}}.\]
\end{lem}
\begin{proof}
    We begin by noting that $n \mu_n(A_j) \sim\text{Bin}(n, \mu(A_j))$. Thus, $\operatorname{Var}(\mu_n(A_j)) = \frac{\mu(A_j) (1-\mu(A_j))}{n}$. Hence,
    \begingroup
    \allowdisplaybreaks
    \begin{align}
        \E \sum_{j=1}^m|\mu_n(A_j) - \mu(A_j)|  \overset{(i)}{\le} &  \sqrt{m} \left(\E \sum_{j=1}^m(\mu_n(A_j) - \mu(A_j))^2\right)^{1/2} \nonumber\\
       =  & \sqrt{m} \left( \sum_{j=1}^m\operatorname{Var}(\mu_n(A_j))\right)^{1/2} \nonumber\\
        = &  \sqrt{m} \left(\sum_{j=1}^m\frac{\mu(A_j) (1-\mu(A_j))}{n}\right) \nonumber\\
        \le &  \sqrt{m} \sqrt{\sum_{j=1}^m\frac{\mu(A_j) }{n}} \nonumber\\
        = & \sqrt{\frac{m \mu(T)}{n}}. \nonumber
    \end{align}
    \endgroup
    In the above calculations, $(i)$ follows from Cauchy-Schwartz inequality.
\end{proof}

\begin{lem}\label{lem_37}
    Suppose that $\text{Poly}(k,d,\alpha)$ denotes the set of all polynomials from $[0,1]^d \to \Real$, with degree at most $k$ and absolutely bounded by $\alpha$. Also let $A_1, \dots, A_m$ are disjoint and suppose that $\mathfrak{P} = \{f = \sum_{j=1}^m f_j \one_{A_j}: f_j \in \text{Poly}(k,d,\alpha)\}$. Then, 
    \[ \gamma(\alpha, m, n) := \E \sup_{f \in \mathfrak{P}} \left( \int f d\mu_n - \int f d\mu\right) \lesssim  \alpha \sqrt{\frac{m (k+1)^d\log n}{n}}.\]
\end{lem} 
\begin{proof}
It is easy to note that, $\operatorname{Pdim}(\mathfrak{P}) \le m (k+1)^d$. Here $\operatorname{Pdim}(\sF)$ denotes the pseudo-dimension of the real-valued function class $\sF$ (see \citet[Definition 11.2]{anthony1999neural}).  Applying Dudley's chaining, we recall that,
\begingroup
\allowdisplaybreaks
\begin{align}
    \E_{\sigma} \sup_{f \in \mathfrak{P}} \sum_{i=1}^n \sigma_i f(x_i) \lesssim & \int_{0}^\alpha \sqrt{\frac{\log \cN(\epsilon; \mathfrak{P}, \|\cdot\|_{\fL_2(\mu_n)})}{n}} d\epsilon \nonumber\\
    \le & \int_{0}^\alpha \sqrt{\frac{\log \cN(\epsilon; \mathfrak{P}, \|\cdot\|_{\fL_\infty(\mu_n)})}{n}} d\epsilon \nonumber\\
    \le & \int_{0}^\alpha \sqrt{\frac{\operatorname{Pdim}(\mathfrak{P})}{n} \log (2\alpha en/\epsilon)} d\epsilon \label{e_s1}\\
    \lesssim & \alpha \sqrt{\frac{\operatorname{Pdim}(\mathfrak{P})\log n}{n}} \nonumber\\
    \le & \alpha \sqrt{\frac{m (k+1)^d\log n}{n}}.\nonumber
\end{align}
\endgroup
Here inequality \eqref{e_s1} follows from Lemma~\ref{lem_anthony_bartlett}. Thus, by symmetrization,
\begin{align*}
    \gamma(\alpha, m, n) = &  \E \sup_{f \in \mathfrak{P}} \left( \int f d\mu_n - \int f d\mu\right)
    \le  2 \E \sup_{f \in \mathfrak{P}} \sum_{i=1}^n \sigma_i f(x_i)
    \lesssim  \alpha \sqrt{\frac{m (k+1)^d\log n}{n}}.
\end{align*}

\end{proof}

We now state and prove a more general version of Theorem~\ref{corr1}.  Theorem~\ref{corr1} follows as a corollary of Theorem~\ref{gen_kolmo} for the special case when $\Lambda = \{\mu\}$.
\begin{thm}\label{gen_kolmo}
    Let $\sF = \sH^\beta(\Real^d, \Real, C)$, for some $C>0$. Then if $d^\star > d^\ast_\beta(\Lambda)$, we can find an $n_0$, such that, if $n \ge n_0$, 
    \[\sup_{\mu \in \Lambda} \E\|\hmu_n - \mu\|_{\sF} \lesssim n^{-\beta/d^\star}.\]
\end{thm}


   \begin{proof}
We fix any $\mu \in \Lambda$. Recall that $P^f_{\theta}(x) = \sum_{|\bs| \le \lfloor \beta \rfloor} \frac{\partial^{\bs} f(\theta)}{\bs!} (x - \theta)^{\bs}$. We note that since, $\|P^f_{\theta}\|_\infty \le C$,
\begingroup
\allowdisplaybreaks
\begin{align}
    |P^f_{\theta}(x) - f(x) |\le & \left|\sum_{|\bs| \le \lfloor \beta \rfloor} \frac{\partial^{\bs} f(\theta)}{\bs!} (x - \theta)^{\bs} - f(x)\right|\nonumber \\ 
    = & \left|\sum_{\bs: |\bs| = \lfloor \beta \rfloor} \frac{(x - \theta)^{\bs}}{\bs !}(\partial^{\bs} f(y) - \partial^{\bs} f(\theta))\right| \label{jul_7_2}\\
    \le & \|x - \theta\|_\infty^{\lfloor \beta \rfloor} \sum_{\bs: |\bs| = \lfloor \beta \rfloor} \frac{1}{\bs !}|\partial^{\bs} f(y) - \partial^{\bs} f(\theta)| \nonumber \\
    \le & 2C \|x - \theta\|_\infty^{\lfloor \beta \rfloor} \|y - \theta\|_\infty^{\beta - \lfloor \beta \rfloor} \nonumber\\
    \le & 2C\|x - \theta\|_\infty^\beta \label{c1}
\end{align}
\endgroup
Equation~\eqref{jul_7_2} follows from Taylor's theorem where $y$ lies in the line segment joining $x$ and $\theta$. Furthermore,
\begin{align}
    |P^f_{\theta}(x) - P^f_{\theta^\prime}(x)| \le & |P^f_{\theta}(x) - f(x)| +  |f(x) - P^f_{\theta^\prime}(x)| \nonumber \\
    \le & 2C \left(\|x - \theta\|_\infty^\beta + \|x - \theta^\prime\|_\infty^\beta \right) \label{c4}
\end{align}
For notational simplicity, let $\xi_r = \sup_{f \in \text{Poly}\left(\lfloor \beta \rfloor , d, \alpha_r,\right)}\sum_{j=1}^{m_r} \int_{Q^r_j} f d(\mu_n - \mu)  $, where, $ \alpha_r = \max_{1 \le j^\prime \le m_{t-1}}\|P^f_{x_{t,j}} - P^f_{x_{t-1,j^\prime}}\|_{\fL_\infty(Q^{t-1}_{j^\prime})}$. We can control $\E \xi_r$ through Lemma~\ref{lem_37}. We note that, for any $f \in \sH^\beta(\Real^d, \Real, C)$,
\begingroup
\allowdisplaybreaks
 \begin{align}
    & \left| \int f d(\mu_n- \mu)\right| \nonumber \\
    = & \left|\sum_{j=0}^{m_t}\int_{S_{t,j}} fd(\mu_n - \mu)\right| \nonumber\\
    \le &  C (\mu_n + \mu)(S_{s:t,0}) + \left|\sum_{j=1}^{m_t}\int_{Q^t_j\setminus S_{s:t,0}} (f - P^f_{x_{t,j}} + P^f_{x_{t,j}}) d(\mu_n - \mu)\right| \nonumber\\
    \le &  C (\mu_n + \mu)(S_{s:t,0}) + 2 \max_{1 \le j \le m_t}\|f - P^f_{x_{t,j}}\|_{\fL_\infty(Q^t_j)}  + \left|\sum_{j=1}^{m_t}\int_{Q^t_j \setminus S_{s:t,0}}  P^f_{x_{t,j}} d(\mu_n - \mu)\right| \nonumber\\
    \le &  C (\mu_n + \mu)(S_{s:t,0}) + 2C   3^{-t\beta}   + \left|\sum_{j=1}^{m_t}\int_{Q^t_j \setminus S_{s:t,0}}  P^f_{x_{t,j}} d(\mu_n - \mu)\right| \label{c2}\\
    = &  C (\mu_n + \mu)(S_{s:t,0}) + 2C  3^{-t\beta}  \nonumber\\
    & + \left|\sum_{j^\prime =1}^{m_{t-1}} \sum_{j: q(t,j) = j^\prime }\int_{Q^t_j \setminus S_{s:t,0}}  (P^f_{x_{t,j}} - P^f_{x_{t-1, j^\prime} } + P^f_{x_{t-1, j^\prime} }) d(\mu_n - \mu)\right| \nonumber\\
     \le  &  C (\mu_n + \mu)(S_{s:t,0}) + 2 C 3^{-t\beta}   + \xi_t + \left|\sum_{j^\prime =1}^{m_{t-1}} \int_{Q^t_{j^\prime} \setminus S_{s:t,0}}   P^f_{x_{t-1, j^\prime} } d(\mu_n - \mu)\right| \label{c3}\\
      \le  &  C (\mu_n + \mu)(S_{s:t,0}) + 2 C 3^{-t\beta} + C M_s + \sum_{r=s}^t \xi_r \label{c5}
 \end{align}
 \endgroup
 Inequality \eqref{c2} follows from observing that $\|f - P^f_{x_{j,t}} \|_{\fL_\infty(Q^t_j)} \le C \sup_{x \in \fL_\infty(Q^t_j)}\|x - x_{j,t}\|_\infty \le C \text{diam}(Q^t_j) \le C 3^{-t}$. Inequality \eqref{c3} follows from observing that
 \[\left|\sum_{j^\prime=1}^{m_{t-1}}\sum_{j: q(t,j) = j^\prime} \int_{Q^t_j \setminus S_{s:t,0}} (P^f_{x_{t,j}} - P^f_{x_{t-1,j^\prime}}) d(\mu_n - \mu)\right| 
 \le  \xi_t.\]
 Equation \eqref{c5} follows easily from an inductive argument and observing that $\|P^f_a\|_\infty \le C$, for any $a \in \Real^d$, by construction.
To bound the expectation of $\xi_r$, we note that,
\begin{align*}
   \E \xi_r \le & \, \gamma\left(\max_{1 \le j^\prime \le m_{t-1}}\|P^f_{x_{t,j}} - P^f_{x_{t-1,j^\prime}}\|_{\fL_\infty(Q^{t-1}_{j^\prime})}, m_r, n \right)\\
   \lesssim & \max_{1 \le j^\prime \le m_{t-1}}\|P^f_{x_{t,j}} - P^f_{x_{t-1,j^\prime}}\|_{\fL_\infty(Q^{t-1}_{j^\prime})} \sqrt{\frac{m_r \log n}{n}} \quad (\text{from Lemma } \ref{lem_37} ) \\
    \le & 4 C 3^{-\beta(t-1)} \sqrt{\frac{m_r \log n}{n}}.
 \end{align*}

Thus, taking expectations on both sides of \eqref{c5}, we get,
\begingroup
\allowdisplaybreaks
\begin{align}
     & \E \left| \int f d(\mu_n- \mu)\right| \nonumber\\
       \lesssim  &  3^{-t \beta} + 2 C \sum_{r=s}^t 3^{-\frac{d^\prime(r+2) \beta}{d^\prime-2\beta}} + C \sqrt{m_s/n} + 4 C \sum_{r=s}^t 3^{-\beta r} \sqrt{\frac{m_r \log n}{n}} \label{c7}\\
       \le  &  3^{-t \beta} + 2 C \sum_{r=s}^t 3^{-\frac{d^\prime(r+2) \beta}{d^\prime-2\beta}} + C \sqrt{\frac{3^{d^\prime (s+2)}}{n}} + 4 C \sum_{r=s}^t 3^{-\beta r} \sqrt{\frac{3^{d^\prime (r+2)} \log n}{n}} \label{c8}\\
       \le  &   3^{-t \beta} + 2 C \sum_{r=s}^\infty 3^{-\frac{d^\prime(r+2) \beta}{d^\prime-2\beta}} + C \sqrt{\frac{\log n}{n}} \left( 3^{\frac{d^\prime (s+2)}{2}} + 4 \sum_{r=s}^t 3^{\frac{2d^\prime + r(d^\prime-2\beta)}{2}}\right) \nonumber \\
       =  &  3^{-t \beta} + 2 C \frac{3^{\frac{-d^\prime(s+2)\beta}{d^\prime-2\beta}}}{1 - 3^{-\frac{d^\prime\beta}{d^\prime-2\beta}}} +  C \sqrt{\frac{\log n}{n}} \left( 3^{\frac{d^\prime (s+2)}{2}} + 4 \frac{3^{\frac{2d^\prime + s(d^\prime-2\beta)}{2}}( (3^{\frac{d^\prime - 2\beta}{2}})^{t - s + 1}-1)}{ 3^{\frac{d^\prime - 2\beta}{2}}-1}\right) \nonumber\\
       \le  &  3^{-t \beta} + 2 C \frac{3^{\frac{-2d^\prime\beta}{d^\prime-2}}}{1 - 3^{-\frac{d^\prime\beta}{d^\prime-2\beta}}} \times (3^{-s})^\frac{d^\prime \beta}{d^\prime-2\beta} \nonumber \\
       & +  C \sqrt{\frac{\log n}{n}} \left( 3^{\frac{d^\prime (s+2)}{2}} + 4 \frac{3^{\frac{2d^\prime + s(d^\prime-2\beta)}{2}}( (3^{\frac{d^\prime - 2\beta}{2}})^{t - s + 1}-1)}{ 3^{\frac{d^\prime - 2\beta}{2}}-1}\right) \nonumber\\
       \le  &   3^{-t \beta} + 2 C \frac{\epsilon^\beta}{1 - 3^{-\frac{d^\prime\beta}{d^\prime - 2\beta}}}   +  C \sqrt{\frac{\log n}{n}} \left( 3^{\frac{d^\prime (s+1)}{2}} + 4 \frac{3^{\frac{2d^\prime + s(d^\prime-2\beta)}{2}}( (3^{\frac{d^\prime - 2\beta}{2}})^{t - s + 1}-1)}{ 3^{\frac{d^\prime - 2\beta}{2}}-1}\right) \nonumber\\
       \lesssim  &  \epsilon^\beta   +   \sqrt{\frac{\log n}{n}} \left( 3^{\frac{d^\prime s}{2}} + 3^{\frac{(d^\prime -2\beta)t}{2}}\right) \nonumber\\
       \lesssim  &  \epsilon^\beta   +  \sqrt{\frac{\log n}{n}}  \epsilon^{\beta -d^\prime /2} \nonumber\\
       \lesssim  &  n^{-\beta/d^\prime} \sqrt{\log n}\nonumber\\
       \lesssim & n^{-\beta/d^\star}.\nonumber
 \end{align}
 \endgroup
 Equation \eqref{c7} follows from Lemma~\ref{lem_c3}, where as \eqref{c8} follows from Lemma~\ref{lem_c2}.
 \end{proof}

 \section{Proofs from Section \ref{assumptions}}
 
 \subsection{Proof of Lemma~\ref{skorohod}}
 \skorohod*
 \begin{proof}
 To show the first implication, it is enough to show that $\int \| \tilde{D} ( \tilde{E}(x)) - x\|_2^2 d\mu(x) = 0$. To see this, we observe that,
 \begin{align*}
     \int \| \tilde{D} \circ \tilde{E}(x) - x\|_2^2 d\mu(x) = & \int \| \tilde{D} (z) - \tilde{D}(z)\|_2^2 d\nu(z) = 0.
 \end{align*}
 Here the first equality follows from the fact that $(X, \tilde{E}(X)) \overset{d}{=} (\tilde{D}(Z), Z)$. The other statement follows similarly. 
 \end{proof}
 \subsection{Proof of Proposition~\ref{prop19}}
 \label{prop_19_pf}
 \propnineteen*
 \begin{proof}
    We use the so-called good set principle to prove this result. Suppose $A = \{x \in [0,1]^d: \Tilde{D} \circ \Tilde{E} (x) = x\}$ and let $x \in \operatorname{supp}(\mu)$. Clearly, by Lemma~\ref{skorohod}, $\mu(A) = 1$. By definition of the support, for any $n \in \mathbb{N}$, $\mu(B(x,1/n) \cap A) = \mu(B(x,1/n)) >0$, which implies that $B(x,1/n) \cap A$ is non-empty. Let $x_n \in B(x,1/n) \cap A$. Thus, for any $x \in \operatorname{supp}(\mu)$, there exist a sequence $\{x_n(x)\}_{n \in \mathbb{N}} \subseteq A$, such that $x_n(x) \to x$ as $n \to \infty$. We first show that $\Tilde{E}$, restricted to the support of $\mu$ is a bijection. To show injectivity, let $x, \, x^\prime \in \operatorname{supp}(\mu)$ and $\Tilde{E}(x) = \Tilde{E}(x^\prime)$. Thus, $\Tilde{D} \circ \Tilde{E}(x) = \Tilde{D} \circ \Tilde{E}(x^\prime) \implies \Tilde{D} \circ \Tilde{E}(\lim_{n \to \infty}x_n(x)) = \Tilde{D} \circ \Tilde{E}(\lim_{n \to \infty} x_n(x^\prime)).$ Owing to the uniform continuity of $\Tilde{D}$ and $\Tilde{E}$ and the fact that $x_n(x)$ and $x_n(x')$ are in $A$, this implies that $\lim_{n \to \infty} \Tilde{D} \circ \Tilde{E}(x_n(x)) = \lim_{n \to \infty}\Tilde{D} \circ \Tilde{E}( x_n(x^\prime)) \implies \lim_{n \to \infty}x_n(x) = \lim_{n \to \infty} x_n(x^\prime) \implies x = x^\prime$.

    To show surjectivity, let $z \in \operatorname{supp}(\nu)$.  Let $A^\prime = \{z \in [0,1]^\ell: \Tilde{E} \circ \Tilde{D} (z) = z\}$. By a similar argument, there exists a sequence $z_n \to z$, such that $\{z_n\}_{n \in \fN} \subseteq A^\prime $. We consider $\Tilde{D}(z)$. Clearly, \(\Tilde{E} \circ \Tilde{D} (z) = \Tilde{E} \circ \Tilde{D} (\lim_{n \to \infty} z_n) = \lim_{n \to \infty} \Tilde{E} \circ \Tilde{D} (z_n) = \lim_{n \to \infty} z_n = z.\) The only thing remaining is to show that $\Tilde{D}(z) \in \operatorname{supp}(\mu)$. To show this, let $z^\prime \in B(z,\epsilon)$. Clearly, $\nu(B(z,\epsilon)) > 0$. $\|\Tilde{D}(z) - \Tilde{D}(z^\prime)\| \lesssim \epsilon^{\alpha_d \wedge 1}$, since, $\Tilde{D} \in \sH^{\alpha_d}\left(\Real^\ell,\Real^d, C_d \right)$. Thus, there exists a $c>0$ such that  $\|\Tilde{D}(z) - \Tilde{D}(z^\prime)\| \le c \, \epsilon^{\alpha_d \wedge 1}$. Hence, $B(\Tilde{D}(z), c \epsilon^{\alpha_d \wedge 1}) \supseteq  \{\Tilde{D}(z^\prime): z^\prime \in B(z,\epsilon)\} \implies \mu\left(B(\Tilde{D}(z), c \epsilon^{\alpha_d \wedge 1})\right) \ge \mu \left(\{D(z^\prime): z^\prime \in B(z,\epsilon)\} \right) = \nu(B(z,\epsilon)) > 0$, for any $\epsilon > 0$. Thus, for any $\delta > 0$, taking $\epsilon = (\delta/c)^{1/(\alpha_d \wedge 1)}$, gives that $\mu\left(B(\Tilde{D}(z), \delta)\right)>0$, for any $\delta>0$, which implies that $D(z) \in \operatorname{supp}(\mu)$.

    Now to show that $\Tilde{E}$ is the inverse of $\Tilde{D}$ on the support of $\mu$, we let $x \in \operatorname{supp}(\mu)$. By the argument above, there exists a sequence $\{x_n\}_{n \in \fN} \subseteq A$, such that $\lim_{n \to \infty} x_n = x$. Thus, $\Tilde{D} \circ \Tilde{E}(x) = \Tilde{D} \circ \Tilde{E} (\lim_{n \to \infty} x_n ) = \lim_{n \to \infty} \Tilde{D} \circ \Tilde{E}(x_n) = \lim_{n \to \infty} x_n = x$. Similarly, one can show that $\Tilde{E} \circ \Tilde{D} (z) = z$, for any $z \in \operatorname{supp}(\nu)$. Clearly, $\Tilde{D}$ and $\Tilde{E}$ are continuous. Thus, $\operatorname{supp}(\mu)$ and $\operatorname{supp}(\nu)$ are homeomorphic.
    \end{proof}
 \section{Proofs from Section \ref{errord}}
 \label{app_c}
 \subsection{Proof of Lemma~\ref{oracle}}\label{p_oracle}
 \pforacle*
 \begin{proof}
     \textbf{Proof of part (a)}: For any $G \in \sG$, we observe that,
     \begingroup
     \allowdisplaybreaks
     \begin{align}
         \|\mu - (\hG_n)_\sharp \nu\|_\Phi \le & \|\hmu_n - (\hG_n)_\sharp \nu\|_\Phi + \|\mu - \hmu_n\|_\Phi \nonumber \\
         \le & \|\hmu_n - G_\sharp \nu\|_\Phi + \|\mu - \hmu_n\|_\Phi \nonumber\\
         \le &  \|\mu - G_\sharp \nu\|_\Phi + 2 \|\mu - \hmu_n\|_\Phi \nonumber
     \end{align}
     \endgroup
     Taking infimum on both sides w.r.t. $G \in \sG$ gives us the desired result.

     Similarly, for the estimator $\hG_{n,m}$, we note that,
     \begin{align*}
         \|\mu - (\hG_{n,m})_\sharp \nu\|_\Phi \le & \|\hmu_n - (\hG_{n,m})_\sharp \hnu_m\|_\Phi +   \|\hmu_n - \mu\|_\Phi + \|\hnu_m - \nu\|_{\Phi \circ \sG}\\
         \le & \|\hmu_n - G_\sharp \hnu_m\|_\Phi +   \|\hmu_n - \mu\|_\Phi + \|\hnu_m - \nu\|_{\Phi \circ \sG}\\
         \le & \|\mu - G_\sharp \nu\|_\Phi +  2 \|\hmu_n - \mu\|_\Phi + 2\|\hnu_m - \nu\|_{\Phi \circ \sG}\\
     \end{align*}
     Again taking infimum on both sides w.r.t. $G \in \sG$ gives us the desired result.

      \textbf{Proof of part (b)}: For any $E\in \sE$ and $D \in \sD$,
      \begingroup
      \allowdisplaybreaks
      \begin{align*}
         &  \|(\mu, (\hE_n)_\sharp \mu) - ((\hD_n)_\sharp \nu, \nu) \|_{\Psi} \\
         \le & \|(\hmu_n, (\hE_n)_\sharp \hmu_n) - ((\hD_n)_\sharp \nu, \nu) \|_{\Psi} + \|(\hmu_n, (\hE_n)_\sharp \hmu_n) - (\mu, (\hE_n)_\sharp \mu) \|_{\Psi}\\
          \le & \|(\hmu_n,E_\sharp \hmu_n) - (D_\sharp \nu, \nu) \|_{\Psi} + \|\hmu_n - \hmu\|_{\sF_1}\\
           \le & \|(\mu,E_\sharp \mu) - (D_\sharp \nu, \nu) \|_{\Psi} + 2 \|\hmu_n - \hmu\|_{\sF_1}.
      \end{align*}
      \endgroup
      Taking infimum over $E$ and $D$ gives us the desired result.

      Similarly, 
       \begin{align*}
         & \|(\mu, (\hE_{n,m})_\sharp \mu) - ((\hD_{n,m})_\sharp \nu, \nu) \|_{\Psi} \\
         \le & \|(\hmu_n, (\hE_{n,m})_\sharp \hmu_n) - ((\hD_{n,m})_\sharp \hnu_m, \hnu_m) \|_{\Psi} + \|\hmu_n - \hmu\|_{\sF_1} + \|\hnu_m - \hnu\|_{\sF_2}\\
          \le & \|(\hmu_n,E_\sharp \hmu_n) - (D_\sharp \hnu_m, \hnu_m) \|_{\Psi} + \|\hmu_n - \hmu\|_{\sF_1}\\
           \le & \|(\mu,E_\sharp \mu) - (D_\sharp \nu, \nu) \|_{\Psi} + 2 \|\hmu_n - \hmu\|_{\sF_1} + 2 \|\hnu_m - \hnu\|_{\sF_2}.
      \end{align*}
      Again taking infimum over $E$ and $D$ gives us the desired result.
 \end{proof}
 \section{Proofs from Section \ref{sec_mis}}
\label{app_d}
 \subsection{Proof of Theorem \ref{approx}}
  We begin by proving Theorem \ref{approx}, which discusses the approximation capabilities of ReLU networks when the underlying measure has a low entropic dimension. Some supporting lemmas required for the proof are stated in Appendix~\ref{sec_c2}.

For $s> \bar{d}_{p\beta}(\mu)$, by Lemma~\ref{lem_e1}, we can find an $\epsilon^\prime \in (0,1)$, such that if $\epsilon \in (0,\epsilon^\prime]$ and a set $S$, such that $\mu(S) \ge 1 - \epsilon^{p \beta}$ and $\cN(\epsilon; S, \ell_\infty) \le \epsilon^{-s}$. In what follows, we take $\epsilon \in (0, \epsilon^\prime]$.

Let $K = \lceil \frac{1}{2\epsilon}\rceil$. For any $\bi \in [K]^d$, let $\theta^{\bi} = ( \epsilon + 2 (i_1-1) \epsilon, \dots,  \epsilon + 2(i_d-1) \epsilon )$. We also let, $\sP_\epsilon = \{B_{\ell_\infty}(\theta^{\bi}, \epsilon): \bi \in [K]^d\}$. By construction, $\theta^{\bi}$'s are at least $2\epsilon$ apart, making the sets in $\sP_\epsilon$ disjoint. Following the proof of Lemma~\ref{lem_b2}, we make the following claim. 
\begin{lem}\label{lem_e2}
$|\{A \in \sP_\epsilon: A \cap S \neq \emptyset\} | \le 2^d \epsilon^{-s} $.
\end{lem}
\pfapprox*
\begin{proof} We also let $\cI = \left\{ \bi \in [K]^d :  B_{\ell_\infty}(\theta^{\bi}, \epsilon) \cap S \neq \emptyset\right\}$. We also let $\cI^\dagger = \{\bj \in  [K]^d: \min_{\bi \in \cI} \|\bi - \bj\|_1 \le 1\}$. We know that $|\cI^\dagger| \le 3^d |\cI| \le 6^d\cN(\epsilon;S,\ell_\infty)$. For $0 < b \le a$, let,
\[\xi_{a, b}(x) = \relu\left(\frac{x+a}{a-b}\right) - \relu\left(\frac{x+b}{a-b}\right) - \relu\left(\frac{x-b}{a-b}\right) + \relu\left(\frac{x-a}{a-b}\right).\]
 \begin{figure}
    \centering
    \begin{tikzpicture}[scale=3]
\draw (-1.2,0) -- (1.2,0);
\draw (-0,-0.2) -- (0,1.2);
\draw[blue] (-1,0) -- (-0.65,1);
\draw[blue] (-0.65,1) -- (0.65,1);
\draw[blue] (0.65,1) -- (1,0); 
\filldraw [magenta] (1,0) circle (0.5pt);
\filldraw [magenta] (-1,0) circle (0.5pt);
\filldraw [magenta] (0.65,1) circle (0.5pt);
\filldraw [magenta] (-0.65,1) circle (0.5pt);

\draw [dashed] (0.65,0) -- (0.65,1);
\draw [dashed] (-0.65,0) -- (-0.65,1);
\node[] at (0.65,-0.1) {$a-b$};
\node[] at (-0.65,-0.1) {$-a+b$};
\node[] at (1.2,-0.1) {$a$};
\node[] at (-1.2,-0.1) {$-b$};
\node[] at (0.1,-0.1) {$0$};
\node[] at (0.1,1.1) {$1$};
\end{tikzpicture}
    \caption{{\small Plot of $\xi_{a, b}(\cdot)$.}}
    \label{fig:xi}
\end{figure} A pictorial view of this function is given in Figure \ref{fig:xi} and can be implemented by a ReLU  network of depth two and width four. Thus, $\cL(\xi_{a, b}) = 2$ and $\cW(\xi_{a, b}) = 12$. Suppose that $\delta = \epsilon/3$ and let, $\zeta(x) = \prod_{\ell=1}^d \xi_{\epsilon+\delta, \delta}(x_\ell)$. Clearly, $\cB(\xi_{\epsilon+\delta, \delta}) \le \frac{1}{\delta}$. It is easy to observe that $\{\zeta(\cdot - \theta^{\bi}): \bi \in \cI^\dagger\}$ forms a partition of unity on $S$, i.e. $\sum_{\bi \in \cI^\dagger} \zeta(x - \theta^{\bi}) = 1, \forall x \in S$.

We consider the Taylor approximation of $f$ around $\theta$ as,
\[P_{\theta}(x) = \sum_{|\bs| \le \lfloor \alpha \rfloor} \frac{\partial^{\bs} f(\theta)}{ \bs !} \left(x - \theta\right)^{\bs} .\]
Note that for any $x \in [0,1]^d$, $f(x) - P_{\theta}(x) = \sum_{\bs: |\bs| = \lfloor \alpha \rfloor} \frac{(x - \theta)^{\bs}}{\bs !}(\partial^{\bs} f(y) - \partial^{\bs} f(\theta))$, for some $y$, which is a convex combination of $x$ and $\theta$. Thus,
\begingroup
\allowdisplaybreaks
\begin{align}
    f(x) - P_{\theta}(x)   =  & \sum_{\bs: |\bs| = \lfloor \alpha \rfloor} \frac{(x - \theta)^{\bs}}{\bs !}(\partial^{\bs} f(y) - \partial^{\bs} f(\theta)) \nonumber\\ 
    \le & \|x - \theta\|_\infty^{\lfloor \alpha \rfloor} \sum_{\bs: |\bs| = \lfloor \alpha \rfloor} \frac{1}{\bs !}|\partial^{\bs} f(y) - \partial^{\bs} f(\theta)| \nonumber \\
    \le & 2 C \|x - \theta\|_\infty^{\lfloor \alpha \rfloor} \|y - \theta\|_\infty^{\alpha - \lfloor \alpha \rfloor} \nonumber\\
    \le & 2 C \|x - \theta\|_\infty^\alpha \label{apr6_1}.
\end{align}
\endgroup
Next we define $\tilde{f}(x) = \sum_{\bi \in \cI^\dagger} \zeta(x - \theta^{\bi}) P_{\theta^{\bi}}(x)$. 
Thus, if $x \in S$,
\begingroup
\allowdisplaybreaks
\begin{align}
    |f(x) - \tilde{f}(x)| =  \left| \sum_{\bi \in \cI^\dagger} \zeta(x - \theta^{\bi}) (f(x) - P_{\theta^{\bi}}(x))\right| 
    \le & \sum_{\bi \in \cI^\dagger: \|x - \theta^{\bi}\|_\infty \le 2 \epsilon }|f(x) -  P_{\theta^{\bi}}(x)| \nonumber\\
    \le & C 2^{d+1} (2\epsilon)^\alpha \nonumber\\
    = & C 2^{d + \alpha+1} \epsilon^\alpha.
    \label{apr6_2}
\end{align}
\endgroup
We note that, 
\(\tilde{f}(x) = \sum_{\bi \in \cI^\dagger} \zeta(x - \theta^{\bi}) P_{\theta^{\bi}}(x)  = \sum_{\bi \in \cI^\dagger} \sum_{|\bs| \le \lfloor \alpha \rfloor}  \frac{\partial^{\bs} f(\theta^{\bi})}{ \bs !}  \zeta(x - \theta^{\bi})  \left(x - \theta^{\bi}\right)^{\bs}.\)
Let $a_{\bi, \bs} =  \frac{\partial^{\bs} f(\theta^{\bi})}{ \bs !}$ and 
\begin{align*}
    \hat{f}_{\bi, \bs}(x) = & \text{prod}_m^{(d+|\bs|)}(\xi_{\epsilon_1, \delta_1}(x_1 - \theta_1^{\bi}), \dots, \xi_{\epsilon_d, \delta_d}(x_d - \theta_d^{\bi}), \underbrace{(x_1-\theta_1^{\bi}), \dots, (x_1-\theta_1^{\bi})}_{\text{$s_1$ times}},\\
   & \hspace{2cm} \ldots, \underbrace{(x_1-\theta_d^{\bi}), \dots ,(x_d -\theta_d^{\bi})}_{\text{$s_d$ times}}),
\end{align*}

where $\text{prod}_m^{(d+|\bs|)}(\cdot)$ is defined in Lemma~\ref{lem_43} and approximates the product function for $d+|\bs|$ real numbers. We note that $\text{prod}_m^{(d+|\bs|)}$ has at most $d + |\bs| \le d + \lfloor \alpha \rfloor$ many inputs. By Lemma~\ref{lem_43}, $\text{prod}_m^{(d+|\bs|)}$ can be implemented by a ReLU network with $\cL(\text{prod}_m^{(d+|\bs|)}), \cW(\text{prod}_m^{(d+|\bs|)}) \le c_3 m$ and $\cB(\text{prod}_m^{(d+|\bs|)}) \le 4 \vee (d-1)^2$.  
Thus, $\cL(\hat{f}_{\bi, \bs}) \le c_3 m + 2$ and $\cW(\hat{f}_{\bi, \bs}) \le c_3 m + 8 d + 4 |s| \le c_3 m+ 8 d + 4 \lfloor \alpha \rfloor $. Furthermore, $\cB(\hat{f}_{\bi,\bs}) \le 4  \vee \frac{1}{\delta} \le 1/\delta$, when $\delta$ is small enough. With this $\hat{f}_{\bi, \bs}$, we observe from Lemma~\ref{lem_43} that, 
\begin{align}
    \left|\hat{f}_{\bi, \bs}(x) - \zeta(x - \theta^{\bi})  \left(x - \theta^{\bi}\right)^{\bs}\right| \le \frac{d+ \lfloor \alpha \rfloor}{2^{2m -1}}, \, \forall x \in S.
\end{align}
Finally, let, $\hat{f}(x) =  \sum_{\bi \in \cI^\dagger} \sum_{|\bs| \le \lfloor \alpha \rfloor}  a_{\bi, \bs}  \hat{f}_{\bi, \bs}(x)$. Clearly, $\cL(\hat{f}_{\bi, \bs}) \le c_3 m + 3$ and $\cW(\hat{f}_{\bi, \bs}) \le \lfloor \alpha \rfloor^d \left(c_3 m + 8 d + 4 \lfloor \alpha \rfloor \right)$. This implies that, 
\begin{align*}
    |\hat{f}(x) - \tilde{f}(x)| \le & \sum_{\bi \in \cI^\dagger: \|x - \theta^{\bi}\|_\infty \le 2 \epsilon } \sum_{|\bs| \le \lfloor \alpha \rfloor}  |a_{\bi, \bs} | \zeta(x - \theta^{\bi})  |\hat{f}_{\bi \bs}(x) -  \left(x - \theta^{\bi}\right)^{\bs}| \\
    \le & 2^d  \sum_{|\bs| \le \lfloor \alpha \rfloor}  |a_{\theta, \bs} | \left|\hat{f}_{\theta^{\bi(x)}, \bs}(x) - \zeta_{\bepsilon, \bdelta}(x - \theta^{(\bi(x)})  \left(x - \theta^{\bi(x)}\right)^{\bs}\right| \\
    \le & \frac{ (d+\lfloor \alpha \rfloor) C}{2^{2m - d - 1}}.
\end{align*}
 
We thus get that if $x \in S$,
\begin{align}
    |f(x) - \hat{f}(x)| \le & |f(x) - \tilde{f}(x)| + |\hat{f}(x) - \tilde{f}(x)| \le  C 2^{d +\alpha + 1} \epsilon^\alpha +  \frac{ (d+\lfloor \alpha \rfloor) C}{2^{2m - d -1}}.
\end{align}

We also note that by construction $\|\hat{f}\|_\infty \le C^\prime$, for some constant $C^\prime$ that might depend on $C, \alpha$ and $d$. Hence,
\begingroup
\allowdisplaybreaks
\begin{align*}
    \|f - \hat{f}\|_{\fL_p(\mu)}^p = &\int_{S} |f(x) - \hat{f}(x)|^p d\mu(x) + \int_{S^\complement } |f(x) - \hat{f}(x)|^p d\mu(x)\\
    \le &  \left(C 2^{d + \alpha +1} \epsilon^\alpha +  \frac{ (d+\lfloor \alpha \rfloor)^3C}{2^{2m -d -1}}\right)^p \mu\left(S\right) + (C + C^\prime) \mu(S^\complement)\\
   \lesssim & \left(C 2^{d + \alpha +1} \epsilon^\alpha +  \frac{ (d+\lfloor \alpha \rfloor)^3C}{2^{2m -d -1}}\right)^p + \epsilon^{p\alpha}\\
   \implies  \|f - \hat{f}\|_{\fL_p(\mu)} \lesssim & C 2^{d + \alpha +1} \epsilon^\alpha +  \frac{ (d+\lfloor \alpha \rfloor) C}{2^{2m -d - 1}} + \epsilon^\alpha \lesssim \epsilon^\alpha + 4^{-m}.
\end{align*}
\endgroup
Taking $\epsilon \asymp \eta^{1/\alpha}$ and $m \asymp \log(1/\eta)$ ensures that $\|f - \hat{f}\|_{\fL_p(\mu)}  \le \eta$. We note that $\hat{f}$ has $|\cI^\dagger| \le 6^d \sN_\epsilon(S) \lesssim 6^d \epsilon^{-s} $ many networks with depth $c_3 m + 3$ and number of weights $ \lfloor \alpha \rfloor^d \left(c_3 m + 8 d + 4 \lfloor \alpha \rfloor \right)$. Thus, 
$\cL(\hat{f}) \le c_3 m + 4$ and $\cW(\hat{f}) \le  \epsilon^{-s} (6\lfloor \alpha \rfloor)^d \left(c_3 m + 8 d + 4 \lfloor \alpha \rfloor\right)$. we thus get, \(\cL(\hat{f}) \le c_3 m + 4 
    \le c_4 \log\left(\frac{1}{\eta} \right),\)
where $c_4$ is a function of $\delta$, $\lfloor \alpha \rfloor$ and $d$. Similarly,
\begin{align*}
    \cW(\hat{f}) \le & \epsilon^{-s} (6\lfloor \alpha \rfloor)^d \left(c_3 m + 8 d + 4 \lfloor \alpha \rfloor \right)    \le  c_6 \log(1/\eta)  \eta^{-s/\alpha}.
\end{align*}
Taking $a= c_4 \vee c_6$ gives the result.
Furthermore, by construction, we note that, 
\(\cB(\hat{f}) \lesssim 1/\delta = 3/\epsilon \lesssim \eta^{-1/\alpha}.\)
\end{proof}
\subsection{Additional Results}\label{sec_c2}
\begin{lem}[Proposition 2 of \citet{yarotsky2017error}]\label{lem:b1}
     The function $f(x) = x^2$ on the segment $[0,1]$ can be approximated with any error by a ReLU network, $\text{sq}_m(\cdot)$, such that,
    \begin{enumerate}
        \item $\cL(\text{sq}_m),\,  \cW(\text{sq}_m)\le c_1 m$. 
        \item $\text{sq}_m\left(\frac{k}{2^m}\right) = \left(\frac{k}{2^m}\right)^2$, for all $k =0, 1, \dots, 2^m$.
        \item $\|\text{sq}_m - x^2\|_{\cL_\infty([0,1])} \le \frac{1}{2^{2m + 2}}$.
        \item $\cB(\text{sq}_m) \le 4$.
    \end{enumerate}
\end{lem}

\begin{lem}\label{lem:b3}
    Let $M \ge 1$, then we can find a ReLU network $\text{prod}^{(2)}_m$, such that,
    \begin{enumerate}
        \item $\cL(\text{prod}^{(2)}_m), \cW(\text{prod}^{(2)}_m) \le c_2 m$, for some absolute constant $c_2$.
        \item $\|\text{prod}^{(2)}_m - xy\|_{\cL_\infty ([-M,M] \times [-M,M])}  \le \frac{M^2}{2^{2m+1}}.$
        \item $\cB(\text{prod}^{(2)}_m) \le 4 \vee M^2$.
    \end{enumerate}
\end{lem}
\begin{proof}
    Let $\text{prod}^{(2)}_m(x,y) = M^2 \left(\text{sq}_m\left(\frac{|x+y|}{2M}\right) - \text{sq}_m\left(\frac{|x-y|}{2M}\right)\right)$. Clearly, $\text{prod}^{(2)}_m(x,y) = 0$, if $xy = 0$. We note that, $\cL(\text{prod}^{(2)}_m) \le c_1 m + 1 \le c_2 m$ and $\cW(\text{prod}^{(2)}_m) \le 2 c_1 m + 2 \le c_2 m$, for some absolute constant $c_2$. Clearly,
    \begin{align*}
        \|\text{prod}^{(2)}_m - xy\|_{\fL_\infty ([-M,M] \times [-M,M])} \le 2 M^2 \|\text{sq}_m - x^2\|_{\fL_\infty([0,1])} \le \frac{M^2}{2^{2m+1}}.
    \end{align*}
    Part 3 of the Lemma follows easily by applying part 4 of Lemma~\ref{lem:b1}.
\end{proof}
\begin{lem}\label{lem_43}
   For any $m \ge \frac{1}{2} (\log_2 (4d) - 1)$, we can construct a ReLU network $\text{prod}^{(d)}_m : \Real^d \to \Real$, such that for any $x_1, \dots, x_d \in [-1,1]$, $\| \text{prod}_m^{(d)}(x_1, \dots, x_d) - x_1\dots x_d\|_{\cL_\infty([-1,1]^d)} \le \frac{1}{2^{m}}$. Furthermore,
   \begin{enumerate}
       \item $\cL(\text{prod}^{(d)}_m) \le c_3 m$, $\cW(\text{prod}^{(d)}_m) \le c_3 m$.
       \item $\cB(\text{prod}^{(d)}_m) \le 4 $.
   \end{enumerate}
\end{lem}
\begin{proof}
    Let $M = 1$ and $d \ge 2$. We define \[\text{prod}_m^{(k)}(x_1, \dots, x_k) = \text{prod}_m^{(2)} (\text{prod}_m^{(k-1)}(x_1, \dots, x_{k-1}), x_d), \ k \ge 3.\] Clearly $\cW(\text{prod}_m^{(d)}), \cL(\text{prod}_m^{(d)}) \le c_3 d m $, for some absolute constant $c_3$. Suppose that $m \ge \frac{1}{2} (\log_2 (4d) - 1)$. We  show that $|\text{prod}_m^{(k)}(x_1, \dots, x_k)| \le 2$, for all $2 \le k \le d$. Clearly, the statement holds for $k = 2$. Suppose the statement holds for some $2 \le k <d$. Then, 
    \begingroup
    \allowdisplaybreaks
    \begin{align*}
        |\text{prod}_m^{(k+1)}(x_1, \dots, x_{k+1})| \le & \frac{2^2}{2^{2m + 1}} + |x_{k+1}| |\text{prod}_m^{(k)}(x_1, \dots, x_{k})| \\
        \le & \frac{1}{d} + |\text{prod}_m^{(k)}(x_1, \dots, x_{k})|\\
        \le & \frac{k-2}{d} + |\text{prod}_m^{(2)}(x_1, x_2)|\\
        \le & \frac{k-1}{d} + 1\\
        \le & 2.
    \end{align*}
    \endgroup
     Thus, by induction, it is easy to see that, $|\text{prod}_m^{(d)}| \le 2$. Taking $M = 2$, we get that, 
    \begingroup
    \allowdisplaybreaks
    \begin{align*}
       & \| \text{prod}_m^{(d)}(x_1, \dots, x_d) - x_1\dots x_d\|_{\cL_\infty([-1,1]^d)} \\
       = &  \| \text{prod}_m^{(2)} (\text{prod}_m^{(d-1)}(x_1, \dots, x_{d-1}), x_d) - x_1\dots x_d\|_{\cL_\infty([-1,1]^d)}\\
       \le &  \|\text{prod}_m^{(d-1)}(x_1, \dots, x_{d-1}) - x_1\dots x_{d-1}\|_{\cL_\infty([-1,1]^d)} + \frac{M^2}{2^{2m + 1}}\\
       \le & \frac{d M^2}{2^{2m + 1}}\\
       = & \frac{d}{2^{2m - 1}}.
    \end{align*}
    \endgroup
\end{proof}
\subsection{Proofs of Lemmata \ref{approx_1} and \ref{approx_2}}
\subsubsection{Proof of Lemma~\ref{approx_1}}
\lemapproxone*
\begin{proof}
We begin by noting that,
\begingroup
\allowdisplaybreaks
\begin{align}
    \|\mu - G_\sharp \nu\|_\Phi =   \|\tilde{G}_\sharp \nu - G_\sharp \nu\|_\Phi
    = & \sup_{\phi \in \Phi} \left( \int \phi(\tilde{G}(x)) d\nu(x) - \int \phi(G(x)) d\nu(x)\right) \nonumber \\
    = & \sup_{\phi \in \Phi} \left( \int (\phi(\tilde{G}(x))- \phi(G(x)) )d\nu(x)\right) \nonumber \\
    \le &  \int |\tilde{G}(x)- G(x)|^{\beta \wedge 1} d\nu(x) \label{j_8_1}\\
    \le & \|\tilde{G} - G\|_{\fL_1(\nu)}^{\beta \wedge 1}\label{j_8_2}.
\end{align}
\endgroup
Here, inequality \eqref{j_8_1} follows from observing that $\phi \in \sH^{\beta}(\Real^d, \Real, 1)$ and \eqref{j_8_2} follows from Jensen's inequality. The Lemma immediately follows from observing that $\tilde{G} \in \sH^{\alpha_g}(\Real^\ell, \Real, C_g)$ (by assumption A\ref{model1}) and applying Theorem \ref{approx} to approximate the $i$-th coordinate of $\tilde{G}$ ($i \in [d]$) and stacking the networks parallelly.
\end{proof}
\subsubsection{Proof of Lemma~\ref{approx_2}}

\lemapproxtwo*
        
        \begin{proof}
We begin by noting that,
\begingroup
\allowdisplaybreaks
\begin{align*}
   & \|(\mu,E_\sharp\mu) - (D_\sharp \nu, \nu)\|_\Psi \\
   = & \|(\mu,E_\sharp\mu) - (\mu,\tilde{E}_\sharp\mu)\|_\Psi +  \| (\tilde{D}_\sharp \nu, \nu) - (D_\sharp \nu, \nu)\|_\Psi + \|(\mu,\tilde{E}_\sharp\mu) - (\tilde{D}_\sharp \nu, \nu)\|_\Psi\\
    = & \|(\mu,E_\sharp\mu) - (\mu,\tilde{E}_\sharp\mu)\|_\Psi +  \| (\tilde{D}_\sharp \nu, \nu) - (D_\sharp \nu, \nu)\|_\Psi \\
    = & \sup_{\psi \in \Psi}\int\left(\psi(x,E(x)) -\psi(x,\tilde{E}(x))\right)d\mu(x) + \sup_{\psi \in \Psi}\int\left(\psi(D(z),z) -\psi(\tilde{D}(z),z)\right)d\nu(z)\\
    \le & \int |E(x) - \tilde{E}(x)|^{\beta \wedge 1} d\mu(x) + \int |D(z) - \tilde{D}(z)|^{\beta \wedge 1} d\nu(z)\\
    \le & \|E - \tilde{E}\|^{\beta \wedge 1}_{\fL_1(\mu)} + \|D - \tilde{D}\|^{\beta \wedge 1}_{\fL_1(\nu)}.
\end{align*}
\endgroup
 The Lemma immediately follows from observing that $\tilde{E} \in \sH^{\alpha_e}(\Real^d, \Real, C_e)$ and $\tilde{D} \in \sH^{\alpha_d}(\Real^\ell, \Real, C_d)$ (by assumption B\ref{model2}) and applying Theorem \ref{approx} to approximate the coordinate-wise and stacking the networks parallelly.
 \end{proof}
 \section{Proofs from Section \ref{sec_gen}}
 \subsection{Proof of Lemma~\ref{lem_apr_18}}
 \pflemapr*
 \begin{figure}
     \centering
\begin{tikzpicture}[scale=0.5,
roundnode/.style={circle, draw=blue!80, fill=blue!5, very thick, minimum size=10mm},
ynode/.style={circle, draw=yellow!80, fill=yellow!5, very thick, minimum size=10mm},
squarednode/.style={rectangle, draw=red!60, fill=red!5, very thick, minimum size=10mm},
]
\node[roundnode]      (x)          {$x$};
\node[squarednode] (f) [right = of x] {$f$};
\node[roundnode]      (y1) [below = of x]          {$y$};
\node[squarednode] (id) [below = of f] {$id$};
\node[roundnode] (fx) [right = of f] {$f(x)$};
\node[roundnode] (y) [right = of id] {$y$};
\node[ynode] (plus) [right = of fx] {$f(x) + y$};
\node[ynode] (minus) [right = of y] {$f(x) - y$};
\node[roundnode] (h) [above right = 0.25 cm and 4cm of y] {$h(x,y)$};
\begin{scope}[very thick,decoration={
    markings,
    mark=at position 0.5 with {\arrow{>}}}
    ] 
    \draw[postaction={decorate}] (x.east) --  (f.west) ;
\draw[postaction={decorate}] (f.east) -- (fx.west);
\draw[postaction={decorate}] (y1.east) -- (id.west);
\draw[postaction={decorate}] (id.east) -- (y.west);
\draw[magenta,very thick, postaction={decorate}] (fx.east) -- (plus.west)node[midway,above]{$1$};

\draw[cyan,postaction={decorate},very thick] (y.east) -- (minus.west) node[midway,below]{$-1$};

\draw[orange,very thick, postaction={decorate}] (plus.east) -- (h.west)node[midway,above]{$0.25$};
\draw[teal,very thick, postaction={decorate}] (minus.east) -- (h.west)node[midway,below]{$-0.25$};
\end{scope}

\begin{scope}[very thick,decoration={
    markings,
    mark=at position 0.75 with {\arrow{>}}}
    ] 
    \draw[magenta,postaction={decorate},very thick] (y.east) -- (plus.west)node[below]{$1$};
\draw[magenta,very thick, postaction={decorate}] (fx.east) -- (minus.west)node[above]{$1$};

\end{scope}
\end{tikzpicture}
\caption{A representation of the network $h(\cdot, \cdot)$. The \textcolor{magenta}{magenta} lines represent $d^\prime$ weights of value $1$. Similarly, \textcolor{cyan}{cyan} lines represent $d^\prime$ weights of value $-1$. Finally, the \textcolor{orange}{orange} and \textcolor{teal}{teal} lines represent $d^\prime$ weights (each) with values $+0.25$ and $-0.25$, respectively. The identity map takes $2 d^\prime \cL(f)$ many weights \citep[Remark 15 (iv)]{JMLR:v21:20-002}. The \textcolor{magenta}{magenta}, \textcolor{cyan}{cyan}, \textcolor{orange}{orange} and \textcolor{teal}{teal} connections take $6d^\prime$ many weights. All activations are taken to be ReLU, except the output of the yellow nodes, whose activation is $\sigma(x) = x^2$.  }
\label{fig:network}
\end{figure}
 \begin{proof}
       We let, $h(x,y) = y^\top f(x)$ and let $\cH = \{h(x,y) = y^\top f(x) : f \in \cF\}$. Also, let, $\cT = \{(h(X_i, e_\ell)|_{i \in [n], \ell \in [d^\prime]}) \in \Real^{n d^\prime} : h \in \cH\}$. By construction of $\cT$, it is clear that, $\cN (\epsilon; \cF_{|_{X_{1:n}}}, \ell_\infty) = \cN(\epsilon; \cT, \ell_\infty)$. We observe that, 
\[h(x,y) = \frac{1}{4}(\|y + f(x)\|_2^2 - \|y - f(x)\|_2^2)\]
Clearly, $h$ can be implemented by a network with $\cL(h) = \cL(f) + 3$ and $\cW(h) = \cW(f) + 6d^\prime + 2d^\prime \cL(f)$ (see Figure \ref{fig:network} for such a construction). Let $B = \max_{1 \le i \le d}\sup_{f\in \cF} \|f_j\|_\infty$, where $f_j$ denotes the $j$-th component of $f$. Thus, from Theorem 12.9 of \cite{anthony1999neural} (see Lemma~\ref{lem_anthony_bartlett}), we note that, if $n \ge \operatorname{Pdim}(\cH)$, \(\cN(\epsilon; \cT, \ell_\infty) \le  \left(\frac{2eBn d^\prime}{\epsilon \pdim(\cH)}\right)^{\pdim(\cH)}\). Furthermore, by Lemma~\ref{lem_bartlett}, \(\pdim(\cH) \lesssim (W + 6 d^\prime + 2 d^\prime L ) (L+3) \left(\log (W + 6 d^\prime + 2 d^\prime L) + L + 3\right)\). 
This implies that, 
\begingroup
\allowdisplaybreaks
\begin{align*}
    \log \cN(\epsilon; \cH, \ell_\infty) \lesssim & \pdim(\cH) \log\left(\frac{2eBn d^\prime}{\epsilon \pdim(\cH)}\right)\\
    \lesssim & \pdim(\cH) \log\left(\frac{n d^\prime}{\epsilon }\right)\\
    \lesssim & (W + 6 d^\prime + 2 d^\prime L ) (L+3) \left(\log (W + 6 d^\prime + 2 d^\prime L) + L + 3\right)\log\left(\frac{n d^\prime}{\epsilon }\right).
\end{align*}
\endgroup

 \end{proof}
 \subsection{Proof of Corollary \ref{cor1}}
 \pfcorone*
    \begin{proof}
 Let $\cV = \{v_1, \dots, v_r\}$ be an optimal $\ell_\infty$ $\epsilon$-cover of $\sG_{|_{Z_{1:m}}}$. Clearly, $\log r \lesssim W_g L_g (\log W_g + L_g) \log\left(\frac{m d}{\epsilon }\right)$, by Lemma~\ref{lem_apr_18}. Suppose that $\cF_\delta = \{f_1, \dots, f_k\}$ be an optimal $\delta$-cover of $\Phi$ w.r.t. the $\ell_\infty$-norm. By Lemma~\ref{kt_og}, we know that, $\log k \le \delta^{-\frac{d}{\beta}}$. We note that, for any $f\in \cF$, we can find $\tilde{f} \in \cF_\delta$, such that, $\|f-\tilde{f}\|_\infty \le \delta$. For any $Z = Z_{1:m}$, let $\|Z - v^{Z}\|_\infty \le \epsilon$, with $v^{Z} \in \cV$.
 \begin{align*}
     |f(Z_j) - \tilde{f}(v^{Z}_j)| \le & |f(Z_j) - f(v^{Z}_j)| + |f(v^{Z}_j) - \tilde{f}(v^Z_j)| \\
     \lesssim &  \|Z_j - v^Z_j\|_\infty^{\beta \wedge 1} + \| f - \tilde{f}\|_\infty \lesssim \epsilon^{\beta \wedge 1} + \delta.
 \end{align*}
 Taking $\delta \asymp \eta/2$ and $\epsilon \asymp \left(\frac{\eta}{2}\right)^{1/(\beta \wedge 1)}$, we conclude that $\max_{1 \le j \le m}|f(Z_j) - \tilde{f}(v^{Z}_j)| \le \eta$. Thus, $\{(\tilde{f}(v_1), \dots, \tilde{f}(v_m)): \tilde{f} \in \cF_\delta, \, v \in \cV\}$
 constitutes an $\eta$-net of $(\Phi \circ \sG)_{|_{Z_{1:m}}}$. Hence,
 \begingroup
 \allowdisplaybreaks
 \begin{align*}
     \log\cN(\eta; (\Phi \circ \sG)_{|_{Z_{1:m}}}, \ell_\infty) \le & \log \left(\cN(\delta; \Phi, \ell_\infty) \times\cN(\epsilon;  \sG_{|_{Z_{1:m}}}, \ell_\infty)\right) \\
     \lesssim & \eta^{-\frac{d}{\beta}} + W_g L_g (\log W_g + L_g) \log\left(\frac{m d}{\eta^{1/(\beta \wedge 1)} }\right) \\
     \lesssim & \eta^{-\frac{d}{\beta}} + W_g L_g (\log W_g + L_g) \log\left(\frac{m d}{\eta }\right)
 \end{align*}
 \endgroup
 The latter part of the lemma follows from a similar calculation.
     \end{proof}
 \subsection{Proof of Lemma~\ref{gen}}
 Before we discuss the proof of Lemma~\ref{gen}, we state and prove the following result.
 \begin{lem}\label{lem_f1}
     \(\inf_{0 < \delta \le B}\left(\delta + \frac{1}{\sqrt{n}}\int_\delta^B \epsilon^{-\tau} d\epsilon \right) \lesssim \begin{cases}
         & n^{-1/2}, \text{ if } \tau <1\\
         & n^{-1/2} \log n, \text{ if } \tau = 1\\
        & n^{-\frac{1}{2 \tau}}, \text{ if } \tau >1.
     \end{cases} \)
 \end{lem}
 \begin{proof}
     If $\tau <1$, 
     \begingroup
     \allowdisplaybreaks
     \begin{align*}
         \inf_{0 < \delta \le B}\left(\delta + \frac{1}{\sqrt{n}}\int_\delta^B \epsilon^{-\tau} d\epsilon \right) = & \inf_{0 < \delta \le B}\left(\delta + \frac{1}{\sqrt{n}}\frac{\epsilon^{-\tau+1}}{-\tau+1}\bigg|_\delta^B  \right) \\
         = &  \inf_{0 < \delta \le B}\left(\delta + \frac{1}{\sqrt{n}}\frac{B^{1-\tau} - \delta^{1-\tau}}{1-\tau}  \right) \\
         \lesssim & \frac{1}{\sqrt{n}}.
     \end{align*}
     \endgroup
     If $\tau =1$, we observe that
     \begin{align*}
          \inf_{0 < \delta \le B}\left(\delta + \frac{1}{\sqrt{n}}\int_\delta^B \epsilon^{-\tau} d\epsilon \right)  = & \inf_{0 < \delta \le B}\left(\delta + \frac{1}{\sqrt{n}} \log \epsilon\big|_\delta^B \right) \\
          = & \inf_{0 < \delta \le B}\left(\delta + \frac{1}{\sqrt{n}} (\log B - \log \delta) \right)  \\
          \lesssim & \frac{\log n}{\sqrt{n}},
     \end{align*}
     when $\delta \asymp n^{-1/2}$. Furthermore, if $\tau >1$, 
     \begin{align*}
         \inf_{0 < \delta \le B}\left(\delta + \frac{1}{\sqrt{n}}\int_\delta^B \epsilon^{-\tau} d\epsilon \right) 
         = & \inf_{0 < \delta \le B}\left(\delta + \frac{1}{\sqrt{n}}\frac{\delta^{1-\tau} - B^{1-\tau}}{\tau-1}  \right) \lesssim & n^{-\frac{1}{2 \tau}}, \text{when, } \delta \asymp n^{-\frac{1}{2 \tau}}.
     \end{align*}
 \end{proof}
 
 \pfgen*
 
    \begin{proof}
 From Dudley's chaining argument, we note that,
 \begingroup
 \allowdisplaybreaks
 \begin{align}
     & \E \|\hnu_m - \nu\|_{\Phi \circ \sG} \nonumber\\
     \lesssim & \inf_{0 < \delta \le 1/2 }\left(\delta + \frac{1}{\sqrt{m}} \int_\delta^{1/2} \sqrt{\log\cN(\epsilon; \Phi \circ \sG, \fL_2(\hnu_m))} d\epsilon\right) \nonumber \\
     \le & \inf_{0 < \delta \le 1/2 }\left(\delta + \frac{1}{\sqrt{m}} \int_\delta^{1/2} \sqrt{\log\cN(\epsilon; (\Phi \circ \sG)_{|_{Z_{1:m}}}, \ell_\infty)} d\epsilon\right) \nonumber \\
      \lesssim & \inf_{0 < \delta \le 1/2 }\left(\delta + \frac{1}{\sqrt{m}} \int_\delta^{1/2} \sqrt{ \epsilon^{-\frac{d}{\beta}} + W_g L_g (\log W_g + L_g) \log\left(\frac{m d}{\epsilon}\right)} d\epsilon\right) \label{f1} \\
       \le & \inf_{0 < \delta \le 1/2 }\left(\delta + \frac{1}{\sqrt{m}} \int_\delta^{1/2} \left(\sqrt{ \epsilon^{-\frac{d}{\beta}}} +\sqrt{ W_g L_g (\log W_g + L_g) \log\left(\frac{m d}{\epsilon}\right)}\right) d\epsilon\right) \nonumber\\
       \le & \inf_{0 < \delta \le 1/2 }\left(\delta + \frac{1}{\sqrt{m}} \int_\delta^{1/2}  \epsilon^{-\frac{d}{2\beta}} d\epsilon + \frac{1}{\sqrt{m}} \int_0^{1/2} \sqrt{ W_g L_g (\log W_g + L_g) \log\left(\frac{m d}{\epsilon}\right)} d\epsilon\right)\nonumber\\
       \lesssim & \inf_{0 < \delta \le 1/2 }\left(\delta + \frac{1}{\sqrt{m}} \int_\delta^{1/2}  \epsilon^{-\frac{d}{2\beta}} d\epsilon +   \sqrt{ \frac{W_g L_g (\log W_g + L_g) \log(md)}{m}} \right)\nonumber\\
       \lesssim & m^{-\beta/d} \vee (m^{-1/2} \log m) + \sqrt{ \frac{W_g L_g (\log W_g + L_g) \log(md)}{m}} \label{f2}
 \end{align}
 \endgroup
 In the above calculations, \eqref{f1} follows from Corollary \ref{cor1} and \eqref{f2} follows from Lemma~\ref{lem_f1}. The latter part of the lemma follows from a similar calculation as above.
 \end{proof}
 \subsection{Proof of Lemma~\ref{Lem_3.7}}\label{app_g4}
We first let $\Lambda = \{(\mu, E_\sharp \mu): E \in \sE\}$. We first argue that $d^\ast_\alpha\left(\Lambda\right) \le d^\ast_\alpha(\mu) $. To see this, let $s>d^\ast_\alpha(\mu) + \ell $. We note that for any $\epsilon >0$, we can find $S \subseteq [0,1]^d$, such that $\mu(S) \ge 1-\epsilon^{\frac{s\alpha}{s - 2 \alpha}}$ and $\cN(\epsilon; S, \ell_\infty) = \sN_\epsilon(\mu, \epsilon^{\frac{s\alpha}{s - 2 \alpha}})$. We let $S^\prime = S \times [0,1]^\ell$. Clearly, $[0,1]^\ell$ can be covered with $\epsilon^{-\ell}$-many $\ell_\infty$-balls, which implies that $S^\prime$ can be covered with at most $\sN_\epsilon(\mu, \epsilon^{\frac{s\alpha}{s - 2 \alpha}}) \times \epsilon^{-\ell}$-many $\ell_\infty$-balls. Thus, $\sN_\epsilon(\Lambda, \epsilon^{\frac{s\alpha}{s - 2 \alpha}}) \le \sN_\epsilon(\mu,\epsilon^{\frac{s\alpha}{s - 2 \alpha}}) \times \epsilon^{-\ell}$. Thus, 
\[\limsup_{\epsilon \downarrow 0}\frac{\log \sN_\epsilon\left(\Lambda, \epsilon^{\frac{s \alpha}{s - 2 \alpha}}\right)}{\log(1/\epsilon)} \le \limsup_{\epsilon \downarrow 0}\frac{\log \sN_\epsilon(\mu, \epsilon^{\frac{s \alpha}{s - 2 \alpha}})}{\log(1/\epsilon)} + \ell \le s+\ell.\]
 Here, the last inequality follows from the fact that $s>d^\ast_\alpha(\mu)$ 
 and applying Lemma~\ref{lem_c1}. We now note that, $\frac{(s+\ell)\alpha}{s+\ell - 2\alpha} \le \frac{s \alpha}{s - 2\alpha} \implies \epsilon^{\frac{(s+\ell)\alpha}{s+\ell - 2\alpha}} \ge \epsilon^{\frac{s \alpha}{s - 2\alpha}} \implies \sN_\epsilon(\Lambda, \epsilon^{\frac{(s+\ell)\alpha}{s+\ell - 2\alpha}}) \le \sN_\epsilon(\Lambda, \epsilon^{\frac{s \alpha}{s - 2\alpha}})$. Thus,
 \[\limsup_{\epsilon \downarrow 0}\frac{\log \sN_\epsilon\left(\Lambda, \epsilon^{\frac{(s+\ell) \alpha}{s + \ell - 2 \alpha}}\right)}{\log(1/\epsilon)} \le \limsup_{\epsilon \downarrow 0}\frac{\log \sN_\epsilon\left(\Lambda, \epsilon^{\frac{s \alpha}{s - 2 \alpha}}\right)}{\log(1/\epsilon)} \le s+\ell \implies d^\ast_\alpha(\Lambda) \le s + \ell.\]
Thus, for any $s>d^\ast_\alpha(\mu)$, $d^\ast_\alpha(\Lambda) \le s+ \ell$. 
 
\lemtf*
\begin{proof}Let, $Y_E = \sup_{\psi \in \Psi}\left(\frac{1}{n}\sum_{i=1}^n  \psi(X_i,E(X_i)) - \int \psi(x,E(x))d\mu(x)\right)$. Clearly, if one replaces $X_i$ with $X_i^\prime$, the change in $Y_E$ is at most $\frac{2C_e}{n}$. Thus, by the bounded difference inequality,
\begin{equation}
    \label{i2}
    \prob \left(|Y_E - \E(Y_E)|\ge t\right) \le 2 \exp\left(-\frac{n t^2}{2C_e^2}\right).
\end{equation}
Inequality \eqref{i2} implies that:
\begin{equation}
    \label{i3}
 \E \exp\left(\lambda (Y_E - \E Y_E)\right) \le \exp\left(\frac{c^\prime \lambda^2}{n}\right),
\end{equation}
for some $c^\prime>0$, by Proposition 2.5.2 of \citet{vershynin2018high}.
We thus note that,
\begingroup
\allowdisplaybreaks
\begin{align}
    &\E \sup_{\psi \in \Psi, E \in \sE} \frac{1}{n} \sum_{i=1}^n \psi(X_i, E(X_i)) - \int \psi(x,E(x))d\mu(x)\nonumber\\
    = &\E \sup_{ E \in \sE} Y_E \nonumber\\
   = &\E \sup_{E \in \sE} \left(Y_E - \E(Y_E) + \E(Y_E)\right)\nonumber\\
   \le  &\E \sup_{E \in \sE} \left(Y_E - \E(Y_E) \right) + \sup_{E \in \sE} \E Y_E \nonumber\\
   =  &\E \sup_{\tilde{E} \in \cC(\epsilon; \sE, \ell_\infty)}\sup_{E \in B_{\ell_\infty}(\tilde{E},\epsilon)} \left(Y_E - \E(Y_E) \right) + \sup_{E \in \sE} \E Y_E \nonumber\\
   \lesssim  &\E \sup_{\tilde{E} \in \cC(\epsilon; \sE, \ell_\infty)}\left(Y_{\tilde{E}} - \E(Y_{\tilde{E}}) \right) +  \epsilon^{\beta \wedge 1}+ \sup_{E \in \sE} \E Y_E \nonumber\\
   \lesssim  &\sqrt{\frac{\cN(\epsilon; \sE, \ell_\infty)}{n} } + 4 C \epsilon^{\beta \wedge 1}+ \sup_{E \in \sE} \E Y_E \label{i4}\\
   \lesssim  &\left(\frac{W_e}{n} \log \left(\frac{2L_eB_e^{L_e} (W_e+1)^{L_e}}{\epsilon}\right) \right)^{1/2} + 4 C \epsilon^{\beta \wedge 1}+ \sup_{E \in \sE} \E Y_E \label{i5}
\end{align}
\endgroup
Here, \eqref{i4} follows bounding the expectation of the maximum of subgaussian random variables  \citep[Theorem 2.5]{boucheron2013concentration}. Inequality \eqref{i5} follows from Lemma~\ref{lem_nakada}. For any $E \in \sE$, by applying Theorem~\ref{gen_kolmo}, we note that, 
\[\E Y_E = \E \sup_{\psi \in \Psi} \int \psi(x,E(x)) d\mu_n(x) - \int \psi(x, E(x)) d\mu(x) \lesssim n^{-\frac{\beta}{s + \ell}}.\]
Plugging this into \eqref{i5} and taking $\epsilon = n^{-\frac{1}{2 (\beta \wedge 1)}}$, we get that,
\begingroup
\allowdisplaybreaks
\begin{align*}
    & \E \sup_{\psi \in \Psi, E \in \Psi} \frac{1}{n} \sum_{i=1}^n \psi(X_i, E(X_i)) - \int \psi(x,E(x))d\mu(x) \\
    \lesssim & \sqrt{\frac{W_e}{n} \log \left(2L_eB_e^{L_e} (W_e+1)^{L_e} n^{\frac{1}{2 (\beta \wedge 1)}}\right) } + \frac{1}{\sqrt{n}}+ n^{-\frac{\beta}{s + \ell}}\\
    \lesssim & \sqrt{\frac{W_e}{n} \log \left(2L_eB_e^{L_e} (W_e+1)^{L_e} n^{\frac{1}{2 (\beta \wedge 1)}}\right) } + n^{-\frac{\beta}{s + \ell}}.
\end{align*}
\endgroup
The last inequality follows since $s > d^\ast_\beta(\mu) \ge 2 \beta$.
\end{proof}
\section{Proofs of the Main Results (Section \ref{main results})}\label{pf_main}
\subsection{Proof of Theorem \ref{main_gan}}
\pfgan*
\begin{proof}
We note from the oracle inequality (Lemma \ref{oracle}, inequality \eqref{o1}) 
 \begin{align*}
     \|\mu - (\hG_n)_\sharp \nu\|_{\Phi} \le \inf_{G \in \sG} \|\mu - G_\sharp \nu\|_{\Phi} + 2 \|\mu -\hmu_n\|_{\Phi}
 \end{align*}
 If $\sG $ is chosen as in Lemma~\ref{approx_1}, then, for large $n$,
 \begin{align*}
     \|\mu - (\hG_n)_\sharp \nu\|_{\Phi} \lesssim \epsilon ^{\beta \wedge 1} + n^{-\beta/s}.
 \end{align*}
 We take $\epsilon = n^{-\frac{\beta}{(\beta \wedge 1)s}}$. This makes, $L_g \lesssim \log n$, $W_g \lesssim n^{\frac{\ell \beta}{s \alpha_g (\beta \wedge 1)}} \log n$ and gives the desired bound \eqref{m1}.

To derive \eqref{m2}, we note that, for the estimate $\hG_{n,m}$, we note that, for the choice of $\sG$ as in Lemma~\ref{approx_1}, 
 \begin{align*}
     & \|\mu - (\hG_{n,m})_\sharp \nu\|_{\Psi} \\
     \le & \inf_{G \in \sG} \|\mu - G_\sharp \nu\|_{\Phi} + 2 \|\mu -\hmu_n\|_{\Phi} + 2 \|\nu - \hnu_m\|_{\Phi \circ \sG}\\
     \lesssim & \epsilon^{\beta \wedge 1} + n^{-\beta/s} + m^{-\beta/d} \vee (m^{-1/2} \log m) + \sqrt{ \frac{W_g L_g (\log W_g + L_g) \log(md)}{m}}\\
     \lesssim & \epsilon^{\beta \wedge 1} + n^{-\beta/s} + m^{-\beta/d} \vee (m^{-1/2} \log m) + \sqrt{ \frac{\epsilon^{-\ell/\alpha_g} (\log(1/\epsilon))^3\log(md)}{m}}
 \end{align*}
 We choose, $\epsilon \asymp m^{-\frac{1}{2 (\beta \wedge 1) + \frac{\ell}{\alpha_g}}} = m^{-\frac{\alpha_g  }{2 \alpha_g (\beta \wedge 1) + \ell}}$. This makes,
 \begin{align*}
     \|\mu - (\hG_{n,m})_\sharp \nu\|_{\Psi} 
     \lesssim &   n^{-\beta/s} + m^{-\beta/d} \vee (m^{-1/2} \log m) + (\log m)^2 m^{-\frac{\alpha_g(\beta \wedge 1)}{2 \alpha_g (\beta \wedge 1)+ \ell}}\\
     \lesssim & n^{-\beta/s} + (\log m)^2 m^{-\frac{1}{\max\left\{2 + \frac{\ell}{\alpha_g(\beta \wedge 1)}, \frac{d}{\beta}\right\}}}.
 \end{align*}
 We take $\upsilon_n = \inf \left\{m \in \mathbb{N}: (\log m)^2 m^{-\frac{1}{\max\left\{2 + \frac{\ell}{\alpha_g(\beta \wedge 1)}, \frac{d}{\beta}\right\}}} \le n^{-\beta/s}\right\}$. Clearly, this if  $m \ge \upsilon_n$, the choice of $\epsilon$ gives us the desired bound, i.e. \eqref{m2} and also satisfies the bounds on $L_g$ and $W_g$.
 \end{proof}
 \subsection{Proof of Theorem \ref{main_bigan}}
 \pfbigan*
 \begin{proof}
 From Lemma~\ref{oracle}, \eqref{o3}, we note that,
 \begin{align*}
      \|(\mu, (\hE_n)_\sharp \mu) - (\hD_n)_\sharp \nu, \nu)\|_{\Psi}\le &  \inf_{D \in \sD, \, E \in \sE}\|(\mu, E_\sharp \mu) - (D_\sharp \nu, \nu)\|_{\Psi} + 2 \|\hmu_n - \mu\|_{\sF_1} \\
      \lesssim & \epsilon_e^{\beta \wedge 1} + \epsilon_d^{\beta \wedge 1} + 2 \|\hmu_n - \mu\|_{\sF_1},
 \end{align*}
 if $\sE = \cR(L_e, W_e, B_e, 2 C_e)$ and $\sD = \cR\cN(L_d, W_d, B_d, 2 C_d)$, with
 \begin{align*}
        & L_e \lesssim \log(1/\epsilon_e),\, W_e \lesssim \epsilon_e^{-s_1/\alpha_e} \log(1/\epsilon_e), \text{ and } B_e \lesssim \epsilon^{-1/\alpha_e} ;\\
        & L_d \lesssim \log(1/\epsilon_d), \, W_d \lesssim \epsilon_d^{-\ell/\alpha_d} \log(1/\epsilon_d),\text{ and } B_d \lesssim \epsilon_d^{-1/\alpha_d} ,
     \end{align*}    
 by Lemma~\ref{approx_2}. Applying Lemma~\ref{Lem_3.7}, we get,
  \begin{align*}
     & \|(\mu, (\hE_n)_\sharp \mu) - (\hD_n)_\sharp \nu, \nu)\|_{\Psi} \\
     \le  &   \epsilon_e^{\beta \wedge 1} + \epsilon_d^{\beta \wedge 1} +\sqrt{\frac{W_e}{n} \log \left(2L_eB_e^{L_e} (W_e+1)^{L_e} n^{\frac{1}{2 (\beta \wedge 1)}}\right) } + n^{-\frac{\beta}{s_2 + \ell}}.
 \end{align*}
 We choose $\epsilon_e \asymp n^{-\frac{1}{2 (\beta \wedge 1) + \frac{s_1}{\alpha_e}}}$ and $\epsilon_d \asymp n^{-\frac{\beta}{ (\beta \wedge 1)(s+\ell)}}$. This makes,
  \begin{align*}
      \|(\mu, (\hE_n)_\sharp \mu) - (\hD_n)_\sharp \nu, \nu)\|_{\Psi} \lesssim &    (\log n)^2 n^{-\frac{1}{2 + \frac{s_1}{\alpha_e (\beta \wedge 1)}}} + n^{-\frac{\beta}{s_2 + \ell}} \\
      \lesssim &  n^{-\frac{1}{\max\left\{2  + \frac{s_1}{\alpha_e(\beta \wedge 1)}, \frac{s_2 + \ell}{\beta} \right\}}} (\log n)^2.
 \end{align*}
 For the latter part of the theorem, we note that, 
 \begin{align*}
    & \|(\mu, (\hE_{n,m})_\sharp \mu) - (\hD_{n,m})_\sharp \nu, \nu)\|_{\Psi}\\
    \le &  \inf_{D \in \sD, \, E \in \sE}\|(\mu, E_\sharp \mu) - (D_\sharp \nu, \nu)\|_{\Psi} + 2 \|\hmu_n - \mu\|_{\sF_1} + 2 \|\hnu_m - \nu\|_{\sF_2}\\
     \lesssim & \epsilon_e^{\beta \wedge 1} + \epsilon_d^{\beta \wedge 1} + \sqrt{\frac{W_e}{n} \log \left(2L_eB_e^{L_e} (W_e+1)^{L_e} n^{\frac{1}{2 (\beta \wedge 1)}}\right) } + n^{-\frac{\beta}{s_2 + \ell}} \\
     & + m^{-\frac{\beta}{d + \ell}} \vee (m^{-1/2} \log m)  +  \sqrt{\frac{W_d L_d (\log W_d + L_d) \log(md)}{m}}
 \end{align*}
 We choose $\epsilon_e \asymp n^{-\frac{1}{2 (\beta \wedge 1) + \frac{s_1}{\alpha_e}}}$ and $\epsilon_d \asymp m^{-\frac{1}{2 (\beta \wedge 1) + \frac{\ell}{\alpha_d}}}$. This makes,
 \begin{align*}
     &\|(\mu, (\hE_{n,m})_\sharp \mu) - (\hD_{n,m})_\sharp \nu, \nu)\|_{\Psi} \\
     \lesssim & (\log n)^2 n^{-\frac{1}{2 +\frac{s_1}{\alpha_e(\beta \wedge 1)} }} + n^{-\frac{\beta}{s_2 + \ell}}  + m^{-\frac{\beta}{d + \ell}} \vee (m^{-1/2} \log m)  +  (\log m)^2 m^{-\frac{1}{2 + \frac{\ell}{\alpha_d(\beta \wedge 1)}}}.
 \end{align*}
\end{proof}
 \section{Proofs from Section \ref{minimax bounds}}
 \subsection{Proof of Theorem~\ref{main 2}}\label{pf_main_2}
\begin{lem}\label{lem44}
    Suppose that $\gamma$ is a measure supported on $k$ points on $[0,1]^d$ and $\nu$ be an absolutely continuous distribution on $[0,1]^\ell$. Then, we can choose $W$ and $L$, such that, $W \gtrsim k$ and $\inf_{f \in \cR \cN(L,W,\infty, 1) } \sW_1(\gamma, f_\sharp \nu) =0$. 
\end{lem}
\begin{proof}
    We take the first layer of $f$ to be mapping to $\Real$ with the output being the first coordinate variable. We then pass this output through a network of width $w$ and depth $L$, such that $k \le \frac{w-d-1}{2} \lfloor (w-d-1)/(6d)\rfloor \lfloor L/2\rfloor + 2$. Since the output of the first layer is still absolutely continuous, the network can make the Wasserstein distance between $\mu$ and $f_\sharp \nu$, arbitrarily small by Lemma 3.1 of \cite{yang2022capacity}. By the choice of the network width and depth, we note that, $k \lesssim w^2 L \asymp W$. Since, by construction, the networks can be taken as linear interpolators taking values in $[0,1]^d$, the network output remains bounded by $1$ in the $\ell_\infty$-norm.  
\end{proof}
\begin{lem}\label{lem45}
    For the GAN estimators $\hG_n$ and $\hG_{n,m}$, we observe that,
    \begin{align*}
         \|\mu - (\hG_n)_\sharp \nu\|_\Phi \le & \|\hmu_n - G_\sharp \nu\|_\Phi + \|\mu - \hmu_n\|_\Phi \\ 
          \|\mu - (\hG_{n,m})_\sharp \nu\|_\Phi \le &  \|\mu_n - G_\sharp \nu\|_\Phi +   \|\hmu_n - \mu\|_\Phi + 2\|\hnu_m - \nu\|_{\Phi \circ \sG}.
    \end{align*}
\end{lem}
\begin{proof}
 For any $G \in \sG$, we observe that,
     \begin{align}
         \|\mu - (\hG_n)_\sharp \nu\|_\Phi \le  \|\hmu_n - (\hG_n)_\sharp \nu\|_\Phi + \|\mu - \hmu_n\|_\Phi 
         \le & \|\hmu_n - G_\sharp \nu\|_\Phi + \|\mu - \hmu_n\|_\Phi \nonumber 
     \end{align}
     Taking infimum on both sides w.r.t. $G \in \sG$ gives us the desired result. For the estimator $\hG_{n,m}$, we note that,
     \begingroup
     \allowdisplaybreaks
     \begin{align*}
         \|\mu - (\hG_{n,m})_\sharp \nu\|_\Phi \le & \inf_{G \in \sG}\|\hmu_n - (\hG_{n,m})_\sharp \hnu_m\|_\Phi +   \|\hmu_n - \mu\|_\Phi + \|\hnu_m - \nu\|_{\Phi \circ \sG}\\
         \le & \inf_{G \in \sG}\|\hmu_n - G_\sharp \hnu_m\|_\Phi +   \|\hmu_n - \mu\|_\Phi + \|\hnu_m - \nu\|_{\Phi \circ \sG}\\
         \le & \|\mu_n - G_\sharp \nu\|_\Phi +   \|\hmu_n - \mu\|_\Phi + 2\|\hnu_m - \nu\|_{\Phi \circ \sG}
     \end{align*}
     \endgroup
     Again taking infimum on both sides w.r.t. $G \in \sG$ gives us the desired result.  
 \end{proof}

\maintwo*

\begin{proof}
We first note that for any two random variables, $V_1$ and $V_2$, that follow the distributions $\gamma_1$ and $\gamma_2$, respectively, 
\begin{align*}
    \|\gamma_1 - \gamma_2\|_{\sH^\beta(\Real^d, \Real,1)} =  \sup_{f \in \sH^\beta(\Real^d, \Real,1)} \E(f(V_1) - f(V_2))
    \le  \E \|V_1 - V_2\|^{\beta \wedge 1} 
    \overset{(i)}{\le} & (\E \|V_1 - V_2\|)^{\beta \wedge 1} 
\end{align*}
Here $(i)$ follows from applying Jensen's inequality. Taking an infimum with respect to the joint distribution of $V_1$ and $V_2$ yields, $\|\gamma_1 - \gamma_2\|_{\sH^\beta(\Real^d, \Real,1)} \le \sW_1(V_1, V_2)^{\beta \wedge 1}$.

From Lemmas \ref{lem44} and \ref{lem45}, we note that for the GAN estimator, $\hG_n$, we note that, if $\sG = \cR\cN(L_g,W_g, \infty, R_g)$, with $W_g \ge n$
\begingroup
\allowdisplaybreaks
\begin{align*}
    \E  \|\mu - (\hG_n)_\sharp \nu\|_\Phi \le &  \E \inf_{G \in \sG}\|\hmu_n - G_\sharp \nu\|_\Phi + \E \|\mu - \hmu_n\|_\Phi \\
    \le &  \E \inf_{G \in \sG}\sW_1(\hmu_n, G_\sharp \nu)^{\beta \wedge 1} + \E \|\mu - \hmu_n\|_\Phi \\
    = & \E \|\mu - \hmu_n\|_\Phi \lesssim  n^{-\frac{\beta}{s}}
\end{align*}
\endgroup
Similarly, for the GAN estimator, $\hG_{n,m}$, if $\sG = \cR \cN (L_g, W_g, \infty, 1)$, with $L_g \ge 2$ as constant and $W_g \asymp n$, we observe that,
\begingroup
\allowdisplaybreaks
\begin{align*}
  \E \|\mu - (\hG_{n,m})_\sharp \nu\|_\Phi \le &  \E \inf_{G \in \sG}\|\mu_n - G_\sharp \nu\|_\Phi +   \|\hmu_n - \mu\|_\Phi + 2\|\hnu_m - \nu\|_{\Phi \circ \sG}    \\
  \lesssim & n^{-\beta / s} + m^{-\beta/d} \vee m^{-1/2} \log m + \sqrt{\frac{W_g L_g(\log W_g + L_g) \log (md)}{m}}\\ \lesssim & n^{-\beta / s} + m^{-\beta/d} \vee m^{-1/2} \log m + \sqrt{\frac{n \log n  \log (md)}{m}}\\
  \lesssim & n^{-\beta / s} ,
\end{align*}
\endgroup
taking $m \ge n^{d/s+1}$.
\end{proof}
\subsection{Proof of Theorems \ref{thm_30} and \ref{thm_31}}
As a first step for deriving a minimax bound, we first show that the H\"{o}lder IPM can be lower bounded by the total variation distance and the minimum separation of the support of the distributions. For any finite set, we use the notation, $\text{sep}(\Xi) = \inf_{\xi, \xi^\prime \in \Xi: \xi \neq \xi^\prime} \|\xi - \xi^\prime\|_{\infty}$. 
\begin{lem}\label{lem_28}
    Let $\Xi$ be a finite subset of $\Real^p$ and let, $P,Q \in \Pi_\Xi$. Then, 
    \[\|P-Q\|_{\sH^\beta(\Real^d, \Real, C)} \gtrsim \, (\text{sep}(\Xi))^\beta \|P-Q\|_{\text{TV}}.\]
\end{lem}

\begin{proof}
    Let $b(x) = \exp\left(\frac{1}{x^2 -1} \right) \one\{|x| \le 1\}$ be the standard bump function on $\Real$. For any $x \in \Real^d$ and $\delta \in (0,1]$, we let, $h_\delta(x) = a \delta^\beta \prod_{j=1}^d b(x_j/\delta)$. Here $a$ is such that $a b(x) \in \sH^\beta(\Real, \Real, C)$. It is easy to observe that $h_\delta \in \sH^\beta(\Real^d, \Real, C)$.
Let $P$ and $Q$ be two distributions on $\Xi = \{\xi_1, \dots, \xi_k\}$. Let $\delta = \frac{1}{3}\min_{i \neq j} \|\xi_i - \xi_j\|_\infty$ We define $h^\star(x) = \sum_{i=1}^k \alpha_i h_\delta(x-\xi_i)$, with $\alpha_i \in \{-1,+1\}$, to be chosen later. Since the individual terms in $h^\star$ are members of $\sH^\beta(\Real^d, \Real, C)$ and have disjoint supports, $h^\star\in \sH^\beta(\Real^d, \Real, C) $. We take $\alpha_i = 2\one\{P(\xi_i) \ge Q(\xi_i)\}-1$.
Thus,
\begingroup
\allowdisplaybreaks
\begin{align*}
    \|P-Q\|_{\sH^\beta(\Real^d, \Real, C)} \ge   \int h^\ast dP - \int h^\ast dQ =  &\sum_{i=1}^k a \delta^\beta \alpha_i (P(\xi_i) - Q(\xi_i))\\
    = & a \delta^\beta \sum_{i=1}^k  |P(\xi_i) - Q(\xi_i)|\\
    = & 2 a \delta^\beta \|P-Q\|_{\text{TV}}\\
    \gtrsim & \, (\text{sep}(\Xi))^\beta \|P-Q\|_{\text{TV}}.
\end{align*}
\endgroup
\end{proof}

To derive our main result, we also need to show that there indeed exists a distribution $\mu_\star \in \fM_{d_\star}$, such that, $d_\ast(\mu_\star) = d_\star$. Proposition \ref{prop_45} gives a construction that ensures the existence of such a distribution. Before proceeding, we recall the definition of the Hausdorff dimension of a measure. We adopt the definition stated by \cite{weed2019sharp}.
\begin{defn}[Hausdorff dimension \citep{weed2019sharp}]
\label{defn_hausdorff}
The $d$-Hausdroff measure of a set $S$ is defined as, \[\mathbb{H}^d(S) := \liminf_{\epsilon \downarrow 0} \left\{\sum_{k=1}^\infty r_k^d : S \subseteq \sum_{k=1}^\infty B_\varrho (x_k, r_k), r_k \le \epsilon, \, \forall k\right\}.\]
    The Hausdorff dimension of the set $S$ is defined as,
    \(\operatorname{dim}_H(S) := \inf\{d: \mathbb{H}^d(S) = 0\}.\)
    For a measure $\mu$, then the Hausdorff dimension of $\mu$ is defined as:
    \(\operatorname{dim}_H(S) : = \inf\left\{\operatorname{dim}_H(S): \mu(S) = 1\right\}.\)
\end{defn}
\begin{propos} \label{prop_45}
     For any $d_\star \in [0, d]$, one can find a distribution $\mu_\star$ on $[0,1]^d$, such that, $d_\ast(\mu_\star) = \overline{\text{dim}}_M(\mu_\star) = d_\star$.
\end{propos}
   
\begin{proof}
If $d_\star \in \mathbb{N}$, then, the distribution which has is i.i.d. uniform on the first $d_\star$ coordinates and is identically $0$, otherwise satisfies $d_\ast(\mu_\star) = d_\star$. Now suppose that $d_\star \not \in \fN$. Let $\alpha = 1 - \exp\left(\left(1 - \frac{1}{d_\star - \lfloor d_\star \rfloor}\right) \log 2\right)$. It is well known that the Cantor set $C_\alpha$, which is constructed by removing the middle $\alpha$ (instead of $1/3$-rd in regular Cantor set) has a Hausdorff and upper Minkowski dimension of $\frac{\log 2}{ \log 2 - \log(1-\alpha) } = d_\star - \lfloor d_\star \rfloor$. A proof of this result can be found in \citep[Example 1.1.3]{bishop2017fractals} and \citep[Chapter 4.8]{ziemer2017modern}.  Let $\operatorname{Cf}_\alpha:[0,1] \to [0,1]$ be the corresponding Cantor function. It is easy to check that $\operatorname{Cf}_\alpha$ satisfies all the properties of a cumulative distribution function (CDF). We let $U_1$ be a random variable, whose CDF is $\operatorname{Cf}_\alpha$. It is known \citep{presnell2022geometric} that the support of $U_1$ is $C_\alpha$. 
 To construct $\mu_\star$, we define $U_2, \ldots, U_{\lfloor d_\star \rfloor + 1} \overset{i.i.d.}{\sim} \operatorname{Unif}([0,1])$ and let, $U = (U_1, \ldots, U_{\lfloor d_\star \rfloor + 1}, 0, \ldots, 0) $. The distribution of $U$ is denoted as $\mu_\star$. First note that, 
\[\overline{\text{dim}}_M(\mu) = \overline{\text{dim}}_H(C_\alpha \times [0,1]^{\lfloor d_\star \rfloor} ) \le \overline{\text{dim}}_M(C_\alpha) + \lfloor d_\star \rfloor = d_\star.\]
Similarly, we note that, 
\[\text{dim}_H(\mu)  = \text{dim}_H(C_\alpha \times [0,1]^{\lfloor d_\star \rfloor} ) \ge \text{dim}_H(C_\alpha) + \lfloor d_\star \rfloor = d_\star.\]
 Applying Proposition~\ref{lem_weed}, we observe that $d_\ast(\mu_\star) = d_\star$.
 \end{proof}
\begin{cor}\label{cor46}
    Suppose that $d \ge 2\beta$ and $d^\star \in [2 \beta , d]$. Then, one can find a distribution $\mu^\star$ on $[0,1]^d$, such that, $d^\ast_\beta(\mu^\star) = d_\ast(\mu^\star) = d^\star$.
\end{cor}
\begin{proof}
    It is easy to see that by Proposition \ref{prop_45}, one can find $\mu^\star$, such that $d_\ast(\mu^\star) = \overline{\text{dim}}_M(\mu^\star) = d^\star$. If $d^\star \in [2 \beta , d]$, then, by Proposition~\ref{lem_weed}, we note that, $d^\star = d_\ast(\mu^\star) \le d^\ast_\beta(\mu^\star) \le \overline{\text{dim}}_M(\mu^\star) = d^\star$. Hence the result.
\end{proof}
\thmtwentynine*
\begin{proof}
 Let  $\mu^\star \in \fM^{d^\star, \beta}$ be such that $d_\ast(\mu^\star) = d^\star$ (this existence is guaranteed by Corollary \ref{cor46}). 
 Now let, $s<d_\ast(\mu) = d^\star$. We can find a $\tau^\star$ and $\epsilon^\star$, such that $\sN_\epsilon(\mu^\star, \tau) \ge \epsilon^{-s}$, for any $\epsilon \in (0, \epsilon^\star]$ and $\tau \in (0, \tau^\star]$. We can thus find a set $S^\star \subseteq [0,1]^d$, such that $\cM(\epsilon; S^\star, \varrho) \ge \cN(\epsilon; S^\star, \varrho) \ge \epsilon^{-s}$. We take $n \ge (128 (\epsilon^\star)^{-s}) \vee 8192 $. Suppose $\epsilon = (n/128)^{-1/s}$. Let $\Theta = \{\theta_1, \ldots, \theta_k\}$ be a $\epsilon$-separated set of $S^\star$. For the above choices of $n$ and $\epsilon$, we observe that, $k = \epsilon^{-s} = n/128 \ge 64$ and $n \ge 64 k$.

Let $ \phi_j(x) = \one\{ x= \theta_j\} - \one\{ x= \theta_{\lfloor k/2\rfloor + j}\}$, for all $j = 1, \dots, \lfloor k/2 \rfloor$.  
Let,  $\bomega \in \{0,1\}^k$. We define the probability mass function on $\Theta$,

\[P_{\bomega} (x) = \frac{1}{k} + \frac{\delta_k}{k} \sum_{j=1}^{\lfloor k/2 \rfloor} \omega_j \phi_j(x),\]
with $\delta_k \in (0,1/2]$. By construction, $d^\ast_\beta(P_{\bomega}) \le d^\star$. 
Furthermore,
\(\|P_{\bomega} - P_{\bomega^\prime}\|_{\text{TV}} = \frac{ \delta_k}{k} \|\bomega- \bomega^\prime\|_1.\)
By the Varshamov-Gilbert bound (Lemma \ref{vg bound}), let $\Omega \subseteq \{0,1\}^{\lfloor k/2 \rfloor}$ be such that $|\Omega| \ge 2^{\frac{1}{8}\lfloor k/2 \rfloor}$ and $\|\bomega-\bomega^\prime\|_1 \ge \frac{1}{8} \lfloor k/2 \rfloor$, for all $\bomega \neq \bomega^\prime$ both in $\Omega$. Thus for any $\bomega \neq \bomega^\prime$, both in $\Omega$, 
\begin{align}
    \|P_{\bomega} - P_{\bomega^\prime}\|_{\text{TV}} \ge  \frac{ \delta_k \lfloor k/2 \rfloor}{8k} . \label{mm1}
\end{align}
Hence, by Lemma \ref{lem_28}, $\|P_{\bomega} - P_{\bomega^\prime}\|_{\sH^\beta(\Real^d, \Real, 1)} \gtrsim \epsilon^\beta \frac{ \delta_k \lfloor k/2 \rfloor}{k}$. Furthermore, we observe that
\begingroup
\allowdisplaybreaks
\begin{align*}
       \operatorname{KL}(P_{\bomega}^{\otimes_n}\| P_{\bomega^\prime}^{\otimes_n})=  n \operatorname{KL}(P_{\bomega}\| P_{\bomega^\prime}) \le  \, n \, \chi^2(P_{\bomega}\| P_{\bomega^\prime})
      = & n \sum_{i=1}^k \frac{(P_{\bomega}(\xi_i) - P_{\bomega^\prime(\xi_i})^2}{P_{\bomega}(\xi_i)}\\
      \le & 2 nk \sum_{i=1}^k (P_{\bomega}(\xi_i) - P_{\bomega^\prime}(\xi_i)^2\\
      \le & 2nk  \lfloor k/2 \rfloor (2\delta_k /k)^2\\
      = & 8 \frac{n \lfloor k/2 \rfloor \delta_k^2}{k}.
\end{align*}
\endgroup
Thus, 
\(\frac{1}{|\Omega|^2} \sum_{\bomega, \bomega^\prime \in \Omega} \operatorname{KL}(P_{\bomega}^{\otimes_n}\| P_{\bomega^\prime}^{\otimes_n}) \le 8 \frac{n \lfloor k/2 \rfloor \delta_k^2}{k}.\)
Let $\cP = \{P_{\bomega}: \bomega \in \Omega\}$. Let $J \sim \operatorname{Unif}(\Omega)$ and $Z|J=\bomega \sim P_{\omega}$. By the convexity of KL divergence (see equation 15.34 of \cite{wainwright2019high}), we know that,
\[I(Z;J) \le \frac{1}{|\Omega|^2} \sum_{\bomega, \bomega^\prime \in \Omega} \operatorname{KL}(P_{\bomega}^{\otimes_n}\| P_{\bomega^\prime}^{\otimes_n}) \le 8 \frac{n \lfloor k/2 \rfloor \delta_k^2}{k}.\]
\begin{align}
   \text{Thus, } \frac{I(Z;J) + \log 2}{\log |\Omega|} \le 8\frac{I(Z;J) + \log 2}{\lfloor k/2 \rfloor\log 2} \le 64 \frac{n  \delta_k^2}{k \log 2} + \frac{8}{ \lfloor k/2 \rfloor} \le  64 \frac{n  \delta_k^2}{k \log 2} + \frac{1}{ 4}.
\end{align}
The last inequality follows since $k \ge 64$. We take $\delta_k = \frac{1}{16} \sqrt{\frac{k \log 2}{n}}$. Clearly, $\epsilon \le 1/2$ as $n \ge 64 k$. This choice of $\epsilon$ makes,
\( \frac{I(Z;J) + \log 2}{\log |\Omega|} \le \frac{1}{2}.\)
Thus, by Theorem 15.2 of \cite{wainwright2019high}, 
\begin{align*}
    \inf_{\hat{\mu}} \sup_{\mu \in \fM} \| \hat{\mu} - \mu\|_{\sH^\beta(\Real^d, \Real, C)} \gtrsim  \frac{ \epsilon^\beta \delta_k \lfloor k/2 \rfloor}{k} =  \frac{ \epsilon^\beta \lfloor k/2 \rfloor}{16 k}  \sqrt{\frac{k \log 2}{n}} \ge & \frac{  \epsilon^\beta \lfloor k/2 \rfloor}{16k} \sqrt{\frac{ \log 2}{128}} \\
    \gtrsim & \epsilon^\beta \\
    \asymp & n^{-\beta/s}.
\end{align*}
\end{proof}

The proof of Theorem \ref{thm_31} is similar to that of Theorem \ref{thm_30}. For the sake of completeness, the proof is provided below.
\thmthirty*
\begin{proof}
 Let  $\mu_\star \in \fM_{d_\star}$ be such that $d_\ast(\mu_\star) = d_\star$ (this existence is guaranteed by Proposition \ref{prop_45}).
 Now let, $s<d_\ast(\mu) = d_\star$. We can find a $\tau_\star$ and $\epsilon_\star$, such that $\sN_\epsilon (\mu_\star, \tau) \ge \epsilon^{-s}$, for any $\epsilon \in (0, \epsilon_\star]$ and $\tau \in (0, \tau_\star]$. We can thus find a set $S_\star \subseteq [0,1]^d$, such that $\cM(\epsilon; S_\star, \varrho) \ge \cN(\epsilon; S_\star, \varrho) \ge \epsilon^{-s}$. We take $n \ge (128 \epsilon_\star^{-s}) \vee 8192 $. Suppose $\epsilon = (n/128)^{-1/s}$. Let $\Theta = \{\theta_1, \ldots, \theta_k\}$ be a $\epsilon$-separated set of $S_\star$. For the above choices of $n$ and $\epsilon$, we observe that, $k = \epsilon^{-s} = n/128 \ge 64$ and $n \ge 64 k$,

Let $ p_j(x) = \one\{ x= \theta_j\} - \one\{ x= \theta_{\lfloor k/2\rfloor + j}\}$, for all $j = 1, \dots, \lfloor k/2 \rfloor$. 
Let,  $\bomega \in \{0,1\}^k$. We define the probability mass function on $\Theta$,
\(P_{\bomega} (x) = \frac{1}{k} + \frac{\delta_k}{k} \sum_{j=1}^{\lfloor k/2 \rfloor} \omega_j \phi_j(x),\)
with $\delta_k \in (0,1/2]$. 
Furthermore,
\(\|P_{\bomega} - P_{\bomega^\prime}\|_{\text{TV}} = \frac{ \delta_k}{k} \|\bomega- \bomega^\prime\|_1.\)
By the Varshamov-Gilbert Bound (Lemma \ref{vg bound}), let $\Omega \subseteq \{0,1\}^{\lfloor k/2 \rfloor}$ be such that $|\Omega| \ge 2^{\frac{1}{8}\lfloor k/2 \rfloor}$ and $\|\bomega-\bomega^\prime\|_1 \ge \frac{1}{8} \lfloor k/2 \rfloor$, for all $\bomega \neq \bomega^\prime$ both in $\Omega$. Thus for any $\bomega \neq \bomega^\prime$, both in $\Omega$, \(\|P_{\bomega} - P_{\bomega^\prime}\|_{\text{TV}} \ge  \frac{ \delta_k \lfloor k/2 \rfloor}{8k} . \)
Hence, by Lemma \ref{lem_28}, $\|P_{\bomega} - P_{\bomega^\prime}\|_{\sH^\beta(\Real^d, \Real, 1)} \gtrsim \epsilon^\beta \frac{ \delta_k \lfloor k/2 \rfloor}{k}$. Furthermore, we observe that \(\operatorname{KL}(P_{\bomega}^{\otimes_n}\| P_{\bomega^\prime}^{\otimes_n}) \le \frac{n \lfloor k/2 \rfloor \delta_k^2}{k}.\) 
Thus, 
\(\frac{1}{|\Omega|^2} \sum_{\bomega, \bomega^\prime \in \Omega} \operatorname{KL}(P_{\bomega}^{\otimes_n}\| P_{\bomega^\prime}^{\otimes_n}) \le 8 \frac{n \lfloor k/2 \rfloor \delta_k^2}{k}.\)
Let $\cP = \{P_{\bomega}: \bomega \in \Omega\}$. Let $J \sim \operatorname{Unif}(\Omega)$ and $Z|J=\bomega \sim P_{\omega}$. By the convexity of KL divergence (see equation 15.34 of \cite{wainwright2019high}), we know that,
\(I(Z;J) \le \frac{1}{|\Omega|^2} \sum_{\bomega, \bomega^\prime \in \Omega} \operatorname{KL}(P_{\bomega}^{\otimes_n}\| P_{\bomega^\prime}^{\otimes_n}) \le 8 \frac{n \lfloor k/2 \rfloor \delta_k^2}{k}.\)
Thus,
\begin{align}
    \frac{I(Z;J) + \log 2}{\log |\Omega|} \le 8\frac{I(Z;J) + \log 2}{\lfloor k/2 \rfloor\log 2} \le 64 \frac{n  \delta_k^2}{k \log 2} + \frac{8}{ \lfloor k/2 \rfloor} \le  64 \frac{n  \delta_k^2}{k \log 2} + \frac{1}{ 4}.
\end{align}
The last inequality follows since $k \ge 64$. We take $\delta_k = \frac{1}{16} \sqrt{\frac{k \log 2}{n}}$. Clearly, $\epsilon \le 1/2$ as $n \ge 64 k$. This choice of $\epsilon$ makes,
\( \frac{I(Z;J) + \log 2}{\log |\Omega|} \le \frac{1}{2}.\)
Thus, by Theorem 15.2 of \cite{wainwright2019high}, 
\begin{align*}
    \inf_{\hat{\mu}} \sup_{\mu \in \fM_{d_\star}} \|\mu - \hat{\mu}\|_{\sH^\beta(\Real^d, \Real, C)} \gtrsim \epsilon^\beta \frac{ \delta_k \lfloor k/2 \rfloor}{k} = &   \frac{ \epsilon^\beta \lfloor k/2 \rfloor}{16k} \sqrt{\frac{k \log 2}{n}} \\
    \ge & \frac{\epsilon^\beta   \lfloor k/2 \rfloor}{16k}  \sqrt{\frac{ \log 2}{128}} \\ \gtrsim & \epsilon^\beta \asymp n^{-\beta/s}.
\end{align*}

To show that $\inf_{\hat{\mu}} \sup_{\mu \in \fM_{d_\star}} \| \hat{\mu} - \mu\|_{\sH^\beta(\Real^d, \Real, C)} \succsim n^{-1/2}$, we use Le Cam's method \cite[Chapter 15.2]{wainwright2019high}. Let $\theta_0, \theta_1 \in [0,1]^d$ be such that $\|\theta_0 - \theta_1\|_\infty \ge 1/2$. Let $P_0(\theta_0) = P_0(\theta_1) = 1/2 $ and $P_1(\theta_0) = 1- P_1(\theta_1) = 1/2 - \delta $ with $\delta \in (0,1/4)$. Clearly, $P_0, P_1 \in \fM_{d_\star}$. Further, $\operatorname{TV}(P_0, P_1) = \delta.$ Thus, by Lemma~\ref{lem_28}, we observe that \[\|P_1 - P_0\|_{\sH^\beta(\Real^d, \Real, C)} \succsim 2^{-\beta} \delta \succsim \delta.\]  Again,
\begin{align*}
       \operatorname{KL}(P_{1}^{\otimes_n}\| P_{0}^{\otimes_n})=  n \operatorname{KL}(P_{1}\| P_{0}) \le & \, n \, \chi^2(P_{1}\| P_{0}) = 4n \delta^2.
\end{align*}
By Pinsker's inequality \cite[Lemma 2.5]{tsybakov2009introduction}, we note that, 
\begin{align*}
       \operatorname{TV}(P_{1}^{\otimes_n}, P_{0}^{\otimes_n}) \le \sqrt{\frac{1}{2}\operatorname{KL}(P_{1}^{\otimes_n}\| P_{0}^{\otimes_n})} = 2 \delta \sqrt{n} = 1/4,
\end{align*}
if $\delta = \frac{1}{8\sqrt{n}}$. Thus from equation 15.14 of \citet{wainwright2019high}, we observe that,
\begin{align}
     \inf_{\hat{\mu}} \sup_{\mu \in \fM_{d_\star}} \E_\mu \| \hat{\mu} - \mu\|_{\sH^\beta(\Real^d, \Real, C)}  \succsim \delta \asymp 1/\sqrt{n}. \label{e22}
\end{align}
 \end{proof}

 \section{Supporting Results}
 This section states some of the supporting results used in this paper.
 \begin{lem}[Lemma 5.5 of \citet{wainwright2019high}]
     \label{cov_pack}For any metric space, $(S,\varrho)$ and $\epsilon>0$,
     $M(2\epsilon; S, \varrho) \le\cN(\epsilon; S, \varrho) \le M(\epsilon; S, \varrho)$.
 \end{lem}
 
 \begin{lem}[Theorem XIV of \citet{kolmogorov1961}; also see page 114 of \citet{shiryayev1992selected}]\label{kt_og}
Suppose that $\cM$ be a subset of a real vector space, then for any $s>\overline{\text{dim}}_M(\cM)$, there exists $\epsilon_0 \in (0,1)$, such that if $\epsilon \in (0, \epsilon_0]$, \[\log \cN \left(\epsilon; \sH^\beta (\cM), \|\cdot \|_{\infty}\right) \lesssim \epsilon^{-s/\beta}.\] Furthermore, \(\log \cN \left(\epsilon; \sH^\beta ([0,1]^d), \|\cdot \|_{\infty}\right) \lesssim \epsilon^{-d/\beta}.\)
\end{lem}

 \begin{lem}[Lemma 21 of \citet{JMLR:v21:20-002}]\label{lem_nakada} 
     Let $\cF = \cR\cN(W, L, B)$ be a space of ReLU networks with the number of weights, the number of layers, and the maximum absolute value of weights bounded by $W$, $L$, and $B$ respectively. Then,
\[
\log\cN(\epsilon, \cF, \ell_\infty) \leq W \log \left( 2LB^L(W + 1)^L \frac{1}{\epsilon} \right).
\]
 \end{lem}
 \begin{lem}[Theorem 12.2 of \citet{anthony1999neural}]
 \label{lem_anthony_bartlett} 
     Assume for all $f \in \cF$, $\|f\|_{\infty} \leq M$. Denote the pseudo-dimension of $\cF$ as $\text{Pdim}(\cF)$, then for $n \geq \text{Pdim}(\cF)$, we have for any $\epsilon$ and any $X_1, \ldots, X_n$,
\(
\cN(\epsilon; \sF_{|_{X_{1:n}}}, \ell_\infty) \leq \left( \frac{2eM  n}{\epsilon\text{Pdim}(\cF)}\right)^{\text{Pdim}(\cF)}.
\)
 \end{lem}
 \begin{lem}[Theorem 6 of \cite{bartlett2019nearly}]
 \label{lem_bartlett}
     Consider the function class computed by a feed-forward neural network architecture with $W$ many weight parameters and $U$ many computation units arranged in $L$ layers. Suppose that all non-output units have piecewise-polynomial activation functions with $p+1$ pieces and degrees no more than $d$, and the output unit has the identity function as its activation function. Then the VC-dimension and pseudo-dimension are upper-bounded as
\[
\text{VCdim}(\cF), \text{Pdim}(\cF) \leq C \cdot LW \log(pU) + L^2 W \log d.
\]
 \end{lem}

\begin{lem} [Lemma 2.9 of \cite{tsybakov2009introduction}]
\label{vg bound}
    Let $m \geq 8$. Then there exists a subset $\{\omega_0, \ldots, \omega_M\}$ of $\{0,1\}^m$ such that $\omega_0 = (0, \ldots, 0)$,
$\|\omega_j - \omega_k\|_1 \geq \frac{m}{8}$, for all $0 \leq j < k \leq M$,
and $M \geq 2^{m/8}$.
\end{lem}
\bibliographystyle{apalike}

\end{document}

%% file: custom_definitions.tex
\def\Real{\mathop{\mathbb{R}}\nolimits}

\def\argmin{\mathop{\rm argmin}}

\newcommand{\pdim}{\text{Pdim}}
\newcommand{\prob}{\mathbb{P}}
\newcommand{\one}{\mathbbm{1}}

\newcommand{\bi}{\boldsymbol{i}}
\newcommand{\bj}{\boldsymbol{j}}

\newcommand{\bs}{\boldsymbol{s}}

\newcommand{\bdelta}{\boldsymbol{\delta}}
\newcommand{\bepsilon}{\boldsymbol{\epsilon}}

\newcommand{\bomega}{\boldsymbol{\omega}}

\newcommand{\cA}{ \mathcal{A}}
\newcommand{\cB}{ \mathcal{B}}
\newcommand{\cC}{ \mathcal{C}}

\newcommand{\cF}{ \mathcal{F}}

\newcommand{\cH}{ \mathcal{H}}
\newcommand{\cI}{ \mathcal{I}}

\newcommand{\cL}{ \mathcal{L}}
\newcommand{\cM}{ \mathcal{M}}
\newcommand{\cN}{ \mathcal{N}}
\newcommand{\cO}{ \mathcal{O}}
\newcommand{\cP}{ \mathcal{P}}

\newcommand{\cR}{ \mathcal{R}}
\newcommand{\cS}{ \mathcal{S}}
\newcommand{\cT}{ \mathcal{T}}

\newcommand{\cV}{ \mathcal{V}}
\newcommand{\cW}{ \mathcal{W}}

\newcommand{\sA}{ \mathscr{A}}

\newcommand{\sD}{ \mathscr{D}}
\newcommand{\sE}{ \mathscr{E}}
\newcommand{\sF}{ \mathscr{F}}
\newcommand{\sG}{ \mathscr{G}}
\newcommand{\sH}{ \mathscr{H}}
\newcommand{\sI}{ \mathscr{I}}

\newcommand{\sN}{ \mathscr{N}}

\newcommand{\sP}{ \mathscr{P}}
\newcommand{\sQ}{ \mathscr{Q}}

\newcommand{\sW}{ \mathscr{W}}
\newcommand{\sX}{ \mathscr{X}}

\newcommand{\sZ}{ \mathscr{Z}}

\newcommand{\hD}{ \hat{D}}
\newcommand{\hE}{ \hat{E}}

\newcommand{\hG}{ \hat{G}}

\newcommand{\hnu}{\hat{\nu}}
\newcommand{\hmu}{\hat{\mu}}
\newcommand{\E}{\mathbb{E}}

\newcommand{\relu}{\text{ReLU}}

\newcommand{\fH}{\mathbb{H}}

\newcommand{\fL}{\mathbb{L}}
\newcommand{\fM}{\mathbb{M}}
\newcommand{\fN}{\mathbb{N}}